\documentclass[10pt]{article}

\usepackage[utf8]{inputenc} 
\usepackage[T1]{fontenc}    

\usepackage{url}            
\usepackage{booktabs}       
\usepackage{amsfonts}       
\usepackage{nicefrac}       
\usepackage{microtype}      
\usepackage{fullpage}
\usepackage{mathtools}
\usepackage{xcolor}

\usepackage{microtype}
\usepackage{graphicx}
\usepackage{subfigure}
\usepackage{booktabs} 
\usepackage[paper=a4paper]{geometry}
\usepackage{hyperref}

\usepackage{amsmath}
\usepackage{amssymb}
\usepackage{mathtools}
\usepackage{amsthm}
\usepackage{enumitem}
 
\usepackage{natbib}
\usepackage{subfigure}
\usepackage{adjustbox}

\usepackage{commath}
\usepackage{wrapfig}

\usepackage{algorithm}
\usepackage{algorithmic}

\usepackage{amsfonts,bm,epsfig,graphicx,epstopdf,bbm}
\usepackage{setspace,latexsym,epsf,dsfont,caption,mathtools,color}
\usepackage{tabularx}
\usepackage{multirow}

\usepackage[capitalize,noabbrev]{cleveref}

\usepackage{tikz}

\theoremstyle{plain}
\newtheorem{theorem}{Theorem}[section]
\newtheorem{proposition}[theorem]{Proposition}
\newtheorem{lemma}[theorem]{Lemma}
\newtheorem{corollary}[theorem]{Corollary}
\theoremstyle{definition}
\newtheorem{definition}[theorem]{Definition}
\newtheorem{assumption}[theorem]{Assumption}
\theoremstyle{remark}
\newtheorem{remark}[theorem]{Remark}

\theoremstyle{claim}
\newtheorem{claim}[theorem]{Claim}

\crefname{claim}{claim}{claims}
\crefname{assumption}{assumption}{assumptions}

\usepackage[textsize=tiny]{todonotes}

\usepackage{multirow}
\usepackage{pifont}

\newcommand{\Rb}{\mathbb{R}}
\newcommand{\Eb}{\mathbb{E}}
\newcommand{\Pb}{\mathbb{P}}
\newcommand{\Qb}{\mathbb{Q}}

\newcommand{\myparagraph}[1]{{\noindent \bf #1}\hspace{0.0in}}

\newcommand{\Dc}{\mathcal{D}}
\newcommand{\Ac}{\mathcal{A}}
\newcommand{\Tc}{\mathcal{T}}

\newcommand{\Qc}{\mathcal{Q}}

\newcommand{\Ec}{\mathcal{E}}
\newcommand{\Hc}{\mathcal{H}}

\newcommand{\TT}{\intercal}

\newcommand{\KL}{\texttt{KL}}

\newcommand{\Rt}{\mathtt{R}}
\newcommand{\Alg}{\mathtt{Alg}}

\newcommand{\robin}{\texttt{ROBIN\hspace{0.03in}}}
\newcommand{\WMHD}{\texttt{WMHD\hspace{0.03in}}}

\title{Federated Linear Contextual Bandits with User-level Differential Privacy}

\author{
 Ruiquan Huang\thanks{School of EECS, The Pennsylvania State University, University Park, PA, USA. Correspondence to: Jing Yang <yangjing@psu.edu>.},~~~
 Huanyu Zhang\thanks{Meta Platforms, Inc., New York, NY 10003, USA.},~~~
 Luca Melis\footnotemark[2],~~~
 Milan Shen\footnotemark[2],\\
  Meisam Hajzinia\thanks{Google, Mountain View, CA 94043, USA. The collaboration was during the time  when the author worked at Meta.},~~~
 Jing Yang\footnotemark[1] 
}
\date{}

\begin{document}

\maketitle




\begin{abstract}
This paper studies federated linear contextual bandits under the notion of user-level differential privacy (DP). We first introduce a unified federated bandits framework that can accommodate various definitions of DP in the sequential decision-making setting. We then formally introduce user-level central DP (CDP) and local DP (LDP) in the federated bandits framework, and investigate the fundamental trade-offs between the learning regrets and the corresponding DP guarantees in a federated linear contextual bandits model. For CDP, we propose a federated algorithm termed as \robin and show that it is near-optimal in terms of the number of clients $M$ and the privacy budget $\varepsilon$ by deriving nearly-matching upper and lower regret bounds when user-level DP is satisfied. For LDP, we obtain several lower bounds, indicating that learning under user-level $(\varepsilon,\delta)$-LDP must suffer a regret blow-up factor at least  {$\min\{1/\varepsilon,M\}$ or $\min\{1/\sqrt{\varepsilon},\sqrt{M}\}$} under different conditions.

 
\end{abstract}

\section{Introduction}
Federated learning (FL)~\citep{mcmahan2017communicationefficient} has become a trending distributed machine learning paradigm where numerous clients collaboratively train a prediction model under the coordination of a central server while keeping the local training data at each client. FL is motivated by various applications where real-world data are exogenously generated at edge devices, and it is desirable to protect the privacy of local data by only sharing model updates instead of the raw data. 

While the majority of FL studies focus on the supervised learning setting, recently, a few researchers begin to extend FL to the multi-armed bandits (MAB) framework~\citep{lai1985asymptotically,auer2002finite,bubeck2012regret,agrawal2012analysis,agrawal2013further}, and have proposed several federated bandits models and learning algorithms~\citep{shi2021federated,shi2021aistats,dubey2020differentially,huang2021federated,he2022simple}. Under the classical MAB model, a player chooses to play one out of a set of arms at each time slot. An arm, if played, will generate a reward that is randomly drawn from a fixed but unknown distribution. With all previous observations, the player needs to decide which arm to pull in each time in order to maximize the cumulative expected reward. MAB thus represents an online learning model that naturally captures the intrinsic exploration-exploitation tradeoff in many sequential decision-making problems. Under the new realm of federated bandits, each client is facing a local bandits model with shared parameters. While classical MAB allows immediate access to the sequentially generated data for decision-making, in federated bandits, local data streams are generated and analyzed at the clients, with periodic information exchange among clients.
Such a model is naturally motivated by a corpus of applications, such as recommender systems, clinical trials, and cognitive radio, where the sequential decision-making involves multiple clients and is distributed by nature. 

Meanwhile, {\it user-level} differential privacy (DP) has emerged as a stronger but more realistic notion of privacy, which guarantees the privacy of a user's entire contribution of data instead of individual samples. Roughly speaking, user-level DP requires that the output of an algorithm does not significantly
change when a user's entire contribution has been changed. While user-level DP has been studied in applications involving {\it static} datasets, such as discrete distribution estimation~\citep{acharya2022discrete,cummings2021mean}, learning~\citep{levy2021learning,ghazi2021user}, and federated learning~\citep{girgis2022distributed}, to the best of our knowledge, it has not been studied in the {\it online} setting, where data is generated sequentially as decisions are made and outcomes are observed.


In this work, we take an initial effort to incorporate user-level DP into the framework of federated bandits, where $M$ clients work collaboratively to learn a shared bandits model through a central server. We note that ensuring user-level DP in federated bandits setting is extremely challenging, mainly due to the following reasons. First, due to the interactive sequential decision-making nature of bandits, we do not have static datasets at clients. Rather, each local sample is generated online accordingly to the federated learning mechanism adopted by the system. As a result, its distribution depends on not only the historical data at the client, but also the entire history across all clients and the server. Such intricate dependency makes the definition of user-level DP elusive in the federated bandits setting. 

Second, due to the bandit feedback, only the reward of the pulled arm is observed, where the arm-pulling decision depends on history and the DP mechanism. Thus, the local samples are non-independent and identically distributed ({non-}IID), which is in stark contrast to the IID assumption usually adopted in the literature of user-level DP.

Third, in the federated bandits setting, multiple clients jointly learn the shared bandits parameter collaboratively. In general, the number of data samples contributed by a single client grows linearly in time horizon $T$, which is  unbounded. Thus, to achieve user-level DP, vanilla noise-adding DP mechanisms may require the corresponding noise variance scale in the same order. On the other hand, in order to obtain sublinear learning regret, it requires that the estimation error in the estimated model parameters decays fast in time. Thus, it is unclear whether it is possible to still achieve sublinear learning regret while guaranteeing user-level DP.





\begin{table*}[t]
\caption{Regret Bounds of Linear Contextual Bandits under Different $(\varepsilon,\delta)$-DP Constraints.}
	\vspace{-0.2in}
\label{table:regret}
\vskip 0.15in
\begin{center}
\begin{tiny}
\begin{sc}
\begin{tabular}{lcccr}
\toprule
    Algorithm & Asm.\ref{asm:margin} &Model & Constraint & Regret\\
    \midrule
    lower bound {\tiny \citep{he2022reduction}}& \ding{55} & Sing & Item-Level JDP &$\Omega\left(\sqrt{dT} + d/\varepsilon\right)$ \\
    \midrule 
   FedUCB {\tiny\citep{dubey2020differentially}} & \ding{55} &Fed & Item-level JDP & $\tilde{O}\left(d^{3/4}M^{3/4}\sqrt{T/\varepsilon}\right)^{\dagger}$ \\
    \midrule
    P-FedLinUCB {\tiny\citep{zhou2023differentially}} & \ding{55} & Fed & Item-level CDP & $\tilde{O}\left(d^{3/4} \sqrt{MT/\varepsilon} + {d\sqrt{MT}}\right)$ \\
    \midrule
    \robin (\Cref{thm:upper bound}) & \checkmark & Fed & User-Level CDP &$\tilde{O}\left(\min\left\{M, \max\left\{ 1, \frac{d\log T}{M\varepsilon^2}\right\} \right\} C_0d\log T\right)$ \\
    \midrule
    \multirow{4}{12em}{Lower Bound (\Cref{thm:main-CDP-minimax,thm:main-LDP})}
      & \checkmark & Fed & User-level CDP & $\Omega(\min \left\{M, \max\left\{1, \frac{1}{M\varepsilon^2} \right\}\right\}C_0d\log T)$  \\
     &\ding{55} & Fed & User-level CDP & $\Omega\left(\min\left\{M, \max\left\{\sqrt{M}, \frac{1}{\varepsilon} \right\} \right\} \sqrt{dT}\right)$ \\
     &\checkmark&Fed & User-level LDP$^*$ & $\Omega\left(\min\left\{M, 1/\varepsilon\right\}C_0d\log T  \right)$ \\
     &\ding{55} &Fed & User-level LDP$^*$ & $\Omega\left(\min\left\{M, \sqrt{M/\varepsilon}\right\}\sqrt{dT}  \right) $ \\
\bottomrule
\end{tabular}
\end{sc}
\end{tiny}
\end{center}

    \begin{center}
        \scriptsize
	$C_0$: parameter of Asm~\ref{asm:margin}, $d$: dimension of model parameter, $M$: number of clients, $T$: time horizon. \texttt{SING} and \texttt{Fed} stand for single-client and multi-client settings, respectively. \texttt{JDP} stands for jointly differential privacy. The standard non-private lower bounds are $\Omega(\sqrt{dMT})$  and $\Omega(C_0d\log T)$, without or with Asm~\ref{asm:margin}, respectively.\\
 $*$: The result for user-level $(\varepsilon,0)$-LDP is presented in \Cref{coro:pure LDP lower bound}.
 $\dagger$: We adopt the result in \citet{zhou2023differentially}.
	\end{center}
	\vspace{-0.1in}
\vskip -0.1in
\end{table*}


We tackle those aforementioned challenges explicitly in this work. Specifically, our main contributions can be summarized as follows. 

First, we formally introduce a DP oriented federated bandits framework, which provides a principled viewpoint to capture and characterize potential privacy leakage in online decision-making problems. Our general framework can accommodate all previously introduced differential privacy notions in the bandits settings, such as joint DP~\citep{shariff2018differentially,dubey2020differentially}. We then specialize the framework to capture user-level DP, and introduce both central and local user-level DP.

Second, we investigate user-level central differential privacy (CDP), and study the fundamental trade-off between learning regret and DP guarantee. Under standard margin condition and diversity condition studied in conventional linear contextual bandits \citep{hao2020adaptive,papini2021leveraging}, we propose a near-optimal algorithm termed as \robin with user-level $(\varepsilon,\delta)$-DP guarantee. \robin enjoys a regret of $\tilde{O}\left(\max\left\{1,\frac{d\log T}{M\varepsilon^2}\right\}C_0d\log T\right)$, where $C_0$ is a margin parameter, $d$ is the dimension of the features, $M$ is the number of clients, and $T$ is the time horizon. The near optimality is established by obtaining a minimax lower bound $\Omega\left(\max\left\{1,\frac{1}{M\varepsilon^2}\right\}C_0d\log T\right)$ under the same CDP constraint and the diversity and margin conditions. Furthermore, we also investigate the lower bound without the margin condition. Compared to the non-private counterpart, our results indicate that when $\varepsilon=O(1/\sqrt{M})$, the regret suffers a blow-up factor of at least {$\min\{M,\frac{1}{\varepsilon^2M}\}$ or $\min\{{\sqrt{M}},\frac{1}{\varepsilon\sqrt{M}}\}$}, depending on whether the margin condition is imposed or not. When $\varepsilon=\Omega(1/\sqrt{M})$,  imposing the user-level CDP constraint does not affect the hardness of the federated linear contextual bandits problem, regardless of whether the margin condition is satisfied.

Third, we study user-level local different privacy (LDP) under several settings and conditions. When $\varepsilon = O(1/M)$, the minimax regret lower bound is either $\Omega(M\log T)$ under the margin condition or $\Omega(M\sqrt{T})$ without the margin condition, suggesting that the best policy is to have the clients independently make arm-pulling decisions based on their own local datasets without information sharing. When $\varepsilon = \Omega(1/M)$, we obtain a minimax lower bound in the order of $\Omega(C_0d\log T/\varepsilon)$ under the margin condition and $\Omega(\sqrt{dMT/\varepsilon})$ without the margin condition, indicating that any federated linear contextual bandits algorithm satisfying user-level LDP suffers a regret blow-up factor of at least $1/\sqrt{\varepsilon}$ without the margin condition, or $1/\varepsilon$ with the margin condition. Thus, the user-level LDP constraint makes the learning problem strictly harder than the non-private case.  Moreover, we also develop a tighter lower bound for pure LDP, i.e. $\delta=0$, {which can be obtained by replacing $\varepsilon$ in the CDP lower bound by $\varepsilon/\sqrt{M}$.} 
A summary of our main results is shown in Table~\ref{table:regret}.

\section{The Private Federated Bandits Framework}
\subsection{Notations}
For any $M\in\mathbb{N}$, we denote $[M]:=\{1,\dots,M\}$. For any vector $x$ and symmetric matrix $V$, we denote $\|x\|$ as the $\ell_2$ norm of $x$, $\|x\|_V : = \sqrt{x^\top Vx}$, and $\lambda_{\min}(V)$, $\lambda_{\max}(V)$ refer to the minimum and maximum eigenvalues of $V$, respectively.  For any matrix $X$, we use $\|X\|$ to denote its spectral norm, and use $X^{\dagger}$ to denote its pseudo-inverse. For any set $\Ac$, we denote $\Ac^n$ as its $n$-fold Cartesian product. For two probability measures $\Pb, \Qb$, we use $d_{TV}(\Pb,\Qb)$ and $\KL(\Pb,\Qb)$ to denote their total variation distance and Kullback–Leibler distance, respectively. We denote $q_{\leq t}:=\{q_1,\ldots,q_t\}$. We use $\mathcal{F}(S_1,S_2)$ to denote the set of all possible measurable functions from set $S_1$ to another set $S_2$. {$\tilde{O}(f)$ hides the logarithm term of $f$, i.e. $\tilde{O}(f)=O(f|\log f|)$.} 


\subsection{The Federated Bandits Framework}


We consider a federated linear contextual bandits setting captured by tuple $([M], \mathcal{A}, \mathcal{C}, \phi, \mathcal{P}, d)$, where $[M]$ is the set of clients, $\mathcal{A}$ is the set of arms, $\mathcal{C}$ is the set of contexts, $\mathcal{P}:=\{\rho_i\}_{i\in[M]}$ is a set of distributions over context set $\mathcal{C}$, and $\phi:\mathcal{C}\times\mathcal{A}\rightarrow \mathbb{R}^d$ is the feature mapping. 

At each time slot $t$, each client $i$ observes a context $c_{i,t}\in\mathcal{C}$ drawn according to distribution $\rho_i\in\mathcal{P}$. Then, client $i$ pulls arm $a_{i,t}\in\mathcal{A}$ and receives a reward $r_{i,t} = \phi(c_{i,t},a_{i,t})^{\TT}\theta^* + \eta_{i,t}$, where $\theta^*\in\Rb^d$ is an unknown parameter vector shared among clients, $\eta_{i,t}$ is a random noise, and $\phi(c_{i,t},a_{i,t})\in \Rb^d$ is the feature vector associated with arm $a_{i,t}$ and context $c_{i,t}$. For simplicity, we denote $\phi(c_{i,t},a)$ as $x_{i,t,a}$. 
We assume $\|\phi(c,a)\|_2\leq 1$, $\|\theta^*\|_2\leq 1$, and $\eta_{i,t}$ is an IID standard Gaussian random variable, i.e., $\eta_{i,t}\sim N(0,1)$.
Such assumptions are standard in the linear contextual bandits literature \citep{Abbasi:2011:IAL,chu2011contextual}.


We assume there exists a central server in the system, and similar to FL, the clients can communicate with the server periodically with zero latency. Specifically, the clients can send ``local model updates'' to the central server, which then aggregates and broadcasts the updated ``global model'' to the clients. (We will specify these components later.) 
We also assume that clients and server are fully synchronized \citep{mcmahan2017communicationefficient}.

For federated linear contextual bandits, the utility of primary interest is the expected cumulative regret among all clients, defined as:
\begin{align*}\textstyle
\text{Regret}(M,T) = \mathbb{E}\left[\sum_{i=1}^M\sum_{t=1}^T \left(x_{i,t,a_{i,t}^*}^\TT\theta{^*} - x_{i,t,a_{i,t}}^\TT\theta^*\right)\right],
\end{align*}
where $a_{i,t}^*\in\Ac$ is an optimal arm for client $i$ given context $c_{i,t}$: $\forall b \neq a_{i,t}^*$, $x_{i,t,a_{i,t}^*}^\TT\theta{^*}  -x_{i,t,b}^\TT\theta^*  \geq 0$.

\subsection{User-level Differential Privacy} \label{sec: DP framework}
In order to formally introduce user-level DP into the federated bandits framework, we consider a federated algorithm that consists of $2M+1$ sub-algorithms denoted as $\mathtt{Alg} := (\mathtt{R}_0, \mathtt{Alg}_1,\mathtt{R}_1,\ldots,\mathtt{Alg}_M,\mathtt{R}_M)$, where $\mathtt{Alg}_i$ is an online decision-making algorithm adopted by client $i$, $\mathtt{R}_i$ is a channel that shares the information from client $i$ to the central server, and $\mathtt{R}_0$ is a channel that broadcasts the aggregated information from the server to all clients. 

Mathematically, let $H_{i,t}$ be the local historical  observations and actions of client $i$ before it makes decision at time $t$, i.e. $H_{i,t}=\{c_{i,\tau}, a_{i,\tau},r_{i,\tau}\}_{\tau=1}^{t-1}$, and $q_{i,t}$ and $q_t$ be the outputs of $\mathtt{R}_i$ and $\mathtt{R}_0$ at time $t$, respectively. 

 With a specified $\mathtt{Alg}$, at each time step $t$, the learning procedure proceeds as follows. First, the server aggregates up-to-date local updates from the clients and broadcasts the aggregated information to all clients through channel  $\mathtt{R}_0$, i.e., $\mathtt{R}_0: \{q_{i,\leq t}\}_{i} \mapsto q_t$. Upon receiving the broadcast information, each client performs local online decision-making by executing $\mathtt{Alg}_i: \{c_{i,t}\}\cup{H}_{i,t}\cup\{q_{\leq t}\} \mapsto a_{i,t} $, and obtains $r_{i,t}$. Finally, each client generates a local update and uploads it to the server through channel $\mathtt{R}_i: ({H}_{i,t+1},q_{\leq t}) \mapsto q_{i,t+1}$.

We note that the general federated bandits learning framework can accommodate various bandits and communication models. For example, 
depending on whether communication happens in a step $t$ or not, we can let $q_{i,t} = 0$ to indicate that client $i$ does not upload any information at time $t$, and $q_{t}=0$ if the server does not broadcast at time $t$. 

Based on the specified federated bandits learning framework, we then introduce the user-level DP notions as follows. Let $H_{t}=\{H_{i,t}\}_{i\in[M]}$ be the entire history across all clients. Note that $H_{t}$ is a streaming dataset, i.e. $H_{t'}\subseteq H_{t}$ for any $t'\leq t$. Due to the online setting, we follow the definition of differential privacy {\it under continual observation \citep{dwork2010differential}.} {For simplicity, we denote $\Rt_i(\{H_{i,\tau},q_{\tau-1}\}_{\tau\leq t}) := (\Rt_i(H_{i,1},q_0),\ldots, \Rt_i(H_{i,t},q_{\leq t-1})), \forall i\in[M]$. With a little abuse of notation, we denote $\Rt_0(\{H_{\tau}\}_{\tau\leq t}) := (\Rt_0(\{q_{i,1}\}_{i\in[M]}),\ldots, \Rt_0(\{q_{i,\leq t}\}_{i\in[M]}))$ to indicate the end-to-end relationship between the entire history $\{H_{\tau}\}_{\tau\leq t}$ and the global information $\{q_1,\ldots,q_t\}$ produced by  $\Rt_0$.  Without loss of generality, we use $\mathcal{Q}$ to denote the range of channel $\Rt_i$ for any $i\in[M]\cup\{0\}$.}

\begin{definition}[$i$-neighboring datasets]\label{def:neighbor dataset}
    We say ${H}_{t}=\{H_{j,t}\}_{j\in[M]}$ is $i$-neighboring to $H_{t}' = \{H_{j,t}'\}_{j\in[M]}$ if $H_{j,t}=H_{j,t}'$ for all $j\neq i$.
\end{definition}

\begin{definition}[User-level central DP]\label{def:cdp}
Consider a time horizon $T$. A federated algorithm $\mathtt{Alg} = (\mathtt{R}_0, \mathtt{Alg}_1,\mathtt{R}_1,\ldots,\mathtt{Alg}_M,\mathtt{R}_M)$ is $(\varepsilon,\delta)$-central user-level differentially private if for any $i$-neighboring streaming datasets $\{H_t\}_{t\leq T}$ and $\{H'_t\}_{t\leq T}$, and any subset $Q_{\leq T} := (Q_1,\ldots,Q_T)\subset \Qc^T$, we have
\[
   \Pb[\mathtt{R}_0(\{H_{t}\}_{t\leq T})\in Q_{\leq T}]\leq e^{\varepsilon}\Pb[\mathtt{R}_0(\{H_{t}'\}_{t\leq T}) \in Q_{\leq T}] + \delta.
\]
\end{definition}

Besides the user-level central DP and local DP (to be introduced in \cref{sec:ldp}), we note that the proposed federated bandits framework can accommodate various DP notions, as elaborated in \cref{appx:framework}.


\section{Algorithm Design and Analysis for CDP}
In this section, we aim to design a collaborative learning algorithm for the federated linear contextual bandits that achieves sublinear learning regret under user-level CDP. 

\subsection{Challenges of Adopting Gram Matrix in Algorithms}

To gain a better understanding of our algorithm design, we first elaborate the difficulties encountered by prevalent Gram-Matrix (GM)-based approaches in federated linear contextual bandits when user-level CDP is taken into account. It is worth noting that all current privacy-preserving algorithms utilized in federated linear contextual bandits rely on GM-based approaches. Specifically, we focus on the prominent challenges associated with the widely adopted upper confidence bound (UCB) methods.

The fundamental task in the design of bandits algorithms is to  balance the exploration-exploitation trade-off. UCB-type algorithms achieve this through constructing an upper confidence bound of the estimated reward of each arm, and then picking the arm with the highest UCB in each step $t$. In the linear contextual bandits setting, such UCB is usually in the form of $x^{\top}\hat{\theta} + \alpha\|x\|_{\bar{V}_t^{-1}}$, where $x$ is the feature vector associated with individual arms and the incoming context, $\hat{\theta}$ is the estimated model parameter, $\alpha$ is a constant, and $\bar{V}_t:=I_d +\sum_{\tau<t}x_{\tau}x_{\tau}^{\TT}$ is the matrix defined by the encountered feature vectors that are used to estimate $\hat{\theta}$. Roughly speaking, $\bar{V}_t$ captures the uncertainty in the estimate $\hat{\theta}$, i.e., $\|\hat{\theta} - \theta^*\|_{\bar{V}_t} = \tilde{O}(\sqrt{d})$. By selecting the arm associated with the highest UCB, it ensures that the directions along which $\bar{V}_t$ has small eigenvalues can be sufficiently explored, thus reducing the uncertainty in $\hat{\theta}$ efficiently.


\if{0}
UCB algorithms choose $x_t = \arg\max_{x\in D_t}x^{\top}\hat{\theta} + \alpha\|x\|_{\bar{V}_t^{-1}}$, where $D$ is the decision set at time $t$, and $\alpha=\tilde{O}(\lambda + \sqrt{d})\geq \|\hat{\theta} - \theta^*\|_{\bar{V}_t}$. If $x_*\in D_t$ is the optimal arm at time $t$, then, we can bound the regret at time $t$ as follows.
\begin{align*}
    \text{Regret}_t &= (x_*^{\TT} - x_t^{\TT})\theta^*\\
    & = x_*^{\TT}(\theta^* - \hat{\theta}) - x_t^{\TT}(\theta^*-\hat{\theta}) + (x_*^{\TT} - x_t^{\TT})\hat{\theta}\\
    &\overset{(a)}\leq \alpha\|x_*\|_{V_t^{-1}} + x_*^{\TT}\hat{\theta} + \alpha\|x_t\|_{V_t^{-1}} - x_t^{\TT}\hat{\theta}\\
    &\overset{(b)}\leq \alpha\|x_t\|_{V_t^{-1}} + x_t^{\TT}\hat{\theta} + \alpha\|x_t\|_{V_t^{-1}} - x_t^{\TT}\hat{\theta}\\
    & = 2\alpha\|x_t\|_{V_t^{-1}},
\end{align*}
where $(a)$ is due to Cauchy's inequality and $(b)$ follows from the optimality of $x_t$. Taking summation over $t$, we can show that the total regret is sublinear through the elliptical potential lemma (\Cref{lemma:elliptical} in Appendix).
\begin{align*}
    \text{Regret}(T)&\leq \sum_{t=1}^T 2\alpha\|x_t\|_{V_t^{-1}}\leq \tilde{O}\left(d\sqrt{T}\right)
\end{align*}
\fi

In the federated setting, each client $i$ collects $\{(x_{i,\tau,a_\tau},r_{i,\tau})\}_\tau$ locally and exchanges information with the server for expedited estimation of $\theta^*$. Under the UCB-based framework, in order to characterize the estimation uncertainty and facilitate efficient exploration in subsequent time steps, it in general requires the server to collect and aggregate the Gram matrices $V_{i,t}:=\sum_{\tau\in \Tc(t)}x_{i,t,a_{i,t}}x_{i,t,a_{i,t}}^{\TT}$, where $\Tc(t)$ is a subset of $[t]$ determined by the specific algorithm design, and then broadcast the privatized aggregated Gram matrix $\tilde{V}_t^{-1}$ 
to all clients, along with the privatized estimate $\tilde{\theta}$. 

According to \Cref{def:cdp}, in order to achieve {\it user-level} CDP, it is necessary to ensure that if $V_{i,t}$ is replaced by another Gram matrix $V'_{i,t}$, the broadcast information would remain similar. Standard additive noise mechanisms require that the variance of the additive noise in each dimension scales in $\sup\|V_{i,t}-V'_{i,t}\|$, which, without any additional assumption, scales in $\Omega(t)$ in general. If a noise with $\Omega(t)$ variance is added, the corresponding UCB, denoted as $x^{\top}\tilde{\theta} + \alpha\|x\|_{\tilde{V}_t^{-1}}$, would be largely different from its non-privatized counterpart, resulting in wrong arm-pulling at clients and linear regret.


\subsection{Additional Assumptions}
The aforementioned challenges in balancing user-level CDP and regret faced by UCB-type algorithms motivate us to strive for a possible approach where local Gram matrices are not required for collaborative parameter estimation in the federated linear contextual bandits setting. Towards that, we introduce the following two standard assumptions in the literature of linear contextual bandits.


\begin{assumption}[Context diversity{~\citep{hao2020adaptive}}]\label{asm:diversity}
For any client $i$, let $a_{c_{i}}^* = \arg\max_{a} \phi(c_i,a)^{\TT}\theta^*$ be the optimal arm under context $c_i$. Then, 
\[\lambda_{\min}\left(\Eb_{c_i\sim\rho_i} \left[\phi(c_i,a_{c_i}^*)\phi(c_i,a_{c_i}^*)^{\TT}\right]\right) \geq \lambda_0>0.\]
\end{assumption}
\Cref{asm:diversity} essentially indicates that the minimum eigenvalue of the expected Gram matrix under the {\it optimal} policy is bounded away from zero. 
Thus, it ensures that with high probability, every possible direction in the parameter space can be adequately explored under the optimal policy for each client $i$. Intuitively, if the optimal policy were known beforehand, each client could collect sufficient information in each dimension of the parameter space and obtain an estimate of $\theta$ with favorable accuracy guarantee in each dimension. Therefore, such estimates can be directly aggregated at the server with certain accuracy guarantee even without knowing the corresponding Gram matrices. This could potentially 
avoid adding strong noises to the Gram matrices and degrading the regret performance. Besides, as we discuss in \Cref{sec:diversity is necessary}, \Cref{asm:diversity} is actually necessary in order to achieve sublinear regret for a broad class of algorithms in the federated linear contextual bandits setting.

We remark that \Cref{asm:diversity} does not guarantee that the distributions of local datasets collected by the client are identical or even close in terms of total variation distance, since the covariance matrices still depend on how the clients are making decisions. Even if they focus on the optimal arms, the matrix $\Eb[\phi(c_{i,t},a_{c_{i,t}}^*)\phi(c_{i,t},a_{c_{i,t}}^*)^{\TT}]$ can still vary drastically across clients, as {different clients can potentially have very different context distributions}. This fact is in stark contrast to previous user-level DP works, e.g.,~\citet{levy2021learning}, and makes our problem more challenging.

\begin{assumption}[Margin condition~{\citep{rigollet2010nonparametric,reeve2018k}}]\label{asm:margin}
Let $a_{c_i}^* = \arg\max_{a} \phi(c_i,a)^{\TT}\theta^*$ be the optimal arm under context $c_i$. Then, there exists a constant $C_0$ such that, for any $\epsilon>0$ and any $i\in[M]$,
\[\mathbb{P}_{c_{i}\sim\rho_i} \left[ \forall a\neq a_{c_i}^*,\,[\phi(c_i,a_{c_i}^*) - \phi(c_i,a)]^{\TT}\theta^* \leq \epsilon \right] \leq C_0\epsilon.\]
\end{assumption}
Roughly speaking, \Cref{asm:margin} ensures that  {for a randomly generated context,} the expected reward under the corresponding optimal arm and that under any sub-optimal arm are statistically well separated. This is a standard assumption for {\it instance-dependent} analysis of linear contextual bandits. Similar to the minimum reward gap between the optimal arm and any sub-optimal arm in the stochastic MAB setting, $C_0$ controls the hardness of the problem: the larger $C_0$ is, the more challenging to distinguish the optimal arm and the second sub-optimal arm under a given context. In the following, we investigate upper bound {\it with} \Cref{asm:margin} and lower bounds {\it with} and {\it without} \Cref{asm:margin}.





\subsection{The \robin Algorithm}
In this section, we propose a gReedy explOitation Based prIvatized averagiNg (\texttt{ROBIN}) Algorithm under the federated linear contextual bandits setting. Our objective is to leverage \Cref{asm:diversity} and \Cref{asm:margin} to achieve sub-linear learning regret and guarantee user-level CDP at the same time.

\robin works in phases. In total, it has $P$ phases, and each phse $p\in[P]$ contains $2^p$ time indices. Denote $\Tc_p$ as the set of time indices in phase $p$. Then, it proceeds as follows.

\myparagraph{LinUCB-based Initialization:} In the first $U$ phases, each client performs the classical LinUCB algorithm~\citep{Abbasi:2011:IAL} locally. Specifically, at the beginning, client $i$ initializes an all-zero matrix $V_{i,1}$ and an all-zero vector $Y_{i,1}$. Then, at each time $t$ in phase $p\in U$, upon observing a context $c_{i,t}$ and the corresponding decision set $\mathcal{D}_{i,t} = \{\phi(c_{i,t},a): a\in\mathcal{A}\}$, the client $i$ chooses action $x_{i,t,a_{i,t}}$ according to 
\begin{align}
    x_{i,t,a_{i,t}} = \arg\max_{x\in\Dc_{i,t}}x^{\TT}\hat{\theta}_{i,t} + \alpha\|x\|_{(I_d + V_{i,t})^{-1}},\label{eqn:UCB selection}
\end{align}
where $\hat{\theta}_{i,t} = (I_d + V_{i,t})^{-1}Y_{i,t}$, and $\alpha$ is a parameter to be specified later. After receiving a reward $r_{i,t}$, each client updates matrix $V_{i,t}$ and vector $Y_{i,t}$ according to
\begin{align}
    \left\{
    \begin{aligned}\textstyle
    &V_{i,t+1} = V_{i,t} + x_{i,t,a_{i,t}}x_{i,t,a_{i,t}}^{\TT},\\
    &Y_{i,t+1} = Y_{i,t} + x_{i,t,a_{i,t}}r_{i,t}.
    \end{aligned}
    \right.\label{eqn:update data}
\end{align}

The benefit of such initialization is that the minimum eigenvalue of $V_{i,t}$ can be guaranteed to grow linearly in $t$. 

\myparagraph{Local Greedy Exploitation:} For any phase $p>U$ after the initialization phase, each client receives a
private global estimate $\hat{\theta}^{p}$ from the central server at the beginning of phase $p$. We will specify how to construct such a global estimate in the next several paragraphs. Meanwhile, each client resets matrix $V_{i,t}$ and vector $Y_{i,t}$ to be zero. Then, for all $t\in \Tc_{p}$, each client greedily takes actions with respect the global estimate, i.e. $a_{i,t} = \arg\max_a \phi(c_{i,t},a)^{\TT}\hat{\theta}^{p}$, and collects a reward $r_{i,t}$. Again, $V_{i,t}$ and $Y_{i,t}$ are updated according to \Cref{eqn:update data}. 

\myparagraph{Upload Channel $\Rt_i$: } At the end of each phase $p$ such that $p\geq U$, each client constructs a local estimator $\tilde{\theta}_{i,p} = \tilde{V}_{i,p}^{\dagger}\tilde{Y}_{i,p}$ based on the ordinary least squares method, where $\tilde{V}_{i,p}$ and $\tilde{Y}_{i,p}$ are copies of statistics $V_{i,t}$ and $Y_{i,t}$ at the end of phase $p$. These local estimators are then sent to the central server for privatized aggregation.

\myparagraph{Privatized Aggregation $\Rt_0$ at the Server:}
Once the local estimates $\{\tilde{\theta}_{i,p}\}_{i\in[M]}$ are received at the end of phase $p$, the server needs to aggregate those local estimates and obtain a global estimate $\hat{\theta}^{p+1}$, which will be broadcast to the clients to facilitate their decision-making in phase $p+1$. There are two objectives for this aggregation step: One one hand, due to the {\it user-level CDP} constraint, the corresponding aggregation needs to ensure that $\hat{\theta}^{p+1}$ does not change significantly if $\tilde{\theta}_{i,p}$ is replaced by another possible local estimate $\tilde{\theta}'_{i,p}$ for any $i\in[M]$. On the other hand, in order to achieve low learning regret, it requires that $\hat{\theta}^{p+1}$ is sufficiently close to the ground truth parameter $\theta^*$. In particular, it is desirable to have $\|\hat{\theta}^{p+1}-\theta^*\|$ scales in $\tilde{O}\left(1/\sqrt{M|\Tc_p|}\right)$ with high probability as $p$ increase.

However, in general, the sensitivity of $\frac{1}{M}\sum_i\tilde{\theta}_{i,p}$ scales in $O(1/M)$. 
This is because $\tilde{\theta}_{i,p}$ is a random variable, whose distribution expands the entire parameter space $\{\theta:\|\theta\|\leq 1\}$. Thus, if a vanilla additive noise mechanism is adopted, the variance of the noise should scale in $\Omega(1/M)$. As a result, $\|\hat{\theta}^{p+1}-\theta^*\|$ scales in $\Omega(1/\sqrt{M})$, which cannot provide the desired estimation accuracy in order to achieve sublinear regret. 

To overcome this challenge, we aim to leverage the concentration property of $\tilde{\theta}_{i,p}$ to have a more delicate analysis of the sensitivity of $\frac{1}{M}\sum_i\tilde{\theta}_{i,p}$. 
The intuition is, {under \cref{asm:diversity}}, $\tilde{\theta}_{i,p+1}$ will be concentrated in a ball centered at $\theta^*$ with radius {$\tilde{O}(1/\sqrt{|\Tc_p|})$} with high probability. Therefore, if we are able to leverage such concentration property and reduce the sensitivity of $\frac{1}{M}\sum_i\tilde{\theta}_{i,p}$, we will be able to adaptively reduce the variance of the additive noise in the DP mechanism and achieve sublinear regret.

Motivated by this intuition, we adopt the WinsorizedMeanHighD (\texttt{WMHD}) Algorithm from \citet{levy2021learning} for the private aggregation at the server. Roughly speaking, \WMHD projects local estimates $\{\tilde{\theta}_{i,p} \}_{i\in[M]}$ into a privatized range in the parameter space and then adds noise accordingly. By properly choosing the size of the range and the noise level, it outputs a privatized average of $\{\tilde{\theta}_{i,p} \}_{i\in[M]}$
with sufficient estimation accuracy and CDP guarantee. 
The details of \WMHD can be found in \Cref{alg:winsorizedHigh} in \Cref{appx:alg}.



The \robin algorithm is summarized in Algorithm~\ref{alg:greedy}.

\begin{algorithm}[t]
\small
\caption{The \robin  Algorithm}\label{alg:greedy}
\begin{algorithmic}[1]
\STATE {\bfseries Input:} $P$, $\beta$, $c_1$, $\delta_0 = \delta/(2P)$, $\varepsilon_0 = \varepsilon/\sqrt{6P\log(2/\delta)}$.
\WHILE{not reaching the time horizon $T$} 
\IF{$p\leq U$}
\STATE $\triangleright$ \texttt{Initialization phases}
\STATE $D_{i,p} = \emptyset$
\FOR{each client $i$ and $t$ in Phase $p$}
\STATE Observe $c_{i,t}$ 
\STATE Choose $x_{i,t,a_{i,t}}$ according to Eqn. (\ref{eqn:UCB selection}), and receive $r_{i,t}$
\STATE Update $V_{i,t}$ and $Y_{i,t}$ according to Eqn. (\ref{eqn:update data})

\ENDFOR

\ELSE
\STATE $\triangleright$ {\texttt{Greedy exploitation}}
\FOR{each client $i$ }
\STATE Receive $\hat{\theta}^{(p)}$ from the server
\STATE Set $V_{i,t}=0, Y_{i,t}=0$
\FOR{$t\in$ Phase $p$}
\STATE Receive decision set $\mathcal{D}_{i,t} = \{x_{i,t,a}\}_{a\in\mathcal{A}}$
\STATE Pull arm $x_{i,t,a_t}= \arg\max_{x\in \Dc_{i,t}} x^{\TT}\hat{\theta}^{(p)}$
\STATE Receive reward $r_{i,t}$
\STATE Update $V_{i,t}$ and $Y_{i,t}$ according to Eqn.~(\ref{eqn:update data})
\ENDFOR
\ENDFOR

\ENDIF

\IF{$p\geq U$} 

\FOR{each client $i$}
\STATE $\triangleright$ \texttt{Upload channel:} $\Rt_i$
\STATE $\tilde{\theta}_{i,p}\leftarrow \tilde{V}_{i,p}^{\dagger}\tilde{Y}_{i,p}$
\STATE Send $\tilde{\theta}_{i,p}$ to the server
\ENDFOR

\STATE $\triangleright$ \texttt{Privatized aggregation:} $\Rt_0$
\STATE $\hat{\theta}^{p+1} = $  \WMHD$\left(\{\tilde{\theta}_{i,p}\}_{i\in[M]},\frac{c_1}{\sqrt{|\Tc_p|}}, \frac{\beta}{16P},\varepsilon_0, \delta_0  \right)$
\ENDIF
\STATE $p\leftarrow p +1$
\ENDWHILE
\end{algorithmic}
\end{algorithm}

\subsection{Regret Analysis}

\begin{theorem}\label{thm:upper bound}
Fix $\beta\in(0,1)$ and let $c_1 = \frac{4\sqrt{2d\log(16d(M+1)P/\beta)}}{\lambda_0}$. Then, under \Cref{asm:diversity,asm:margin}, when $ \varepsilon \geq \Omega(\frac{\sqrt{d}\log^{1.5} T}{M})$,
\Cref{alg:greedy} (i) satisfies user-level $(\varepsilon,\delta)$-CDP, and (ii) with probability at least $1-\beta$, achieves a regret upper bounded by
    \begin{align*}
       \tilde{O}\left(\max\left(1, \frac{d\log T\log(\frac{1}{\delta})\log^3(\frac{1}{\beta})}{M\varepsilon^2 }\right)\frac{C_0d\log(\frac{1}{\beta})\log T}{\lambda_0^2}\right).
    \end{align*}
\end{theorem}

{\it Proof of the user-level CDP guarantee:} Since $\hat{\theta}^p$ is a  $(\varepsilon_0,\delta_0)$-differentially private estimation (Theorem 2 in \citet{levy2021learning}), and there are total $P$ phases, by the advanced composition rule (\cref{lemma:advanced composition}), we conclude that the entire algorithm achieves user-level $(\varepsilon_0\sqrt{6P\log(1/\delta')}, \delta' + \delta_0 P)$-CDP for any $\delta'>0$. The proof is finished by noting that $\delta_0 = \delta/(2P)$, $\varepsilon_0 = \varepsilon/\sqrt{6P\log(2/\delta)}$ and choosing $\delta' = \delta/2$.

\begin{proof}[Proof sketch of the regret upper bound] To upper bound the regret, we need to characterize the estimation error of the global estimator $\hat{\theta}^{p}$ for $p> U$. We inductively show that the following two claims hold with high probability. 
\begin{itemize}[leftmargin=*]\itemsep=0pt
    \item Claim 1: $\lambda_{\min} (\tilde{V}_{i,p}) \geq \Omega\left(|\Tc_p|\right)$.
    \item Claim 2: $\|\hat{\theta}^{p+1} - \theta^*\| \leq \tilde{O}\left(\frac{1}{\varepsilon M\sqrt{|\Tc_p|}} + \frac{1}{\sqrt{M|\Tc_p|}}\right).$
\end{itemize}

If Claim 1 holds, it is straightforward to conclude that each local estimator is sufficiently accurate, i.e. $\|\tilde{\theta}_{i,p} - \theta^*\|\leq \tilde{O}\left(1/\sqrt{|\Tc_p|}\right).$ Then, due to Theorem 2 in \citet{levy2021learning}, \WMHD guarantees that $\hat{\theta}^{p+1}$ is ``close'' to the average of local estimators $\{\tilde{\theta}_{i,p}\}_{i\in[M]}$. Specifically, we can show that with high probability, 
\begin{equation*}\textstyle
\left\|\hat{\theta}^{p+1} - \frac{1}{M}\sum_{i\in[M]}\tilde{\theta}_{i,p} \right\|\leq \tilde{O}\left(\frac{1}{\varepsilon M\sqrt{|\Tc_{p}|}}\right).
\end{equation*}

Thus, Claim 2 follows by applying Claim 1 on the average estimator, i.e. $\|\frac{1}{M}\sum_i\tilde{\theta}_{i,p} - \theta^*\|\leq \tilde{O}\left(1/\sqrt{M|\Tc_p|}\right).$ 

If Claim 2 holds, combining with \Cref{asm:margin}, with high probability, the covariance matrix $\sum_{t\in\mathcal{T}_p}x_{i,t,a_{i,t}}x_{i,t,a_{i,t}}^{\TT}$ is close to $|\mathcal{T}_p|\Eb_{c_i\sim\rho_i}\left[\phi(c_i,a_{c_{i}}^*)\phi(c_i,a_{c_i}^*)^{\TT}\right]$. Hence, Claim 1 follows directly from \Cref{asm:diversity}.

Finally, the LinUCB algorithm in the first $U$ phases guarantees that Claim 1 holds with high probability, and by induction, both claims hold for all phases $p\geq U$. Therefore, the regret can be upper bounded by the product of estimation error and the probability of playing sub-optimal arms. Due to \Cref{asm:margin}, the probability is again controlled by the estimation error, and thus the regret upper bound is $\tilde{O}\left((1+1/(M\varepsilon^2))\log T\right).$
\end{proof}
The full proof can be found in \Cref{sec:greedy}.

\begin{remark}
In general, the difficulty of deriving regret upper bounds without \Cref{asm:margin} is due to the weak diversity Assumption~\ref{asm:diversity}, where we only require the covariance matrix associated with {\it the optimal arm } to be sufficiently diverse. Therefore, without \Cref{asm:margin}, if the sub-optimal gap is too small, it is highly likely that the feature vector associated with the decision of the agent is very different from the optimal feature vector, although the corresponding reward (inner product of the feature vector and $\theta$) is close to the optimal reward. If this is the case, then the agent cannot leverage \Cref{asm:diversity} to construct an accurate estimation of $\theta$, as \Cref{asm:diversity} only applies to the feature vectors associated with the optimal arm.
\end{remark}

\section{Lower Bounds under CDP Constraint}

In this section, we present regret lower bounds for any {\it user-level central differentially private} federated algorithms. All proofs in this section can be found in \Cref{sec:proof CDP lower bound}..

To be more precise, we separate the federated algorithms into two categories: almost-memoryless algorithms and with-memory algorithms. Without additional assumptions, the general framework defined in \Cref{sec: DP framework} is a with-memory algorithm, where the decision-making algorithm $\mathtt{Alg}_i$ depends on both the local history and the global information $q_{\leq t}$. If the available information is restricted, we define almost-memoryless algorithm as follows. 
\begin{definition}[Almost-memoryless algorithm]\label{def:memoryless}
  A federated algorithm $(\Rt_0,\Alg_1,\Rt_1,\ldots,\Rt_M)$ is almost {\bf memoryless} if there exists a constant $u=o(T)$ such that for any time step $t\geq u$, the decision-making does not depend on local history data, i.e. $\Alg_i(H_{i,t},q_{\leq t}) = \Alg_i(q_{\leq t})$, for any $ i\in[M]$ and $ H_{i,t}$. 
  \end{definition}

We note that a typical memoryless algorithm is phased elimination type of algorithms \citep{shi2021federated,shi2021aistats,huang2021federated}, where the policy in a single phase does not change and only depends on the information broadcast from the last communication round. Moreover, \Cref{alg:greedy} proposed in this work is also almost-memoryless. 


\begin{theorem}[]\label{thm:main-CDP-lower-margin}
    If $\varepsilon<\log 2$, $\delta=\tilde{O}(\frac{1}{M\sqrt{T}})$, then, {there exists a federated linear contextual bandits instance satisfying  \Cref{asm:diversity,asm:margin}, such that} any {\bf almost-memoryless} federated algorithm satisfying user-level $(\varepsilon,\delta)$-CDP must incur a regret lower bounded by
    \begin{equation*}\textstyle
        \Omega\left(\max\left\{1, \frac{1}{M\varepsilon^2}\right\}C_0d\log T  + e^{-M\varepsilon}C_0MT\right).
    \end{equation*}
\end{theorem}



\begin{remark}
First, we note that in the regime $\varepsilon = \Omega\left(\frac{\sqrt{d}\log^{1.5} T}{M}\right)$ considered in \Cref{thm:upper bound}, the first term of the lower bound dominates the other. Thus, the regret upper bound under \robin nearly matches with the lower bound, indicating that \robin is near-optimal in terms of $M$ and $\varepsilon$ in this regime. On the other hand, when $\varepsilon = \tilde{O}\left( \frac{\log (MT)}{M} \right)$, the second term  $e^{-M\varepsilon}C_0MT$ dominates the first term and grows linearly in $T$. Thus, it is impossible to achieve sublinear regret for any almost-memoryless algorithm in this regime.  Since \robin is an almost-memoryless algorithm, it arguably achieves the best we could hope for among all almost-memoryless algorithms. 

The lower bound also indicates the number of samples contributed by each client (i.e., $T$) cannot be arbitrarily large (i.e., cannot scale faster than $\tilde{O}(e^{M\varepsilon}/M)$) in order to achieve sublinear regret for any almost-memoryless algorithm. 
A similar phenomenon in offline supervised learning is also observed in \citet{levy2021learning}, which states that learning with user-level DP cannot reach zero error when the number of clients is fixed. 
\end{remark}

\if{0}
\begin{corollary}
    If $M\varepsilon\leq 0.5\log T$, and $\delta\leq 1/(48M)$, then any federated algorithm satisfying $(\varepsilon,\delta)$-CDP must incur linear regret $\Omega(M\sqrt{T})$.
\end{corollary}

The corollary suggests that under the user-level $(\log T/(2M),O(1/M))$-CDP constraint, the best federated algorithm is simply that each client makes decisions with its own data, without sharing any information. 
\fi

 The proof of \Cref{thm: Regret to estimation error} is built upon a generic lower bound developed by \citet{he2022reduction}, as informally stated in \Cref{thm: Regret to estimation error}, and the fingerprinting lemma~\cite{kamath2019privately, bun2017make}. 

\begin{theorem}[Informal]\label{thm: Regret to estimation error}
    Let $\theta_1^*$ be uniformly sampled from a two-dimensional sphere $\Theta = \{x\in\Rb^2:\|x\| = r\}$. Then, there exists a federated linear contextual bandits model such that the total regret is lower bounded by 
    \begin{align*}\textstyle
    \Omega \left(\sum_{i\in[M],t\in[T]}\inf_{\theta_{i,t}\in\mathcal{F}(\mathcal{I}_i,\Theta)}\frac{1}{r}\Eb_v\left[ \left\|\theta^*_1 - \theta_{i,t}\right\|^2 \right] \right),
    \end{align*}
    where $\mathcal{I}_i$ is a set of available information provided for client $i$ (e.g., $(H_{i,t},q_{\leq t})$). 
\end{theorem}
\begin{proof}[Proof sketch of \Cref{thm:main-CDP-lower-margin}]
\Cref{thm: Regret to estimation error} states that the regret is lower bounded by the performance of estimating the ``direction'' of the true parameter $\theta^*$. Note that in general, the error of estimating direction cannot be directly lower bounded  by the standard estimation error, especially when the norm of parameters are $\Omega(1)$. However, under the margin condition in~\Cref{asm:margin}, the norm of $\theta^*$ is in general $O(1)$. Hence, to lower bound the estimation error of the direction, we follow an alternative approach that re-parameterizes $\theta^*_s$ by its angle $\gamma_s^*$ and provide an upper bound of the expected inner product $\Eb_{v}[\theta_{i,t,s}^{\TT}\theta_s^*]$. {By leveraging the fingerprinting lemma,} we show that if $\theta_{i,t,s}$ is a private estimator, then, $\Eb_{v}[\theta_{i,t,s}^{\TT}\theta_s^*]\leq O(\varepsilon M\sqrt{(t-1)\Eb_v[\|\theta_s^* - \theta_{i,t,s}\|^2]})$. Combining with the fact that $\Eb_v[\|\theta_s^* - \theta_{i,t,s}\|^2] = 2 - 2\Eb_v[\theta_{i,t,s}^{\TT}\theta_s^*]$, we conclude that $\Eb_v[\|\theta_s^* - \theta_{i,t,s}\|^2]\geq \Omega(1/(\varepsilon^2 M^2(t-1) + 1))$. Taking the summation over $t$ and $i$, and by \Cref{thm: Regret to estimation error}, we obtain a lower bound $\Omega(\log T/(M\varepsilon^2))$. {The second term $e^{-M\varepsilon}C_0MT$ can be derived by calculating the estimation error under the case when all clients collect dummy information (e.g. $x_{i,t,a_t}=0$).} The final result then follows by taking the non-private regret lower bound $\Omega(\log T)$ into consideration. 
\end{proof}

Our results can be generalized to obtain minimax lower bounds for the with-memory algorithms by optimizing the norm $r$ of the true parameter or separately analyze the information contained in local datasets. The results are summarized as follows.

\begin{theorem}\label{thm:main-CDP-minimax}
    Fix any $\varepsilon\in(0,\log 2)$, $\delta=\tilde{O}\left(\frac{1}{M\sqrt{T}}\right)$ , $T\geq d^2$. Then, there exists a federated linear contextual bandits instance satisfying \Cref{asm:diversity,asm:margin} such that any {\bf with-memory} federated algorithm satisfying user-level $(\varepsilon,\delta)$-CDP must incur a regret lower bounded by
    \begin{align*}\textstyle
    &\Omega \left(\min\left\{M, \max \left\{1, \frac{1}{M\varepsilon^2 } \right\} \right\}C_0d\log T  \right).
\end{align*}
If \Cref{asm:margin} is not satisfied, then the minimax regret lower bound becomes 
\begin{align*}\textstyle
    \Omega\left( \min\left\{M, \max \left\{\sqrt{M}, \frac{1}{\varepsilon} \right\}\right\}\sqrt{dT}\right).
\end{align*}
\end{theorem}
\begin{remark}
{Since with-memory algorithms encompass almost-memoryless algorithms as special cases,} the first part of \Cref{thm:main-CDP-minimax} verifies that \robin is nearly optimal compared to {\it any} algorithms when $\varepsilon = \Omega\left(\frac{\sqrt{d}\log^{1.5} T}{M}\right)$.  In fact, \robin can be easily adapted to a with-memory algorithm to remove the constraint on the parameters. Specifically, if at the beginning of the entire algorithm, \robin is allowed to adaptively decides that clients follow \Cref{alg:greedy} when $\varepsilon = \Omega(\frac{\sqrt{d}\log^{1.5} T}{M})$, or independently adopt LinUCB without sharing any information when $\varepsilon = O( \frac{\sqrt{d}\log^{1.5} T}{M})$, then, the algorithm achieves an upper bound that nearly matches with the lower bound  
in terms of $M$ and $\varepsilon$ {for any $\varepsilon\in (0,\log 2)$}. In addition, this adaptation preserves the user-level CDP guarantee, since performing LinUCB locally without information exchange guarantees zero information leakage.

\Cref{thm:main-CDP-minimax} also indicates that, when $\varepsilon=\Omega(\frac{1}{\sqrt{M}})$, the lower bound reduces to $\Omega(C_0d\log T)$ with \Cref{asm:margin} and $\Omega(\sqrt{dMT})$ without \Cref{asm:margin}. This suggests that imposing user-level $(\varepsilon,\delta)$-CDP constraint does not increase the hardness of learning compared with its non-private counterpart for the general with-memory algorithms in this regime. However, when $\varepsilon = O(\frac{1}{\sqrt{M}})$, imposing the CDP constraint incurs a blow-up factor at least $\min\{M, \frac{1}{\varepsilon^2M}\}$  with \Cref{asm:diversity} or $\min\{\sqrt{M}, \frac{1}{\varepsilon\sqrt{M}} \} $  without \Cref{asm:diversity}, indicating the hardness of learning strictly increases.

\end{remark}


\section{Lower Bounds under LDP Constraint}\label{sec:ldp}
In this section, we present regret lower bounds under the non-interactive local differential privacy constraint, which provides an initial view of this more challenging problem. The full proofs in this section can be found in \Cref{sec:proof of LDP lower bound}.

\begin{definition}[Non-interactive upload channel] 
The upload channels $\Rt_i$ of a federated algorithm $\mathtt{Alg} = (\mathtt{R}_0, \mathtt{Alg}_1,\mathtt{R}_1,\ldots,\mathtt{Alg}_M,\mathtt{R}_M)$ is non-interactive if
 $\Rt_i$ does not depend on global information $q_{\leq t-1}$ conditioned on the local history $H_{i,t}$, i.e. $\Rt_i(H_{i,t},q_{\leq t-1}) = \Rt_i(H_{i,t})$.
\end{definition}
We note that existing phased elimination-based federated bandits algorithms are non-interactive~\cite{shi2021federated,shi2021aistats,huang2021federated}. In contrary, UCB-type algorithms are interactive and with-memory~\cite{dubey2020differentially,li2022communication,li2020federated}.

\begin{definition}[Non-interactive user-level local DP]\label{def:ldp}
Consider a time horizon $T$.
A federated algorithm $\mathtt{Alg} = (\mathtt{R}_0, \mathtt{Alg}_1,\mathtt{R}_1,\ldots,\mathtt{Alg}_M,\mathtt{R}_M)$ is user-level $(\varepsilon,\delta)$-locally differentially private if for any {$i$-neighboring streaming datasets $\{H_t\}_{t\leq T}$ and $\{H_t'\}_{t\leq T}$} (see \Cref{def:neighbor dataset}), 
    and any subset $Q_{i,\leq T}=(Q_{i,1},\ldots,Q_{i,T}) \subset\mathcal{Q}^{T}$, we have
\begin{align*}\textstyle
    &\Pb[ \mathtt{R}_i(\{H_{i,t}\}_{t\leq T})\in Q_{i,\leq T}] \\
    &\quad\leq e^{\varepsilon}\Pb[\mathtt{R}_i(\{H'_{i,t}\}_{t\leq T})\in Q_{i,\leq T}] + \delta.
\end{align*}
\end{definition}

We focus on the with-memory setting, since any lower bound of with-memory algorithms must be a lower bound of memoryless algorithms. 

\begin{theorem}\label{thm:main-LDP}
     If $\varepsilon\in(0,\log 2)$, $\delta=\tilde{O}(1/M\sqrt{T})$, there exists a federated linear contextual bandits instance satisfying \Cref{asm:diversity,asm:margin} such that any {\bf with-memory} federated algorithm satisfying user-level $(\varepsilon,\delta)$-LDP must incur a regret lower bounded by
    \begin{equation*}\textstyle
        \Omega\left( \min\left\{1/\varepsilon, M \right\}C_0d\log T\right).
    \end{equation*}
    If \Cref{asm:margin} is not satisfied, then the minimax regret lower bound becomes 
    \begin{equation*}\textstyle
        \Omega\left(\min \left\{\sqrt{M/\varepsilon},M \right\}\sqrt{dT}\right).
    \end{equation*}
    
    \end{theorem}
\begin{remark}
We note that when $\varepsilon = O(1/M)$, the regret lower bound is either $\Omega(M\log T)$ under \Cref{asm:margin} or $\Omega(M\sqrt{T})$ without \Cref{asm:margin}, suggesting that the best policy is to have clients independently make arm-pulling decisions without information exchange. When $\varepsilon = \Omega(1/M)$, the lower bound becomes $\Omega(C_0d\log T/\varepsilon)$ under \Cref{asm:margin}, or $\Omega(\sqrt{dMT/\varepsilon})$ without \Cref{asm:margin}. This indicates that under the $(\varepsilon,\delta)$-LDP constraint, as long as $\varepsilon\in(0,\log 2)$, the regret of any federated algorithm must suffer a blow-up factor at least $\min\{1/\sqrt{\varepsilon}, \sqrt{M}\}$ without \Cref{asm:margin}, or $\min\{1/\varepsilon, M\}$ with \Cref{asm:margin}, compared with the optimal regrets in the non-private setting. 
\end{remark}
For completeness, we also investigate the regret lower bound under user-level pure LDP constraint, i.e. $\delta=0$. 

\begin{corollary}\label{coro:pure LDP lower bound}
    For any $\varepsilon\in(0,\log 2)$, there exists a federated linear contextual bandits instance satisfying \Cref{asm:diversity,asm:margin} such that any {\bf with-memory} federated algorithm satisfying $\varepsilon $-LDP must incur a regret lower bounded by
    \begin{equation*}\textstyle
        \Omega\left( \min\left\{M, 1/\varepsilon^2  \right\}C_0d\log T\right).
    \end{equation*}
    If \Cref{asm:margin} is not satisfied, then the minimax regret lower bound becomes 
    \begin{equation*}\textstyle
        \Omega\left(\min \left\{M, \sqrt{M}/\varepsilon \right\} \sqrt{dT} \right).
    \end{equation*}

\end{corollary}

Compared with the results in \Cref{thm:main-CDP-minimax}, we note that the regret lower bound under pure LDP constraint 
is generally higher than that under CDP constraint. {Specifically, the results in \Cref{coro:pure LDP lower bound} can be obtained by replacing $\varepsilon$ in \Cref{thm:main-CDP-minimax} by $\varepsilon/\sqrt{M}$\footnote{
 Another lens to see this phenomenon is the privacy amplification by shuffling~\cite{feldman2022hiding}. Loosely speaking, an $(\varepsilon/\sqrt{M},\delta)$-CDP lower bound leads to an $(\varepsilon,\delta)$-LDP lower bound, when $\varepsilon$ is sufficiently small~\cite{acharya2022discrete}.}. }

\section{Related Work}

\textbf{Differential Privacy in Bandits.} There is a line of research focusing on differentially private multi-armed bandits (DP-MAB). \citet{mishra2015nearly} first introduce the problem of DP-MAB with algorithms that achieve sublinear regret. Later, \citet{tossou2016algorithms,sajed2019optimal,azize2022privacy} improve the analysis and propose several different algorithms that enjoy the optimal regret. In addition, \citet{hu2022near} achieve similar near-optimal regret based on a Thompson-sampling based approach, \citet{tao2022optimal} consider heavy-tailed rewards case, and \citet{chowdhury2022distributed} provide an optimal regret in distributed DP-MAB. Local DP constraint is also studied by \citet{ren2020multi} in MAB and by \citet{zheng2020locally} in both MAB and linear contextual bandits. \citet{shariff2018differentially} propose LinUCB with changing perturbation to satisfy jointly differential privacy. Later, \citet{wang2020global} consider pure DP in both global and local setting. \citet{wang2022dynamic} propose an algorithm with dynamic global sensitivity. Other models including linear and generalized linear bandits under DP constraints are studied by \citet{hanna2022differentially} and \citet{han2021generalized}. 
The shuffle model has also been addressed by \citet{chen2020locally,chowdhury2022shuffle,garcelon2022privacy,tenenbaum2021differentially}.

\textbf{Federated Bandits.}
There is a growing body of research studying item-level DP based local data privacy protection in federated bandits. \citet{li2020federated,Zhu_2021} study federated bandits with DP guarantee. \citet{dubey2022private} consider private and byzantine-proof cooperative decision making in multi-armed bandits. \citet{dubey2020differentially,zhou2023differentially} consider the linear contextual bandit model with joint DP guarantee.
\citet{li2022differentially} study private distributed bandits with partial feedback. 

Federated bandits without explicit DP constraints have also been studied by~\citet{wang2019distributed,li2022asynchronous,shi2021federated,shi2021aistats,huang2021federated,wang2022federated,he2022simple,li2022federated}.

{\bf User-level DP.} First introduced by \citet{dwork2010pan}, user-level DP has attracted increased attention recently~\citep{mcmahan2017learning,wang2019beyond}. 
\citet{liu2020learning,acharya2022discrete} study discrete distribution estimation under user-level DP. \citet{ghazi2021user} investigate the number of users required for learning under user-level DP constraint. \citet{amin2019bounding,epasto2020smoothly} study approaches to bound individual users' contributions in order to achieve good trade-off between utility and user-level privacy guarantee. \citet{levy2021learning} consider various learning tasks, such as mean estimation, under user-level DP constraint. To the best of our knowledge, user-level DP has not been studied in the online learning setting before.




\vspace{-0.1in}
\section{Conclusions}
\vspace{-0.05in}
In this paper, we investigated federated linear contextual bandits under user-level differential privacy constraints. We first introduced a general federated sequential decision-making framework that can accommodate various notations of DP in the bandits setting. We then proposed an algorithm termed as \robin and showed that it is near-optimal under the user-level CDP constraint. We further provided various lower bounds for federated algorithms under user-level LDP constraint, which imply that learning under LDP constraint is strictly harder than the non-private case. Designing federated algorithms to approach such lower bounds under LDP constraint would be an interesting direction to pursue in the future.

\section*{Acknowledgments}
The authors would like to thank Ilya Mironov and Abhimanyu Dubey for the helpful and inspiring discussions. 

The work of R. Huang and J. Yang was supported in part by the U.S. National Science Foundation under grants 2030026, 2114542, 2133170 and a gift from Meta. 

The work of Meisam Hejazinia was performed when he was at Meta.


\bibliographystyle{icml2023}
\bibliography{arxiv_version}

\begin{thebibliography}{70}
\providecommand{\natexlab}[1]{#1}
\providecommand{\url}[1]{\texttt{#1}}
\expandafter\ifx\csname urlstyle\endcsname\relax
  \providecommand{\doi}[1]{doi: #1}\else
  \providecommand{\doi}{doi: \begingroup \urlstyle{rm}\Url}\fi

\bibitem[Abbasi-Yadkori et~al.(2011)Abbasi-Yadkori, P\'{a}l, and
  Szepesv\'{a}ri]{Abbasi:2011:IAL}
Abbasi-Yadkori, Y., P\'{a}l, D., and Szepesv\'{a}ri, C.
\newblock Improved algorithms for linear stochastic bandits.
\newblock In \emph{Proceedings of the 24th International Conference on Neural
  Information Processing Systems}, pp.\  2312--2320, 2011.

\bibitem[Acharya et~al.(2021)Acharya, Sun, and Zhang]{acharya21}
Acharya, J., Sun, Z., and Zhang, H.
\newblock Differentially private {A}ssouad, {F}ano, and {L}e {C}am.
\newblock In \emph{Algorithmic Learning Theory}, pp.\  48--78. PMLR, 2021.

\bibitem[Acharya et~al.(2022)Acharya, Liu, and Sun]{acharya2022discrete}
Acharya, J., Liu, Y., and Sun, Z.
\newblock Discrete distribution estimation under user-level local differential
  privacy.
\newblock \emph{arXiv preprint arXiv:2211.03757}, 2022.

\bibitem[Agrawal \& Goyal(2012)Agrawal and Goyal]{agrawal2012analysis}
Agrawal, S. and Goyal, N.
\newblock Analysis of {Thompson} sampling for the multi-armed bandit problem.
\newblock In \emph{Conference on Learning Theory}, pp.\  39--1, 2012.

\bibitem[Agrawal \& Goyal(2013)Agrawal and Goyal]{agrawal2013further}
Agrawal, S. and Goyal, N.
\newblock Further optimal regret bounds for {Thompson} sampling.
\newblock In \emph{Artificial Intelligence and Statistics}, pp.\  99--107,
  2013.

\bibitem[Amin et~al.(2019)Amin, Kulesza, Munoz, and
  Vassilvtiskii]{amin2019bounding}
Amin, K., Kulesza, A., Munoz, A., and Vassilvtiskii, S.
\newblock Bounding user contributions: A bias-variance trade-off in
  differential privacy.
\newblock In \emph{International Conference on Machine Learning}, pp.\
  263--271. PMLR, 2019.

\bibitem[Asoodeh et~al.(2021)Asoodeh, Aliakbarpour, and
  Calmon]{asoodeh2021local}
Asoodeh, S., Aliakbarpour, M., and Calmon, F.~P.
\newblock Local differential privacy is equivalent to contraction of $ e_\gamma
  $-divergence.
\newblock \emph{arXiv preprint arXiv:2102.01258}, 2021.

\bibitem[Auer et~al.(2002)Auer, Cesa-Bianchi, and Fischer]{auer2002finite}
Auer, P., Cesa-Bianchi, N., and Fischer, P.
\newblock Finite-time analysis of the multiarmed bandit problem.
\newblock \emph{Machine learning}, 47\penalty0 (2-3):\penalty0 235--256, 2002.

\bibitem[Azize \& Basu(2022)Azize and Basu]{azize2022privacy}
Azize, A. and Basu, D.
\newblock When privacy meets partial information: A refined analysis of
  differentially private bandits.
\newblock \emph{arXiv preprint arXiv:2209.02570}, 2022.

\bibitem[Bubeck \& Cesa-Bianchi(2012)Bubeck and Cesa-Bianchi]{bubeck2012regret}
Bubeck, S. and Cesa-Bianchi, N.
\newblock Regret analysis of stochastic and nonstochastic multi-armed bandit
  problems.
\newblock \emph{Foundations and Trends{\textregistered} in Machine Learning},
  5\penalty0 (1):\penalty0 1--122, 2012.

\bibitem[Bun et~al.(2017)Bun, Steinke, and Ullman]{bun2017make}
Bun, M., Steinke, T., and Ullman, J.
\newblock Make up your mind: The price of online queries in differential
  privacy.
\newblock In \emph{Proceedings of the Twenty-Eighth Annual ACM-SIAM Symposium
  on Discrete Algorithms}, pp.\  1306--1325. SIAM, 2017.

\bibitem[Carpentier et~al.(2020)Carpentier, Vernade, and
  Abbasi-Yadkori]{carpentier2020elliptical}
Carpentier, A., Vernade, C., and Abbasi-Yadkori, Y.
\newblock The elliptical potential lemma revisited.
\newblock \emph{arXiv preprint arXiv:2010.10182}, 2020.

\bibitem[Chen et~al.(2020)Chen, Zheng, Zhou, Yang, Chen, and
  Wang]{chen2020locally}
Chen, X., Zheng, K., Zhou, Z., Yang, Y., Chen, W., and Wang, L.
\newblock ({L}ocally) differentially private combinatorial semi-bandits.
\newblock In \emph{International Conference on Machine Learning}, pp.\
  1757--1767. PMLR, 2020.

\bibitem[Chowdhury \& Zhou(2022{\natexlab{a}})Chowdhury and
  Zhou]{chowdhury2022distributed}
Chowdhury, S.~R. and Zhou, X.
\newblock Distributed differential privacy in multi-armed bandits.
\newblock \emph{arXiv preprint arXiv:2206.05772}, 2022{\natexlab{a}}.

\bibitem[Chowdhury \& Zhou(2022{\natexlab{b}})Chowdhury and
  Zhou]{chowdhury2022shuffle}
Chowdhury, S.~R. and Zhou, X.
\newblock Shuffle private linear contextual bandits.
\newblock \emph{arXiv preprint arXiv:2202.05567}, 2022{\natexlab{b}}.

\bibitem[Chu et~al.(2011)Chu, Li, Reyzin, and Schapire]{chu2011contextual}
Chu, W., Li, L., Reyzin, L., and Schapire, R.
\newblock Contextual bandits with linear payoff functions.
\newblock In \emph{Proceedings of the Fourteenth International Conference on
  Artificial Intelligence and Statistics}, pp.\  208--214. JMLR Workshop and
  Conference Proceedings, 2011.

\bibitem[Cummings et~al.(2021)Cummings, Feldman, McMillan, and
  Talwar]{cummings2021mean}
Cummings, R., Feldman, V., McMillan, A., and Talwar, K.
\newblock Mean estimation with user-level privacy under data heterogeneity.
\newblock In \emph{Advances in Neural Information Processing Systems}, 2021.

\bibitem[Den~Hollander(2012)]{den2012probability}
Den~Hollander, F.
\newblock Probability theory: The coupling method.
\newblock \emph{Lecture notes available online (http://websites. math.
  leidenuniv. nl/probability/lecturenotes/CouplingLectures. pdf)}, 2012.

\bibitem[Dubey \& Pentland(2020)Dubey and Pentland]{dubey2020differentially}
Dubey, A. and Pentland, A.
\newblock Differentially-private federated linear bandits.
\newblock \emph{Advances in Neural Information Processing Systems},
  33:\penalty0 6003--6014, 2020.

\bibitem[Dubey \& Pentland(2022)Dubey and Pentland]{dubey2022private}
Dubey, A. and Pentland, A.
\newblock Private and {Byzantine}-proof cooperative decision-making.
\newblock \emph{arXiv preprint arXiv:2205.14174}, 2022.

\bibitem[Duchi et~al.(2013)Duchi, Jordan, and Wainwright]{duchi2013local}
Duchi, J.~C., Jordan, M.~I., and Wainwright, M.~J.
\newblock Local privacy, data processing inequalities, and minimax rates.
\newblock \emph{arXiv preprint arXiv:1302.3203}, 2013.

\bibitem[Dwork et~al.(2010{\natexlab{a}})Dwork, Naor, Pitassi, and
  Rothblum]{dwork2010differential}
Dwork, C., Naor, M., Pitassi, T., and Rothblum, G.~N.
\newblock Differential privacy under continual observation.
\newblock In \emph{Proceedings of the forty-second ACM symposium on Theory of
  computing}, pp.\  715--724, 2010{\natexlab{a}}.

\bibitem[Dwork et~al.(2010{\natexlab{b}})Dwork, Naor, Pitassi, Rothblum, and
  Yekhanin]{dwork2010pan}
Dwork, C., Naor, M., Pitassi, T., Rothblum, G.~N., and Yekhanin, S.
\newblock Pan-private streaming algorithms.
\newblock In \emph{ics}, pp.\  66--80, 2010{\natexlab{b}}.

\bibitem[Dwork et~al.(2014)Dwork, Roth, et~al.]{dwork2014algorithmic}
Dwork, C., Roth, A., et~al.
\newblock The algorithmic foundations of differential privacy.
\newblock \emph{Foundations and Trends{\textregistered} in Theoretical Computer
  Science}, 9\penalty0 (3--4):\penalty0 211--407, 2014.

\bibitem[Epasto et~al.(2020)Epasto, Mahdian, Mao, Mirrokni, and
  Ren]{epasto2020smoothly}
Epasto, A., Mahdian, M., Mao, J., Mirrokni, V., and Ren, L.
\newblock Smoothly bounding user contributions in differential privacy.
\newblock \emph{Advances in Neural Information Processing Systems},
  33:\penalty0 13999--14010, 2020.

\bibitem[Feldman \& Steinke(2017)Feldman and
  Steinke]{feldman2017generalization}
Feldman, V. and Steinke, T.
\newblock Generalization for adaptively-chosen estimators via stable median.
\newblock In \emph{Conference on Learning Theory}, pp.\  728--757. PMLR, 2017.

\bibitem[Feldman et~al.(2022)Feldman, McMillan, and Talwar]{feldman2022hiding}
Feldman, V., McMillan, A., and Talwar, K.
\newblock Hiding among the clones: A simple and nearly optimal analysis of
  privacy amplification by shuffling.
\newblock In \emph{2021 IEEE 62nd Annual Symposium on Foundations of Computer
  Science (FOCS)}, pp.\  954--964. IEEE, 2022.

\bibitem[Garcelon et~al.(2022)Garcelon, Chaudhuri, Perchet, and
  Pirotta]{garcelon2022privacy}
Garcelon, E., Chaudhuri, K., Perchet, V., and Pirotta, M.
\newblock Privacy amplification via shuffling for linear contextual bandits.
\newblock In \emph{International Conference on Algorithmic Learning Theory},
  pp.\  381--407. PMLR, 2022.

\bibitem[Ghazi et~al.(2021)Ghazi, Kumar, and Manurangsi]{ghazi2021user}
Ghazi, B., Kumar, R., and Manurangsi, P.
\newblock User-level differentially private learning via correlated sampling.
\newblock \emph{Advances in Neural Information Processing Systems},
  34:\penalty0 20172--20184, 2021.

\bibitem[Girgis et~al.(2022)Girgis, Data, and Diggavi]{girgis2022distributed}
Girgis, A.~M., Data, D., and Diggavi, S.
\newblock Distributed user-level private mean estimation.
\newblock In \emph{2022 IEEE International Symposium on Information Theory
  (ISIT)}, pp.\  2196--2201. IEEE, 2022.

\bibitem[Han et~al.(2021)Han, Liang, Wang, and Zhang]{han2021generalized}
Han, Y., Liang, Z., Wang, Y., and Zhang, J.
\newblock Generalized linear bandits with local differential privacy.
\newblock \emph{Advances in Neural Information Processing Systems},
  34:\penalty0 26511--26522, 2021.

\bibitem[Hanna et~al.(2022)Hanna, Girgis, Fragouli, and
  Diggavi]{hanna2022differentially}
Hanna, O.~A., Girgis, A.~M., Fragouli, C., and Diggavi, S.
\newblock Differentially private stochastic linear bandits:(almost) for free.
\newblock \emph{arXiv preprint arXiv:2207.03445}, 2022.

\bibitem[Hao et~al.(2020)Hao, Lattimore, and Szepesvari]{hao2020adaptive}
Hao, B., Lattimore, T., and Szepesvari, C.
\newblock Adaptive exploration in linear contextual bandit.
\newblock In \emph{International Conference on Artificial Intelligence and
  Statistics}, pp.\  3536--3545. PMLR, 2020.

\bibitem[He et~al.(2022{\natexlab{a}})He, Wang, Min, and Gu]{he2022simple}
He, J., Wang, T., Min, Y., and Gu, Q.
\newblock A simple and provably efficient algorithm for asynchronous federated
  contextual linear bandits.
\newblock \emph{arXiv preprint arXiv:2207.03106}, 2022{\natexlab{a}}.

\bibitem[He et~al.(2022{\natexlab{b}})He, Zhang, and Zhang]{he2022reduction}
He, J., Zhang, J., and Zhang, R.
\newblock A reduction from linear contextual bandit lower bounds to estimation
  lower bounds.
\newblock In \emph{International Conference on Machine Learning}, pp.\
  8660--8677. PMLR, 2022{\natexlab{b}}.

\bibitem[Hu \& Hegde(2022)Hu and Hegde]{hu2022near}
Hu, B. and Hegde, N.
\newblock Near-optimal {Thompson} sampling-based algorithms for differentially
  private stochastic bandits.
\newblock In \emph{Uncertainty in Artificial Intelligence}, pp.\  844--852.
  PMLR, 2022.

\bibitem[Huang et~al.(2021)Huang, Wu, Yang, and Shen]{huang2021federated}
Huang, R., Wu, W., Yang, J., and Shen, C.
\newblock Federated linear contextual bandits.
\newblock \emph{Advances in Neural Information Processing Systems},
  34:\penalty0 27057--27068, 2021.

\bibitem[Kamath et~al.(2019)Kamath, Li, Singhal, and
  Ullman]{kamath2019privately}
Kamath, G., Li, J., Singhal, V., and Ullman, J.
\newblock Privately learning high-dimensional distributions.
\newblock In \emph{Conference on Learning Theory}, pp.\  1853--1902. PMLR,
  2019.

\bibitem[Lai \& Robbins(1985)Lai and Robbins]{lai1985asymptotically}
Lai, T.~L. and Robbins, H.
\newblock Asymptotically efficient adaptive allocation rules.
\newblock \emph{Advances in applied mathematics}, 6\penalty0 (1):\penalty0
  4--22, 1985.

\bibitem[Lattimore \& Szepesv{\'a}ri(2020)Lattimore and
  Szepesv{\'a}ri]{LS19bandit-book}
Lattimore, T. and Szepesv{\'a}ri, C.
\newblock \emph{Bandit Algorithms}.
\newblock Cambridge University Press, 2020.

\bibitem[Levy et~al.(2021)Levy, Sun, Amin, Kale, Kulesza, Mohri, and
  Suresh]{levy2021learning}
Levy, D., Sun, Z., Amin, K., Kale, S., Kulesza, A., Mohri, M., and Suresh,
  A.~T.
\newblock Learning with user-level privacy.
\newblock \emph{Advances in Neural Information Processing Systems},
  34:\penalty0 12466--12479, 2021.

\bibitem[Li \& Wang(2022{\natexlab{a}})Li and Wang]{li2022asynchronous}
Li, C. and Wang, H.
\newblock Asynchronous upper confidence bound algorithms for federated linear
  bandits.
\newblock In \emph{International Conference on Artificial Intelligence and
  Statistics}, pp.\  6529--6553. PMLR, 2022{\natexlab{a}}.

\bibitem[Li \& Wang(2022{\natexlab{b}})Li and Wang]{li2022communication}
Li, C. and Wang, H.
\newblock Communication efficient federated learning for generalized linear
  bandits.
\newblock \emph{arXiv preprint arXiv:2202.01087}, 2022{\natexlab{b}}.

\bibitem[Li et~al.(2022{\natexlab{a}})Li, Zhou, and Ji]{li2022differentially}
Li, F., Zhou, X., and Ji, B.
\newblock Differentially private linear bandits with partial distributed
  feedback.
\newblock In \emph{2022 20th International Symposium on Modeling and
  Optimization in Mobile, Ad hoc, and Wireless Networks (WiOpt)}, pp.\  41--48.
  IEEE, 2022{\natexlab{a}}.

\bibitem[Li et~al.(2020)Li, Song, and Fragouli]{li2020federated}
Li, T., Song, L., and Fragouli, C.
\newblock Federated recommendation system via differential privacy.
\newblock In \emph{2020 IEEE International Symposium on Information Theory
  (ISIT)}, pp.\  2592--2597. IEEE, 2020.

\bibitem[Li et~al.(2022{\natexlab{b}})Li, Song, Honorio, and
  Lin]{li2022federated}
Li, W., Song, Q., Honorio, J., and Lin, G.
\newblock Federated x-armed bandit.
\newblock \emph{arXiv preprint arXiv:2205.15268}, 2022{\natexlab{b}}.

\bibitem[Liu et~al.(2020)Liu, Suresh, Yu, Kumar, and Riley]{liu2020learning}
Liu, Y., Suresh, A.~T., Yu, F. X.~X., Kumar, S., and Riley, M.
\newblock Learning discrete distributions: user vs item-level privacy.
\newblock \emph{Advances in Neural Information Processing Systems},
  33:\penalty0 20965--20976, 2020.

\bibitem[McMahan et~al.(2017{\natexlab{a}})McMahan, Moore, Ramage, Hampson, and
  y~Arcas]{mcmahan2017communicationefficient}
McMahan, B., Moore, E., Ramage, D., Hampson, S., and y~Arcas, B.~A.
\newblock Communication-efficient learning of deep networks from decentralized
  data.
\newblock In \emph{Proceedings of the 20th International Conference on
  Artificial Intelligence and Statistics (AISTATS)}, pp.\  1273--1282, Fort
  Lauderdale, FL, USA, Apr. 2017{\natexlab{a}}.

\bibitem[McMahan et~al.(2017{\natexlab{b}})McMahan, Ramage, Talwar, and
  Zhang]{mcmahan2017learning}
McMahan, H.~B., Ramage, D., Talwar, K., and Zhang, L.
\newblock Learning differentially private recurrent language models.
\newblock \emph{arXiv preprint arXiv:1710.06963}, 2017{\natexlab{b}}.

\bibitem[Mishra \& Thakurta(2015)Mishra and Thakurta]{mishra2015nearly}
Mishra, N. and Thakurta, A.
\newblock ({N}early) optimal differentially private stochastic multi-arm
  bandits.
\newblock In \emph{Proceedings of the Thirty-First Conference on Uncertainty in
  Artificial Intelligence}, pp.\  592--601, 2015.

\bibitem[Papini et~al.(2021)Papini, Tirinzoni, Restelli, Lazaric, and
  Pirotta]{papini2021leveraging}
Papini, M., Tirinzoni, A., Restelli, M., Lazaric, A., and Pirotta, M.
\newblock Leveraging good representations in linear contextual bandits.
\newblock In \emph{International Conference on Machine Learning}, pp.\
  8371--8380. PMLR, 2021.

\bibitem[Reeve et~al.(2018)Reeve, Mellor, and Brown]{reeve2018k}
Reeve, H., Mellor, J., and Brown, G.
\newblock The k-nearest neighbour {UCB} algorithm for multi-armed bandits with
  covariates.
\newblock In \emph{Algorithmic Learning Theory}, pp.\  725--752. PMLR, 2018.

\bibitem[Ren et~al.(2020)Ren, Zhou, Liu, and Shroff]{ren2020multi}
Ren, W., Zhou, X., Liu, J., and Shroff, N.~B.
\newblock Multi-armed bandits with local differential privacy.
\newblock \emph{arXiv preprint arXiv:2007.03121}, 2020.

\bibitem[Rigollet \& Zeevi(2010)Rigollet and Zeevi]{rigollet2010nonparametric}
Rigollet, P. and Zeevi, A.
\newblock Nonparametric bandits with covariates.
\newblock \emph{arXiv preprint arXiv:1003.1630}, 2010.

\bibitem[Sajed \& Sheffet(2019)Sajed and Sheffet]{sajed2019optimal}
Sajed, T. and Sheffet, O.
\newblock An optimal private stochastic-mab algorithm based on optimal private
  stopping rule.
\newblock In \emph{International Conference on Machine Learning}, pp.\
  5579--5588. PMLR, 2019.

\bibitem[Shariff \& Sheffet(2018)Shariff and
  Sheffet]{shariff2018differentially}
Shariff, R. and Sheffet, O.
\newblock Differentially private contextual linear bandits.
\newblock \emph{Advances in Neural Information Processing Systems}, 31, 2018.

\bibitem[Shi \& Shen(2021)Shi and Shen]{shi2021federated}
Shi, C. and Shen, C.
\newblock Federated multi-armed bandits.
\newblock In \emph{Proceedings of the AAAI Conference on Artificial
  Intelligence}, volume~35, pp.\  9603--9611, 2021.

\bibitem[Shi et~al.(2021)Shi, Shen, and Yang]{shi2021aistats}
Shi, C., Shen, C., and Yang, J.
\newblock Federated multi-armed bandits with personalization.
\newblock In \emph{Proceedings of the 24rd International Conference on
  Artificial Intelligence and Statistics (AISTATS)}, April 2021.

\bibitem[Tao et~al.(2022)Tao, Wu, Zhao, and Wang]{tao2022optimal}
Tao, Y., Wu, Y., Zhao, P., and Wang, D.
\newblock Optimal rates of (locally) differentially private heavy-tailed
  multi-armed bandits.
\newblock In \emph{International Conference on Artificial Intelligence and
  Statistics}, pp.\  1546--1574. PMLR, 2022.

\bibitem[Tenenbaum et~al.(2021)Tenenbaum, Kaplan, Mansour, and
  Stemmer]{tenenbaum2021differentially}
Tenenbaum, J., Kaplan, H., Mansour, Y., and Stemmer, U.
\newblock Differentially private multi-armed bandits in the shuffle model.
\newblock \emph{Advances in Neural Information Processing Systems},
  34:\penalty0 24956--24967, 2021.

\bibitem[Tossou \& Dimitrakakis(2016)Tossou and
  Dimitrakakis]{tossou2016algorithms}
Tossou, A.~C. and Dimitrakakis, C.
\newblock Algorithms for differentially private multi-armed bandits.
\newblock In \emph{Thirtieth AAAI Conference on Artificial Intelligence}, 2016.

\bibitem[Tropp(2011)]{tropp2011freedman}
Tropp, J.
\newblock Freedman's inequality for matrix martingales.
\newblock \emph{Electronic Communications in Probability}, 16:\penalty0
  262--270, 2011.

\bibitem[Wang et~al.(2022{\natexlab{a}})Wang, Li, Cheng, and
  Lin]{wang2022federated}
Wang, C.-H., Li, W., Cheng, G., and Lin, G.
\newblock Federated online sparse decision making.
\newblock \emph{arXiv preprint arXiv:2202.13448}, 2022{\natexlab{a}}.

\bibitem[Wang et~al.(2020{\natexlab{a}})Wang, Zhao, Wu, Chopra, Khaitan, and
  Wang]{wang2020global}
Wang, H., Zhao, Q., Wu, Q., Chopra, S., Khaitan, A., and Wang, H.
\newblock Global and local differential privacy for collaborative bandits.
\newblock In \emph{Fourteenth ACM Conference on Recommender Systems}, pp.\
  150--159, 2020{\natexlab{a}}.

\bibitem[Wang et~al.(2022{\natexlab{b}})Wang, Zhao, and Wang]{wang2022dynamic}
Wang, H., Zhao, D., and Wang, H.
\newblock Dynamic global sensitivity for differentially private contextual
  bandits.
\newblock In \emph{Proceedings of the 16th ACM Conference on Recommender
  Systems}, pp.\  179--187, 2022{\natexlab{b}}.

\bibitem[Wang et~al.(2020{\natexlab{b}})Wang, Hu, Chen, and
  Wang]{wang2019distributed}
Wang, Y., Hu, J., Chen, X., and Wang, L.
\newblock Distributed bandit learning: Near-optimal regret with efficient
  communication.
\newblock In \emph{International Conference on Learning Representations},
  2020{\natexlab{b}}.

\bibitem[Wang et~al.(2019)Wang, Song, Zhang, Song, Wang, and
  Qi]{wang2019beyond}
Wang, Z., Song, M., Zhang, Z., Song, Y., Wang, Q., and Qi, H.
\newblock Beyond inferring class representatives: User-level privacy leakage
  from federated learning.
\newblock In \emph{IEEE INFOCOM 2019-IEEE conference on computer
  communications}, pp.\  2512--2520. IEEE, 2019.

\bibitem[Zheng et~al.(2020)Zheng, Cai, Huang, Li, and Wang]{zheng2020locally}
Zheng, K., Cai, T., Huang, W., Li, Z., and Wang, L.
\newblock Locally differentially private (contextual) bandits learning.
\newblock \emph{Advances in Neural Information Processing Systems},
  33:\penalty0 12300--12310, 2020.

\bibitem[Zhou \& Chowdhury(2023)Zhou and Chowdhury]{zhou2023differentially}
Zhou, X. and Chowdhury, S.~R.
\newblock On differentially private federated linear contextual bandits.
\newblock \emph{arXiv preprint arXiv:2302.13945}, 2023.

\bibitem[Zhu et~al.(2021)Zhu, Zhu, Liu, and Liu]{Zhu_2021}
Zhu, Z., Zhu, J., Liu, J., and Liu, Y.
\newblock Federated bandit: A gossiping approach.
\newblock \emph{Proceedings of the ACM on Measurement and Analysis of Computing
  Systems}, 5\penalty0 (1):\penalty0 1–29, Feb 2021.
\newblock ISSN 2476-1249.
\newblock \doi{10.1145/3447380}.
\newblock URL \url{http://dx.doi.org/10.1145/3447380}.

\end{thebibliography}

\newpage
\appendix
\onecolumn


\allowdisplaybreaks

\section{Relationship with other DP Notions in Bandits}\label{appx:framework}

\if{0}
\begin{table*}[t]
\caption{Regret bounds under user-level $(\varepsilon,\delta)$-CDP and $(\varepsilon,\delta)$-LDP.}
	\vspace{-0.2in}
\vskip 0.15in
\begin{center}
\begin{small}
\begin{sc}
\begin{tabular}{lcccr}
\toprule
    Algorithm & Asm.\ref{asm:margin} &Model & Constraint & Regret\\
    \midrule
    \citet{shariff2018differentially}& \checkmark & Sing & Item-Level JDP &$\tilde{O}\left(d\sqrt{d}\max\{\sqrt{d}, 1/\varepsilon\}\log^{1.5} T + d^2\log^2 T\right)$ \\
    \midrule
    \citet{shariff2018differentially}& \ding{55} & Sing & Item-Level JDP &$\tilde{O}\left(d^{3/4}\sqrt{T}(d^{1/4} + 1/\sqrt{\varepsilon})\right)$ \\
    \midrule
    lower bound & \ding{55} & Sing & Item-Level JDP &$\Omega\left(\sqrt{dT} + d/\varepsilon\right)$ \\
    \midrule 
    \cite{zheng2020locally}& \ding{55}& Sing &Item-level LDP & $\tilde{O}\left(T^{3/4}/\varepsilon\right)$\\
    \midrule
    \citet{dubey2020differentially} & \ding{55} &Fed & Item-level JDP & $\tilde{O}\left(d\sqrt{MT/\varepsilon}\right)$ \\
    \midrule
    \robin  & \checkmark & Fed & User-Level CDP &$\tilde{O}\left(\min\left\{M, \max\left\{ 1, \frac{d\log T}{M\varepsilon^2}\right\} \right\} d\log T\right)$ \\
    \midrule
     Lower Bound & \checkmark & Fed & User-level CDP & $\Omega(\min \left\{M, \max\left\{1, \frac{1}{M\varepsilon^2} \right\}\right\}\log T)$  \\
     \midrule
     Lower Bound &\ding{55} & Fed & User-level CDP & $\Omega\left(\min\left\{M, \max\left\{\sqrt{M}, \frac{1}{\varepsilon} \right\} \right\} \sqrt{T}\right)$ \\
     \midrule
     Lower Bound &\checkmark&Fed & User-level LDP & $\Omega\left(\min\left\{M, \frac{1}{\varepsilon}\right\}\log T  \right)$ \\
     \midrule
     Lower Bound &\ding{55} &Fed & User-level LDP & $\Omega\left(\min\left\{M, \sqrt{M/\varepsilon}\right\}\sqrt{T}  \right) $ \\
\bottomrule
\end{tabular}
\end{sc}
\end{small}
\end{center}
    \begin{center}
        \scriptsize
	note: dimension $d$ is omitted.
	\end{center}
	\vspace{-0.1in}
\vskip -0.1in
\end{table*}
\fi

We use \Cref{fig:fed alg} to demonstrate all existing DP channels in the bandits literature.

\usetikzlibrary {positioning}

\begin{figure}[h]
   \centering
\begin{tikzpicture}[node distance = 15mm, main/.style = {draw, blue,  rectangle, inner sep = 2pt}, env/.style = {draw, rectangle, node font=\tiny, inner sep = 2pt, orange}, server/.style = {draw,cyan, rectangle, inner sep = 2pt}]


\draw[->,thick,cyan] (-1,0.4) --node[pos = 1, right, black] {$t$} (10,0.4);
\draw[->,thick,cyan] (-1,-0.4) -- (10,-0.4);

\draw[thick,blue] (-1,0.45) -- (10,0.45);
\draw[thick,blue] (-1,-0.45) -- (10,-0.45);

\draw[thick,blue] (-1,-1.7) -- (10,-1.7);
\draw[thick,blue] (-1,1.7) -- (10,1.7);
\draw[thick,orange] (-1,-1.75) -- (10,-1.75);
\draw[thick,orange] (-1,1.75) -- (10,1.75);

\node[main] (H1) at (0,1) {$H_{1,t}$};
\node[main] (a1) at (3,1.5) {$a_{1,t}$};
\node[env] (c1) at (1,1.9) {$c_{1,t}$};
\node[env] (r1) at (4,2.3) {$r_{1,t}$};
\node[main] (H1n) at (6,1) {$H_{1,t+1}$};
\node[main] (a1n) at (9,1.5) {$a_{1,t+1}$};
\node[env] (c1n) at (7,1.9) {$c_{1,t+1}$};
\node[env] (r1n) at (10,2.3) {$r_{1,t+1}$};

\node[server] (qi) at (0,0) {$\{q_{i, \leq t}\}_{i}$};
\node[server] (q) at (2,0) {$q_{\leq t}$};
\node[server] (qin) at (6,0) {$\{q_{i,\leq t+1}\}_{i}$};
\node[server] (qn) at (8,0) {$q_{\leq t+1}$};

\node[main] (H2) at (0,-1) {$H_{2,t}$};
\node[main] (a2) at (3,-1.5) {$a_{2,t}$};
\node[env] (c2) at (1,-1.9) {$c_{2,t}$};
\node[env] (r2) at (4,-2.3) {$r_{2,t}$};
\node[main] (H2n)at (6,-1) {$H_{2,t+1}$};
\node[main] (a2n) at (9,-1.5) {$a_{2,t+1}$};
\node[env] (c2n) at (7,-1.9) {$c_{2,t+1}$};
\node[env] (r2n) at (10,-2.3) {$r_{2,t+1}$};

\node[{draw, blue, circle, node font = \tiny, inner sep = 1pt}] (alg1) at (2,0.7) {$\mathtt{A}_1$};
\node[{draw,  blue, circle, node font = \tiny, inner sep = 1pt}] (alg1n) at (8,0.7) {$\mathtt{A}_1$};
\node[{draw, blue, circle, node font = \tiny, inner sep = 1pt}] (alg2) at (2,-0.7) {$\mathtt{A}_2$};
\node[{draw, blue, circle, node font = \tiny, inner sep = 1pt}] (alg2n) at (8,-0.7) {$\mathtt{A}_2$};

\draw[->,red] (H1) -- node[pos=0.3,left] {$\Rt_1$} (qi);
\draw[->,red] (qi) -- node[pos=0.5,above] {$\Rt_0$} (q);
\draw[->, dashed] (H1) -- (H1n);
\draw[->,dashed] (a1) -- (3,1);
\draw[->,dashed] (c1) -- (1,1);
\draw[->, red] (r1) -- (H1n);
\draw[->, red] (H1) -- (alg1);
\draw[->] (q) -- (alg1);
\draw[->] (c1) -- (alg1);
\draw[->] (alg1) -- (a1);

\draw[->,red] (H2) -- node[pos=0.3,left] {$\Rt_2$} (qi);
\draw[->, dashed] (H2) -- (H2n);
\draw[->,dashed] (a2) -- (3,-1);
\draw[->,dashed] (c2) -- (1,-1);
\draw[->] (r2) -- (H2n);
\draw[->] (H2) -- (alg2);
\draw[->] (q) -- (alg2);
\draw[->] (c2) -- (alg2);
\draw[->] (alg2) -- (a2);

\draw[->,red] (H1n) -- node[pos=0.3,left] {$\Rt_1$} (qin);
\draw[->,red] (qin) -- node[pos=0.5,above] {$\Rt_0$} (qn);
\draw[->,red] (H1n) -- (alg1n);
\draw[->] (qn) -- (alg1n);
\draw[->] (c1n) -- (alg1n);
\draw[->] (alg1n) -- (a1n);

\draw[->,red] (H2n) -- node[pos=0.3,left] {$\Rt_2$} (qin);
\draw[->] (H2n) -- (alg2n);
\draw[->] (qn) -- (alg2n);
\draw[->] (c2n) -- (alg2n);
\draw[->] (alg2n) -- (a2n);

\draw[->] (a1) -- (r1);
\draw[->] (a1n) -- (r1n);
\draw[->] (a2) -- (r2);
\draw[->] (a2n) -- (r2n);
\end{tikzpicture}
   \caption{Graphical structure of all possible private channels in a bandits model. We assume that $\mathtt{Alg}_i$ produces some intermediate random variable $\mathtt{A}_i$ which determines the action. }\label{fig:fed alg}
   
\end{figure}

Recall that $H_{i,t} = \{c_{i,\tau},a_{i,\tau},r_{i,\tau}\}_{\tau=1}^{t-1}.$ We separately discuss different DP mechanisms in single-client and multi-client settings.

{\bf Single-client.} In this case, the federated bandits framework defined in \Cref{sec: DP framework} reduces to the upper half of \Cref{fig:fed alg}. 
\begin{itemize}
    \item A single-client bandits algorithm is called {\bf locally differentially private} \citep{ren2020multi,zheng2020locally},  if the red arrow from $r_{1,t}$ to the dataset $H_{1,t+1}$ is a DP channel. Mathematically, for any two rewards $r_{1,t}$ and $r_{1,t}'$, and any measurable set $S$ containing $H_{1,t+1}$, we have
    \[\Pb(H_{1,t+1}\in S | r_{1,t}) \leq e^{\varepsilon} \Pb(H_{1,t+1}\in S| r_{1,t}')  + \delta.\]
    \item A single-client stochastic multi-armed bandits algorithm is called {\bf globally differentially private} \citep{tossou2016algorithms,azize2022privacy}, if the red arrow from $H_{1,t}$ to $\mathtt{A}_1$ is a DP channel with respect to rewards. More precisely, for any $t'$-neighboring dataset $H_{1,t}$ and $H_{1,t}'$ satisfying $r_{1,t_0} = r_{1,t_0}'$ for all $t_0\leq t$ except at $t_0 = t'$, and any measurable set $S$ containing $\mathtt{A}_1$, we have
    \begin{align*}
        \Pb(\mathtt{A}_1\in S | H_{1,t})\leq e^{\varepsilon}\Pb(\mathtt{A}_1\in S | H_{1,t}') + \delta.
    \end{align*}
    \item A single-client linear contextual bandits algorithm is called {\bf jointly differentially private} \citep{shariff2018differentially}, if the red arrow from $H_{1,t}$ to $\mathtt{A}_1$ is a DP channel with respect to one datum. More precisely, for any $t'$-neighboring dataset $H_{1,t}$ and $H_{1,t}'$ satisfying $(c_{i,t_0},a_{i,t_0},r_{1,t_0}) = (c_{i,t_0}',a_{i,t_0}',r_{1,t_0}')$ for all $t_0\leq t$ except at $t_0 = t'$, and any measurable set $S$ containing $\mathtt{A}_1$, we have
    \begin{align*}
        \Pb(\mathtt{A}_1\in S | H_{1,t},c_{i,t})\leq e^{\varepsilon}\Pb(\mathtt{A}_1\in S | H_{1,t}',c_{i,t}) + \delta.
    \end{align*}
\end{itemize}

{\bf Multi-clients.} In this case, we focus on the red arrows between clients and the server, i.e. $\Rt_i$ and $\Rt_0$.

\begin{itemize}
    \item A multi-client bandits algorithm is called {\bf item-level locally (jointly) differentially private} \citep{dubey2020differentially}, if the red arrow $\Rt_i$ from $H_{i,t}$ to $q_{i,\leq t}$ is a DP channel, and each client can fully rely on its own local dataset (with-memory setting). More precisely, if for any client $i$ and $t'$, $H_{i,t}$ and $H_{i,t}'$ only differ in one datum at time $t'$, i.e. $(c_{i,t_0},a_{i,t_0},r_{i,t_0}) = (c_{i,t_0}',a_{i,t_0}',r_{i,t_0}')$ for all $t_0\leq t$ except at $t_0 = t'$, and for any measurable set containing $q_{i,t}$, we have
    \[\Pb(q_{i,t}\in S | H_{i,t}) \le e^{\varepsilon} \Pb(q_{i,t}\in S|H_{i,t}') + \delta.\]
\end{itemize}

We close this section by noting that there are many other types of DP that can be deduced from \Cref{fig:fed alg}. We list several key ingredients that form a DP type, such as whether it is item-level, whether it is locally differentially private, whether it is jointly differentially private, and whether it is memoryless.

\section{DP Algorithms}\label{appx:alg}

In this section, we list all the differentially private mechanisms that are used in \Cref{alg:greedy} and provide their guarantees. First, we introduce the formal definition of $(r,\beta)$-concentration.

\begin{definition}[$(r,\beta)$-concentration]
$\{x_i\}_i\subset\Rb^d$ is $(r,\beta)$ concentrated if there exists $\mu\in\Rb^d$ such that $\Pb(\forall i,~~ \|x_i-\mu\|\geq r)\leq \beta$.
\end{definition}

The PrivateRange algorithm outputs a private interval with length $4r$ if the input dataset is $(r,\beta)$-concentration such that the dataset falls into this interval with high probability.

\begin{algorithm}[H]
\small
\caption{PrivateRange($\{x_i\}_{i\in[M]}, \varepsilon, r, B$) \citep{feldman2017generalization}}
\begin{algorithmic}[1]

\STATE \textbf{Input:}  $\{x_i\}_{i\in[M]}\in[-B,B]^M$, $r$: concentration radius, $\varepsilon$: privacy parameter.

\STATE Divide interval $[-B,B]$ into $\ell = B/r$ disjoint bins, each with length $2r$. Let $S$ be the set of middle points of those bins.

\STATE For $i\in[M]$, let $x_i' = \arg\min_{x\in S}|x- x_i|$ be the point in $S$ closest to $x_i$.

\STATE For any $x\in S$, define cost function
\[c(x) = \max\left\{\big|i\in[M]: x_i' < x\big|, \big|i\in[M]: x_i' > x\big|\right\}.\]

\STATE Sample $\bar{x}\in S$ from the distribution:
\[\Pb(\bar{x} = x) = \frac{\exp(-\varepsilon c(x)/2)}{\sum_{x'\in S}\exp(-\varepsilon c(x')/2)}.\]

\STATE Return $[\bar{x} - 2r, \bar{x} + 2r].$

\end{algorithmic}
\end{algorithm}
\begin{lemma}
    PrivateRange($\{x_i\}_{i\in[M]}, \varepsilon, r, B$) is $(\varepsilon,0)$ differentially private.
\end{lemma}

\begin{algorithm}[H]
\small
\caption{WinsorizedMean1D($\{x_i\}_{i\in[M]}, r,\beta,\varepsilon, B$) \citep{levy2021learning}}
\begin{algorithmic}[1]

\STATE \textbf{Input:}  $\{x_i\}_{i\in[M]}\in[-B,B]^M$, $r$: concentration radius, $\gamma$: concentration error probability, $\varepsilon,\delta$: privacy parameter.

\STATE $[a,b]=\text{PrivateRange}(\{x_i\}_i, \varepsilon/2,r, B)$, where $b-a = 4r$.

\STATE Sample $\xi\sim  Lap(0,8r/(M\varepsilon))$ and return

\[\bar{x} = \xi + \frac{1}{M}\sum_{i=1}^M\max\left\{a, \min\{b,x_i\}\right\}.\]

\end{algorithmic}
\end{algorithm}

\begin{lemma}[Theorem 1 in \citet{levy2021learning}]
    WinsorizedMean1D is $(\varepsilon, 0)$-differentially private.
\end{lemma}

We slightly modify a constant in the following WinsorizedMeanHighD (\texttt{WMHD}) algorithm. Compared to the original version, this modification is due to a tighter parameter choice in the advanced composition rule~\ref{coro:advanced composition}.

\begin{algorithm}[H]
\small
\caption{WinsorizedMeanHighD($\{x_i\}_{i\in[M]}, r,\beta,\varepsilon,\delta$)~\citep{levy2021learning}}\label{alg:winsorizedHigh}
\begin{algorithmic}[1]

\STATE \textbf{Input:}  $\{x_i\}_{i\in[M]}\subset\mathbb{S}_1\subset\Rb^d$, $r$: concentration radius, $\gamma$: concentration error probability, $\varepsilon,\delta$: privacy parameter.

\STATE Let $D = \mbox{diag}(w_1,\ldots,w_d)$, where $w_s$ are sampled independently and uniformly from $\{1,-1\}$ for all $s\in[d]$.

\STATE Set $\mathbf{U} = \frac{1}{\sqrt{d}}\mathbf{H}D$, where $\mathbf{H}$ is a $d$-dimensional Hadamard matrix.

\STATE For all $i\in[M],s\in[d]$, compute
\[y_{i,s} = e_s^{\top}\mathbf{U}X_i.\]

\STATE Let $\varepsilon' = \frac{\varepsilon}{ \sqrt{6d\log(1/\delta)}},  r' = 10r\sqrt{\frac{\log (dM/\beta)}{d}}$.

\STATE For all $s\in[d]$, compute
\[Y_s = \text{WinsorizedMean1D}(\{y_{i,s}\}_{i\in[M]},r',\beta,\varepsilon',  B\sqrt{d}).\]

\STATE Let $\bar{Y} = (Y_1,\ldots, Y_d)^{\TT}$ and return
\[\bar{X} = \mathbf{U}^{-1}\bar{Y}.\]

\end{algorithmic}
\end{algorithm}

\begin{theorem}[Theorem 2 in \citet{levy2021learning}]\label{lemma:levy winsorized guarantee}
    If $\hat{\theta}\in\Rb^d$ is the output of WinsorizedMeanHighD$(\{\theta_i\}_i, r, \beta, \varepsilon, \delta , B)$, then $\hat{\theta}$ is $(\varepsilon,\delta)$-differentially private. Moreover, let each coordinate of $\boldsymbol{\xi}\in\Rb^d $ be independently sampled from $\text{Lap}(8r'/(M\varepsilon'))$ and $\bar{\theta} = \sum_i\theta_i/M$ be the sample mean, where $\varepsilon' = \frac{\varepsilon}{\sqrt{6d\log(1/\delta)}}$ and $r' = 10r\sqrt{(\log(dM/\beta))/d}$. Then, we have
    \begin{align*}
        d_{TV}\left(\Pb\left(\|\hat{\theta} - \bar{\theta}\|\big|\{\theta_i\}_i \right),\Pb\left(\| \boldsymbol\xi\|\big|\{\theta_i\}_i \right) \right) \leq \beta + \frac{d^2B}{10r\sqrt{(\log(dM/\beta))}}\exp\left(-\frac{M\varepsilon}{8\sqrt{6d\log(1/\delta)}}\right).
    \end{align*}
\end{theorem}

\begin{corollary}\label{coro: winsorized guarantee}
Under the same setting as in \Cref{lemma:levy winsorized guarantee}, we have
    \begin{align*}
        \Pb\left(\|\hat{\theta} - \bar{\theta}\|\geq \frac{80r\log(d/\beta)\sqrt{6d\log(dM/\beta)\log(1/\delta)}}{M\varepsilon} \right)\leq 3\beta + \frac{d^2B}{10r\sqrt{(\log(dM/\beta))}}\exp\left(-\frac{M\varepsilon}{8\sqrt{6d\log(1/\delta)}}\right).
    \end{align*}
\end{corollary}

\begin{proof}
By the definition of $(r,\beta)$-concentration, we have
\begin{align*}
     \Pb\big(\{\theta_i\}_i \text{ is not } (r,0)\text{-concentrated}\big) \leq \beta.
\end{align*}
Therefore, 
    \begin{align*}
        \Pb&\left(\|\hat{\theta} - \bar{\theta}\|\geq \frac{80r\log(d/\beta)\sqrt{6d\log(dM/\beta)\log(1/\delta)}}{M\varepsilon} \right)\\
        &\leq \Pb\left(\|\hat{\theta} - \bar{\theta}\|\geq \frac{80r\log(d/\beta)\sqrt{6d\log(dM/\beta)\log(1/\delta)}}{M\varepsilon} \bigg|\{\theta_i\}_i \text{ is } (r,0)\text{-concentrated}\right) \Pb\left(\{\theta_i\}_i \text{ is } (r,0)\text{-concentrated}\right)\\
        &\quad +  \Pb\left(\{\theta_i\}_i \text{ is not } (r,0)\text{-concentrated}\right)\\
        &\overset{(a)}\leq 2\beta + \frac{d^2B}{10r\sqrt{(\log(dM/\beta))}}\exp\left(-\frac{M\varepsilon}{8\sqrt{6d\log(1/\delta)}}\right) + \Pb\left(\|\boldsymbol{\xi}\|\geq \frac{80r\log(d/\beta)\sqrt{6d\log(dM/\beta)\log(1/\delta)}}{M\varepsilon}\right)\\
        & = 2\beta + \frac{d^2B}{10r\sqrt{(\log(dM/\beta))}}\exp\left(-\frac{M\varepsilon}{8\sqrt{6d\log(1/\delta)}}\right) + \Pb\left(\|\boldsymbol{\xi}\|\geq \frac{8r'\sqrt{d}}{M\varepsilon'}\log(d/\beta)\right)\\
        &\overset{(b)}\leq 3\beta + \frac{d^2B}{10r\sqrt{(\log(dM/\beta))}}\exp\left(-\frac{M\varepsilon}{8\sqrt{6d\log(1/\delta)}}\right),
    \end{align*}
where $(a)$ is due to \Cref{lemma:levy winsorized guarantee}, and $(b)$ follows from the tail bound for Laplace random vectors developed in \Cref{lemma:Laplace tail bound}.

\end{proof}

\section{Proof of the Regret Upper Bound of \robin}\label{sec:greedy}

In this section, we provide the analysis for the upper bound of the regret of \Cref{alg:greedy}. The key idea is to show that the global estimator is sufficiently accurate. 

{\it Outline of this section:} 
In {\bf Step 1}, we first develop \Cref{lemma:first U phase full rank} aided by \Cref{lemma:bound number of  non optimal arm with large margin}, and show the minimum eigenvalue of the Gram matrix $V_{i,U}$ at the end of the initialization phase $U$ scales linearly in $|\Tc_U|$. Then, in {\bf Step 2}, \Cref{prop:good event} shows that linearly growing minimum eigenvalues implies accurate global estimators and vice versa, which inductively verifies that the estimation error of $\hat{\theta}^{p+1}$ decays in the order of $O(1/\sqrt{M|\Tc_p|})$. Finally, with the accurate global estimator, \Cref{thm:greedy upper bound} in  {\bf Step 3} presents the regret upper bound.

Recall that the Gram matrix of client $i$ at the end of phase $p$ is 
\begin{align*}
    \tilde{V}_{i,p} = \sum_{\tau\in\Tc_p}x_{i,\tau,a_{i,\tau}}x_{i,\tau,a_{i,\tau}}^{\TT}.
\end{align*}

{\bf Step 1: Bound the Minimum Eigenvalue in the Initialization Phase.}

This step is akin to Lemma 12 and Lemma 13 in \citet{papini2021leveraging} and consists of two parts. The first part states that when the sub-optimal gap is larger that $\Delta$, the number of times that LinUCB does not choose the optimal arm is upper bounded by $\tilde{O}(1/\Delta)$.

\begin{lemma}\label{lemma:bound number of  non optimal arm with large margin}

Given any $\beta\in(0,1)$, with probability at least $1-\beta$, and for any $\Delta>0$, the following inequality holds for any client $i$.
    \begin{align}
        \Eb\left[\sum_{\tau\in\Tc_U}\mathbbm{1}\left\{a_{i,\tau}\neq a_{c_{i,\tau}^*}, \forall b\neq a_{c_{i,\tau}}^*, (\phi(c_{i,\tau},a_{c_{i,\tau}}^*) - \phi(c_{i,\tau},b))^{\TT}\theta^*\geq \Delta \right\}\right] \leq \frac{2\alpha\sqrt{d|\Tc_U|\log|\Tc_U|}}{\Delta},\label{eqn:bound number of  non optimal arm with large margin}
    \end{align}
    where $\alpha = 1 + \sqrt{2\log(M/\beta) + d\log|\Tc_U|}$.
\end{lemma}

\begin{proof}

By Theorem 20.5 in \citet{LS19bandit-book}, if $\hat{\theta}_{i,t}$ is the local estimator in the initialization phase ($t\in\Tc_U$), we have for each client $i$, and any $\beta>0$,
\begin{align*}
    \Pb\left(\forall t\in \Tc_U, \|\hat{\theta}_{i,t} - \theta^*\|_{I+V_{i,t}} \leq 1 + \sqrt{2\log(1/\beta) + d\log|\Tc_U|}\right) \geq 1-\beta.
\end{align*}

By rescaling $\beta$ to $\beta/M$ and taking the union bound, we have
\begin{align*}
    \Pb\left(\forall i\in[M],t\in \Tc_U, \|\hat{\theta}_{i,t} - \theta^*\|_{I+V_{i,t}} \leq 1 + \sqrt{2\log(M/\beta) + d\log|\Tc_U|}\right) \geq 1-\beta,
\end{align*}
where $\beta$ is any positive number.

Let $\alpha = 1 + \sqrt{2\log(M/\beta) + d\log|\Tc_U|}$. Hence, the instantaneous regret can be upper bounded by
\begin{align*}
    (\phi&(c_{i,\tau},a_{c_{i,\tau}}^*) - \phi(c_{i,\tau},a_{i,\tau}))^{\TT}\theta^*\\
    & = \phi(c_{i,\tau},a_{c_{i,\tau}}^*)^{\TT}(\theta^*-\hat{\theta}_{i,\tau}) - \phi(c_{i,\tau},a_{i,\tau})^{\TT}(\theta^* - \hat{\theta}_{i,\tau}) + \phi(c_{i,\tau},a_{c_{i,\tau}}^*)^{\TT}\hat{\theta}_{i,\tau} - \phi(c_{i,\tau},a_{i,\tau})^{\TT} \hat{\theta}_{i,\tau}\\
    &\leq \|\phi(c_{i,\tau},a_{c_{i,\tau}}^*)\|_{(I+V_{i,\tau})^{-1}}\|\theta^* - \hat{\theta}_{i,\tau}\|_{I+V_{i,\tau}} + \|\phi(c_{i,\tau},a_{i,\tau})\|_{(I+V_{i,\tau})^{-1}}\|\theta^* - \hat{\theta}_{i,\tau}\|_{I+V_{i,\tau}} \\
    &\quad + \phi(c_{i,\tau},a_{c_{i,\tau}}^*)^{\TT}\hat{\theta}_{i,\tau} - \phi(c_{i,\tau},a_{i,\tau})^{\TT} \hat{\theta}_{i,\tau}\\
    &\leq \alpha\|\phi(c_{i,\tau},a_{c_{i,\tau}}^*)\|_{(I+V_{i,\tau})^{-1}} + \alpha\|\phi(c_{i,\tau},a_{i,\tau})\|_{(I+V_{i,\tau})^{-1}} + \phi(c_{i,\tau},a_{c_{i,\tau}}^*)^{\TT}\hat{\theta}_{i,\tau} - \phi(c_{i,\tau},a_{i,\tau})^{\TT} \hat{\theta}_{i,\tau}\\
    &\leq 2\alpha\|\phi(c_{i,\tau},a_{i,\tau})\|_{(I+V_{i,\tau})^{-1}}.
\end{align*}

Due the boundness assumption, we can conclude that
\begin{align*}
     (\phi&(c_{i,\tau},a_{c_{i,\tau}}^*) - \phi(c_{i,\tau},a_{i,\tau}))^{\TT}\theta^*\leq 2\min\left\{\alpha\|\phi(c_{i,\tau},a_{i,\tau})\|_{(I+V_{i,\tau})^{-1}} , 1\right\}.
\end{align*}

Therefore, we can bound \Cref{eqn:bound number of  non optimal arm with large margin} as follows:
    \begin{align*}
        \Eb &\left[\sum_{\tau\in\Tc_U}\mathbbm{1}\left\{a_{i,\tau}\neq a_{c_{i,\tau}^*}, \forall b\neq a_{c_{i,\tau}}^*, (\phi(c_{i,\tau},a_{c_{i,\tau}}^*) - \phi(c_{i,\tau},b))^{\TT}\theta^*\geq \Delta \right\}\right] \\
        &\leq  \Eb \left[\sum_{\tau\in\Tc_U}\mathbbm{1}\left\{a_{i,\tau}\neq a_{c_{i,\tau}^*}, \forall b\neq a_{c_{i,\tau}}^*, (\phi(c_{i,\tau},a_{c_{i,\tau}}^*) - \phi(c_{i,\tau},b))^{\TT}\theta^*\geq \Delta \right\} \frac{(\phi(c_{i,\tau},a_{c_{i,\tau}}^*) - \phi(c_{i,\tau},a_{i,\tau}))^{\TT}\theta^*}{\Delta}\right]\\
        &\leq \Eb \left[\sum_{\tau\in\Tc_U}\mathbbm{1} \left\{a_{i,\tau}\neq a_{c_{i,\tau}^*}, \forall b\neq a_{c_{i,\tau}}^*, (\phi(c_{i,\tau},a_{c_{i,\tau}}^*) - \phi(c_{i,\tau},b))^{\TT}\theta^*\geq \Delta \right\} \frac{2\min\{\alpha\|x_{i,\tau,a_{i,\tau}}\|_{(I + V_{i,\tau})^{-1}},1\}}{\Delta}\right]\\
        &\leq \Eb \left[\sum_{\tau\in\Tc_U}\frac{2\min\{\alpha\|x_{i,\tau,a_{i,\tau}}\|_{(I + V_{i,\tau})^{-1}},1\}}{\Delta}\right]\\
        &\overset{(a)}\leq \frac{2\alpha\sqrt{d|\Tc_U|\log|\Tc_U|}}{\Delta},
    \end{align*}
where $(a)$ follows from the elliptical potential lemma in~\Cref{lemma:elliptical}.
\end{proof}

Based on \Cref{lemma:bound number of  non optimal arm with large margin}, we prove the second part that the minimum eigenvalue of the Gram matrix $\tilde{V}_{i,U}$ scales linearly in $|\Tc_U|$ with high probability.

\begin{lemma}[Guarantee in the first $U$ phases]\label{lemma:first U phase full rank}
If $U$ satisfies
\begin{align*}
    \frac{12 (dC_0+\lambda_0)\sqrt{|\Tc_U|\log(2dM/\beta)}\log|\Tc_U|}{\lambda_0}\leq \frac{\lambda_0}{4}|\Tc_U|,
\end{align*}
    then, for any $\beta\in(0,1)$, with probability at least $1- \beta$, the Gram matrix of any client $i$ at the end of phase $U$ has full rank, and 
    \[\forall i\in[M], ~~\lambda_{\min}\left(\tilde{V}_{i,U}\right) \geq \frac{\lambda_0}{4}|\Tc_U|.\]
\end{lemma}

\begin{proof}
    We analyze the expectation of the Gram matrix $\tilde{V}_{i,U}$ as follows. 
    \begin{align*}
        \Eb[\tilde{V}_{i,U}]  & = \Eb\left[\sum_{\tau\in\Tc_U}x_{i,\tau,a_{i,\tau}}x_{i,\tau,a_{i,\tau}}^{\TT}\right]\\
        & \geq \Eb\left[\sum_{\tau\in\Tc_U}\phi(c_{i,\tau},a_{i,\tau})\phi(c_{i,\tau},a_{i,\tau})^{\TT}\mathbbm{1}\{a_{i,\tau}=a_{c_{i,\tau}^*}, \forall b\neq a_{c_{i,\tau}}^*, (\phi(c_{i,\tau},a_{c_{i,\tau}}^*) - \phi(c_{i,\tau},b))^{\TT}\theta^*\geq \Delta\}\right]\\
        &=\Eb\left[\sum_{\tau\in\Tc_U}\phi(c_{i,\tau},a_{c_{i,\tau}}^*)\phi(c_{i,\tau},a_{c_{i,\tau}}^*)^{\TT}\mathbbm{1}\{a_{i,\tau}=a_{c_{i,\tau}^*}, \forall b\neq a_{c_{i,\tau}}^*, (\phi(c_{i,\tau},a_{c_{i,\tau}}^*) - \phi(c_{i,\tau},b))^{\TT}\theta^*\geq \Delta\}\right] \\
        & = \Eb\left[\sum_{\tau\in\Tc_U}\phi(c_{i,\tau},a_{c_{i,\tau}}^*)\phi(c_{i,\tau},a_{c_{i,\tau}}^*)^{\TT}\right] \\
        &\quad - \Eb\left[\sum_{\tau\in\Tc_U}\phi(c_{i,\tau},a_{c_{i,\tau}}^*)\phi(c_{i,\tau},a_{c_{i,\tau}}^*)^{\TT} \mathbbm{1}\{\exists b\neq a_{c_{i,\tau}}^*, (\phi(c_{i,\tau},a_{c_{i,\tau}}^*) - \phi(c_{i,\tau},b))^{\TT}\theta^* < \Delta\}\right]  \\
        &\quad -\Eb\left[\sum_{\tau\in\Tc_U}\phi(c_{i,\tau},a_{c_{i,\tau}}^*)\phi(c_{i,\tau},a_{c_{i,\tau}}^*)^{\TT}\mathbbm{1}\{a_{i,\tau}\neq a_{c_{i,\tau}^*}, \forall b\neq a_{c_{i,\tau}}^*, (\phi(c_{i,\tau},a_{c_{i,\tau}}^*) - \phi(c_{i,\tau},b))^{\TT}\theta^*\geq \Delta\}\right]\\
        &\overset{(a)}\geq (\lambda_0- C_0\Delta)|\Tc_U|I_d - I_d\Eb\left[\sum_{\tau\in\Tc_U}\mathbbm{1}\{a_{i,\tau}\neq a_{c_{i,\tau}^*}, \forall b\neq a_{c_{i,\tau}}^*, (\phi(c_{i,\tau},a_{c_{i,\tau}}^*) - \phi(c_{i,\tau},b))^{\TT}\theta^*\geq \Delta\}\right]\\
        &\overset{(b)}\geq (\lambda_0- C_0\Delta)|\Tc_U|I_d - \frac{2\alpha\sqrt{d|\Tc_U|\log|\Tc_U|}}{\Delta}I_d, \text{ with probability at least } 1-\beta_0,
    \end{align*}
where $\alpha = 1 + \sqrt{2\log(M/\beta_0) + d\log|\Tc_U|}$, $(a)$ is due to the diversity condition in~\Cref{asm:diversity} and the margin condition in~\Cref{asm:margin}, and $(b)$ follows from \Cref{lemma:bound number of  non optimal arm with large margin}.

Now choose $\Delta = \frac{\lambda_0}{2C_0}$. We have, with probability at least $1-\beta_0$,
\begin{align*}
    \Eb[\tilde{V}_{i,U}]\geq \left(\frac{\lambda_0|\Tc_U|}{2} - \frac{4\alpha C_0\sqrt{d|\Tc_U|\log|\Tc_U|}}{\lambda_0}\right) I_d.
\end{align*}

Consider the difference $\tilde{V}_{i,U} - \Eb[\tilde{V}_{i,U}]$, which is a summation of zero-mean matrices. Thus, we can apply matrix concentration inequality in \Cref{lemma:matrix concentration} to obtain that, for any $\beta_1>0$,
with probability at least $1-\beta_1-\beta_0$, we have
\begin{align*}
    \lambda_{\min}(\tilde{V}_{i,U})& = \lambda_{\min}\left(\tilde{V}_{i,U} - \Eb[\tilde{V}_{i,U}] + \Eb[\tilde{V}_{i,U}] \right) \\
    &\geq \lambda_{\min}\left(\Eb[\tilde{V}_{i,U}])  -  \lambda_{\max}(\tilde{V}_{i,U} - \Eb[\tilde{V}_{i,U}] \right)\\
    &\geq \frac{\lambda_0|\Tc_U|}{2} - \frac{4\alpha C_0\sqrt{d|\Tc_U|\log|\Tc_U|}}{\lambda_0} - \sqrt{2|\Tc_U|\log(d/\beta_1)} - 2/3.
\end{align*}

Choosing $\beta_1 = 1/(M\beta_0)$ and taking the union bound, we have, with probability at least $1-2\beta_0$,
\begin{align*}
    \lambda_{\min}(\tilde{V}_{i,U})
    \geq \frac{\lambda_0|\Tc_U|}{2} - \frac{12 (dC_0+\lambda_0)\sqrt{|\Tc_U|\log(dM/\beta_0)}\log|\Tc_U|}{\lambda_0}.
\end{align*}

We finish the proof by choosing $\beta_0 = \beta/2$ and noting that
\begin{align*}
    \frac{12 (dC_0+\lambda_0)\sqrt{|\Tc_U|\log(2dM/\beta)}\log|\Tc_U|}{\lambda_0}\leq \frac{\lambda_0}{4}|\Tc_U|.
\end{align*}

\end{proof}

{\bf Step 2: Inductively Bound the Minimum Eigenvalue and Global Estimation Error after Initialization.}

\begin{proposition}\label{prop:good event}
    Fix $\beta>0$, $\varepsilon_0$, $\delta_0$, and $c_1 = \frac{4\sqrt{2d\log(16d(M+1)P/\beta)}}{\lambda_0}$ as defined in \Cref{alg:greedy}. If there are total $P$ phases, consider the following events:
\begin{align*}
    &\Ec_{V} = \left\{\forall i\in[M],p\in[U:P],~~ \tilde{V}_{i,p}\geq \frac{\lambda_0|\Tc_p|}{4} I_d\right\},\\
    &\Ec_{\theta} = \left\{\forall p\in[U+1:P],~~ \|\hat{\theta}^{p} -\theta^*\|\leq \frac{c_2}{\sqrt{M|\Tc_p|}}\right\},
\end{align*}
where the parameters are set as
\begin{align*}
    &|\Tc_p| = 2^p, \forall p\in[P],\\
    &U= \left\lceil \max\left\{\log_2\left(\frac{64\log(2dMP/\beta)}{\lambda_0^2}\right), \log_2\left(\frac{144dU^2(C_0+\lambda_0)^2\log(2dMP/\beta) \log 2}{\lambda_0^4} \right)\right\}\right\rceil\\
    &c_2 = \frac{4\sqrt{2d\log(16d(M+1)P/\beta)}}{\lambda_0}\left(\frac{80\log(16dP/\beta)\sqrt{6d\log(16dMP/\beta)\log(1/\delta_0)}}{\sqrt{M}\varepsilon_0}  + 1 \right),\\
    & M\geq \max\left\{\frac{8\sqrt{6d\log(1/\delta_0)}}{\varepsilon_0}\log\left(\frac{4d^2P\sqrt{|\Tc_P|}}{10\beta c_1\sqrt{\log(16dMP/\beta)}}\right),  \frac{32C_0^2c_2^2}{\lambda_0^2|\Tc_{U}|} \right\}.
\end{align*}

Then, \Cref{alg:greedy} guarantees that
\begin{align*}
    \Pb\left(\Ec_V\Ec_{\theta}\right) \geq 1-\beta.
\end{align*}

\end{proposition}

\begin{proof}

Let us divide $\Ec_V$ and $\Ec_{\theta}$ into disjoint sub-events:
\begin{align*}
    &\Ec_{V,p} = \left\{\forall i\in[M],~~\tilde{V}_{i,p}\geq \frac{\lambda_0|\Tc_p|}{4} I_d\right\},\\
    &\Ec_{\theta,p} = \left\{\forall i\in[M],\|\hat{\theta}^{p+1} -\theta^*\|\leq \frac{c_2}{\sqrt{M|\Tc_p|}}\right\},
\end{align*}
such that $\Ec_V = \cap_{p\geq U}\Ec_{V,p}$ and $\Ec_{\theta} = \cap_{p\geq U}\Ec_{\theta,p}$. We aim to verify the following two claims.

\begin{claim}\label{claim:min eigenvalue to global error}
   Under event $\Ec_{V,p}$, for any $\beta>0$, we have $\Pb(\Ec_{\theta,p})\geq 1-\beta/(2P)$.
\end{claim}

\begin{claim}\label{claim:global error to min eigenvalue}
    Under event $\Ec_{\theta,p}$, for any $\beta>0$, we have $\Pb(\Ec_{V,p+1})\geq 1-\beta/(2P)$.
\end{claim}

{\it Proof of \Cref{claim:min eigenvalue to global error}:}

We recall the definition of local estimators $\tilde{\theta}_{i,p}$:
    \begin{align*}
        \tilde{\theta}_{i,p} = \tilde{V}_{i,p}^{\dagger}\left(\sum_{\tau\in\Tc_p}x_{i,\tau,a_{i,\tau}}r_{i,\tau}\right) = \tilde{V}_{i,p}^{\dagger}\left(\sum_{\tau\in\Tc_p}x_{i,\tau,a_{i,\tau}}\big(\eta_{i,\tau} + x_{i,\tau,a_{i,\tau}}^{\TT}\theta^*\big)\right),
    \end{align*}
where $\eta_{i,\tau} = r_{i,\tau} - x_{i,\tau,a_{i,\tau}}^{\TT}\theta^*$ is the IID Gaussian noise. Since $\Ec_{V,p}$ asserts that the covariance matrix $\tilde{V}_{i,p}$ is full rank, we have
    \begin{align*}
        \tilde{\theta}_{i,p}& - \theta^* = \tilde{V}_{i,p}^{-1}\left(\sum_{\tau\in\Tc_p}x_{i,\tau,a_{i,\tau}}\eta_{i,\tau}\right),
    \end{align*}
    which is a summation of independent $\sigma_{i,\tau}^2$-sub-Gaussian random variables conditioned on $\tilde{V}_{i,p}$, and
    \begin{align*}
        \sigma_{i,\tau}^2 \leq \|\tilde{V}_{i,p}^{-1}x_{i,\tau,a_{i,\tau}}\|^2\leq \frac{1}{\lambda_{\min}(\tilde{V}_{i,p})^2}\leq \frac{16}{\lambda_0^2|\Tc_p|^2}.
    \end{align*}

    Due to the independence of $\{\eta_{i,\tau}\}_{\tau}$ conditioned on $\{x_{i,\tau}\}$, we have that $\tilde{\theta}_{i,p} - \theta^* $ is a $\sigma^2$-sub-Gaussian random vector with $\sigma^2 = \frac{16}{\lambda_0^2|\Tc_p|}$. Thus, for any $\beta_1>0$, we have with probability at least $1-\beta_1$,
    \begin{align*}
        \|\tilde{\theta}_{i,p} - \theta^*\| \leq \frac{4\sqrt{2d\log(2d/\beta_1)}}{\lambda_0\sqrt{|\Tc_p|}}.
    \end{align*}
   
    In addition, note that 
    \begin{align*}
        \frac{1}{M}\sum_{i}\big(\tilde{\theta}_{i,p} - \theta^*\big)    = \frac{1}{M}\sum_{i,\tau}\tilde{V}_{i,p}^{-1}\left(x_{i,\tau,a_{i,\tau}}\eta_{i,\tau}\right),
    \end{align*}
    which is a sub-Gaussian random vector conditioned on covariance matrices $\{\tilde{V}_{i,p}\}_i$. Following the same argument, we have
    \begin{align*}
        \Pb\left(\left\|\frac{1}{M}\sum_{i}\big(\tilde{\theta}_{i,p}-\theta^*\big)\right\| \leq \frac{4\sqrt{2d\log(2d/\beta_1)}}{\lambda_0\sqrt{M|\Tc_p|}}\right) \geq 1-\beta_1, \forall \beta_1>0.
    \end{align*}

    So far, by rescaling $\beta_1 $ to $\beta_1/(M+1)$, we have verified that for any $\beta_1>0$, under event $\Ec_{V,p} = \{\forall i,~~ \lambda_{\min}(\tilde{V}_{i,p})\geq \lambda_0|\Tc_p|/4\}$,
    \begin{align}
        \Pb\left(\forall i,~~ \|\tilde{\theta}_{i,p} - \theta^*\| \leq \frac{4\sqrt{2d\log(2d(M+1)/\beta_1)}}{\lambda_0\sqrt{|\Tc_p|}}, \text{ and } \left\|\frac{1}{M}\sum_i\big(\tilde{\theta}_{i,p}-\theta^*\big)\right\| \leq \frac{4\sqrt{2d\log(2d(M+1)/\beta_1)}}{\lambda_0\sqrt{M|\Tc_p|}} \right) \geq 1-\beta_1.\label{eqn:non-private estimation guarantee}
    \end{align}

Let $c_1 = 4\sqrt{2d\log(2d(M+1)/\beta_1)}/\lambda_0$. Then, $\{\tilde{\theta}_{i,p}\}_{i\in[M]}$ are $(c_1/\sqrt{|\Tc_p|},\beta_1)$-concentrated. By the design of \Cref{alg:greedy} and \Cref{coro: winsorized guarantee}, we have
\begin{align*}
        \Pb&\left(\|\hat{\theta}^{p+1} - \frac{1}{M}\sum_i\tilde{\theta}_{i,p}\|\geq \frac{80c_1\log(d/\beta_1)\sqrt{6d\log(dM/\beta_1)\log(1/\delta_0)}}{M\varepsilon_0\sqrt{|\Tc_p|}} \right)\\
        &\leq 3\beta_1 + \frac{d^2\sqrt{|\Tc_p|}}{10c_1\sqrt{(\log(dM/\beta_1))}}\exp\left(-\frac{M\varepsilon_0}{8\sqrt{6d\log(1/\delta_0)}}\right).
    \end{align*}

    Combining with \Cref{eqn:non-private estimation guarantee}, we conclude that, under event $\Ec_{V,p}$,
    \begin{align*}
        \Pb&\left(  \|(\hat{\theta}^{p+1} - \theta^*\| \geq \frac{80c_1\log(d/\beta_1)\sqrt{6d\log(dM/(\beta_1))\log(1/\delta_0)}}{\sqrt{|\Tc_p|}M\varepsilon_0}  + \frac{c_1}{\sqrt{M|\Tc_p|}} \right)\\
        &\leq 4\beta_1 +  \frac{d^2\sqrt{|\Tc_p|}}{10c_1\sqrt{(\log(dM/\beta_1))}}\exp\left(-\frac{M\varepsilon_0}{8\sqrt{6d\log(1/\delta_0)}}\right),
    \end{align*}
    where $c_1 = \frac{4\sqrt{2d\log(d(M+1)/\beta_1)}}{\lambda_0}$, and $\beta_1$ is an arbitrary positive number.

    Let 
    \begin{align*}
        &\beta_1 = \beta/(16P),\\
        &c_1 = \frac{4\sqrt{2d\log(16d(M+1)P/\beta)}}{\lambda_0},\\
        &c_2 =  c_1\left(\frac{80\log(d/\beta_1)\sqrt{6d\log(16dMP/\beta)\log(1/\delta_0)}}{\sqrt{M}\varepsilon_0}  + 1 \right).
    \end{align*}

    Note that 
    \begin{align*}
        M\varepsilon_0\geq 8\sqrt{6d\log(1/\delta_0)}\log\left(\frac{4d^2P\sqrt{|\Tc_p|}}{10\beta c_1\sqrt{\log(16dMP/\beta)}}\right).
    \end{align*}
    
    We simplify the result as follows. Under event $\Ec_{V,p}$,
    \begin{align*}
        \Pb\left(  \|(\hat{\theta}^{p+1} - \theta^*\| \geq \frac{c_2}{\sqrt{M|\Tc_p|}}\right) \leq \frac{\beta}{2P},
    \end{align*}
which completes the proof of \Cref{claim:min eigenvalue to global error}.

Next, we prove \Cref{claim:global error to min eigenvalue}.

{\it Proof of \Cref{claim:global error to min eigenvalue}:}

Recall that at each time slot $t\in\Tc_{p+1}$, client $i$ greedily chooses an arm $a_{i,t}$ with respect to the estimator $\hat{\theta}^{p+1}$ under context $c_{i,t}$ drawn from $\rho_i$. We consider the matrix $\mathbb{E}_{c_{i,t}\sim\rho_i}\left[\phi(c_{i,t},a_{i,t})\phi(c_{i,t}, a_{i,t})^{\TT}\right]$ under the assumption that $\Ec_{\theta,p}$ holds. We have
\begin{align*}
    \mathbb{E}_{c_{i,t}\sim\rho} &\left[\phi(c_{i,t},a_{i,t})\phi(c_{i,t}, a_{i,t})^{\TT}\right] \\
    & = \sum_{a\in\mathcal{A}} \mathbb{E}_{c\sim\rho_i}\left[\phi(c,a)\phi(c,a)^{\TT}\mathds{1}\left\{\forall b\neq a,~ (\phi(c,a)-\phi(c,b))^{\TT}\hat{\theta}^{p+1}\geq 0 \right\}\right]\\
    & = \sum_{a\in\mathcal{A}} \mathbb{E}_{c\sim\rho}\left[\phi(c,a)\phi(c,a)^{\TT}\mathds{1}\left\{\forall b\neq a,~ (\phi(c,a)-\phi(c,b))^{\TT}\theta^* \geq (\phi(c,a)-\phi(c,b))^{\TT}(\theta^* - \hat{\theta}^{p+1}) \right\}\right]\\
    &\overset{(a)}\geq \sum_{a\in\mathcal{A}} \mathbb{E}\left[\phi(c,a)\phi(c,a)^{\TT}\mathds{1} \left\{\forall b\neq a,~ (\phi(c,a)-\phi(c,b))^{\TT}\theta^* \geq \frac{2c_2}{\sqrt{M|\Tc_p|}} \right\}\right]\\
    & = \sum_{a\in\mathcal{A}} \mathbb{E}\left[\phi(c,a)\phi(c,a)^{\TT}\mathds{1} \left\{\forall b\neq a,~ (\phi(c,a)-\phi(c,b))^{\TT}\theta^* \geq \frac{2c_2}{\sqrt{M|\Tc_p|}}  \right\}\right]\\
    & = \mathbb{E}\left[\phi(c,a_c^*)\phi(c,a_c^*)^{\TT}\mathds{1} \left\{\forall b\neq a_c^*,~ (\phi(c,a_c^*)-\phi(c,b))^{\TT}\theta^* \geq \frac{2c_2}{\sqrt{M|\Tc_p|}} \right\}\right]\\
    & = \mathbb{E}\left[\phi(c,a_c^*)\phi(c,a_c^*)^{\TT}\right]\\
    & \quad - \mathbb{E}\left[\phi(c,a_c^*)\phi(c,a_c^*)^{\TT}\mathds{1} \left\{\exists b\neq a_c^*,~ (\phi(c,a_c^*)-\phi(c,b))^{\TT}\theta^* < \frac{2c_2}{\sqrt{M|\Tc_p|}}  \right\}\right]\\
    &\geq \lambda_0 I  - I \mathbb{P}\left(\exists b\neq a_c^*,~ (\phi(c,a_c^*)-\phi(c,b))^{\TT}\theta^* < \frac{2c_2}{\sqrt{M|\Tc_p|}} \right)\\
    &\overset{(b)}\geq \left(\lambda_0  - \frac{2C_0c_2}{\sqrt{M|\Tc_p|}} \right) I_d,
\end{align*}
where $(a)$ follows from $\Ec_{\theta,p}$, and $(b)$ is due to the margin condition in \Cref{asm:margin}.

Therefore, for any $p\geq U$, we have
\begin{align*}
\mathbb{E} \left[\tilde{V}_{i,p+1}\right]
&\geq \lambda_0 |\Tc_{p+1}| -  |\Tc_{p+1}| \frac{2C_0c_2}{\sqrt{M|\Tc_p|}}  \overset{(a)}= \lambda_0|\Tc_{p+1}| - 2C_0c_2\sqrt{\frac{2|\Tc_{p+1}|}{M}},
\end{align*}
where $(a)$ is due to $|\Tc_{p+1}| = 2^{p+1} = 2|\Tc_{p}|$.

Thus, we can apply the matrix concentration inequality \Cref{lemma:matrix concentration} on the martingale difference $\phi(c_{i,t},a_{i,t})\phi(c_{i,t},a_{i,t})^{\TT} - \Eb[\phi(c_{i,t},a_{i,t})\phi(c_{i,t},a_{i,t})^{\TT}]$ for $t\in\Tc_p/\Tc_U$ to make the following conclusions: With probability at least $1-\beta_0$, where $\beta_0>0$ is any positive number, we have
\begin{align*}
    \lambda_{\min}(\tilde{V}_{i,p+1})  & = \lambda_{\min}\left(\tilde{V}_{i,p+1} - \Eb[\tilde{V}_{i,p+1}] + \Eb[\tilde{V}_{i,p+1}]\right)\\
    & \geq \lambda_{\min}\left(\Eb[\tilde{V}_{i,p+1}]\right) - \lambda_{\max}(\tilde{V}_{i,p+1} - \Eb[\tilde{V}_{i,p+1}])\\
    & \geq \lambda_0|\Tc_{p+1}| - 2C_0c_2\sqrt{\frac{2|\Tc_{p+1}|}{M}} -\sqrt{2|\Tc_{p+1}|\log(d/\beta_0)} - 2/3.
\end{align*}

Now, set $\beta_0 = \beta/(2MP)$, and take the union bound over all clients. Then, with probability at least $1-\beta/(2P)$, for all $i\in[M]$ the following holds.
\begin{align*}
    \lambda_{\min}(\tilde{V}_{i,p+1})\geq \lambda_0|\Tc_{p+1}| - 2C_0c_2\sqrt{\frac{2|\Tc_{p+1}|}{M}} -\sqrt{2|\Tc_{p+1}|\log(2dMP/\beta)} - 2/3,
\end{align*}
where $c_2 = \frac{8\sqrt{2d\log(16d(M+1)P/\beta)}}{\lambda_0}\left(\frac{80\log(16dP/\beta)\sqrt{6d\log(16dMP/\beta)\log(1/\delta_0)}}{\sqrt{M}\varepsilon_0}  + 1 \right).$

Note that $p\geq U$, and $|\Tc_U|$ satisfies the following inequality:
\begin{align*}
    \sqrt{2|\Tc_{U}|\log(2dMP/\beta)} + 2/3 \leq \frac{\lambda_0|\Tc_{U}|}{4}.
\end{align*}
In addition, due to $p\geq U$, we have that
\begin{align*}
    2\sqrt{2}C_0c_2 \leq \frac{\lambda_0\sqrt{M|\Tc_{p+1}|}}{2}
\end{align*}
holds for any $p$.

Thus, we conclude that, under event $\Ec_{\theta,p}$,
\begin{align*}
    \Pb\left(\forall i\in[M],~~ \lambda_{\min}(\tilde{V}_{i,p+1})\geq \frac{\lambda_0|\Tc_{p+1}|}{4}\right) \geq 1-\frac{\beta}{2P},
\end{align*}
which finishes the proof of \Cref{claim:global error to min eigenvalue}.

Then, based on \Cref{claim:global error to min eigenvalue} and \Cref{claim:min eigenvalue to global error}, combining with \Cref{lemma:first U phase full rank}, we conclude that
\begin{align*}
    \Pb\left(\Ec_V\Ec_{\theta}\right)  = \Pb\left(\big(\cap_{p\geq U}\Ec_{V,p}\big) \cap \big(\cap_{p\geq U}\Ec_{\theta,p} \big) \right)\geq 1-\beta.
\end{align*}
\end{proof}

\if{0}

\begin{lemma}
With probability $1-\delta$, we have 
\[\forall i,\quad \|\hat{\theta}_{i,p} - \theta^*\|_2\leq \sqrt{\frac{\beta}{T_{p}}}\]
\end{lemma}

\begin{proof}

Let 
\[\lambda_{\min} = \lambda_{\min}\left(\sum_{t\in[T_p]}x_{i,t,a_{i,t}}x_{i,t,a_{i,t}}^{\TT}\right).\]

We already have, with probability at least $1-\delta$,
\[\lambda_{\min}\geq \lambda_0T_p/2\]
Thus, 
\begin{align}
    \hat{\theta}_{i,p} &= \frac{1}{\lambda_{\min}}\sum_{t\in[T_p]}w_{i,t}x_{i,t,a_{i,t}}y_{i,t}\\
    & = \frac{1}{\lambda_{\min}}\sum_{t\in[T_p]}w_tx_{i,t,a_{i,t}}x_{i,t,a_{i,t}}^{\TT}\theta^* + \frac{1}{\lambda_{\min}}\sum_{t\in[T_p]}w_tx_{i,t,a_{i,t}}\eta_{i,t}\\
    & = \theta^* + \frac{1}{\lambda_{\min}}\sum_{t\in[T_p]}w_tx_{i,t,a_{i,t}}\eta_{i,t}
\end{align}
Note that $w_t\in[0,1]$, and $\eta_{i,t}$ are independent 1-sub-Gaussian random variables, we conclude that

with probability at least $1-\delta/M$,
\[\left\|\sum_{t\in[T_p]}w_tx_{i,t,a_{i,t}}\eta_{i,t}\right\|_2 \leq 2\sqrt{T_p}(2\sqrt{d } + \sqrt{2\log(M/\delta)})\]

Therefore, with probability at least $1-\delta$, for any client $i$
\[\|\hat{\theta}_{i,p} - \theta^*\|_2 \leq \frac{4\sqrt{T_p}(2\sqrt{d}+\sqrt{2\log(M/\delta)})}{\lambda_0T_p} = \frac{4(2\sqrt{d}+\sqrt{2\log(M/\delta)})}{\lambda_0\sqrt{T_p}} \leq \sqrt{\frac{\beta}{T_p}}\]
since $\beta \geq \frac{16(2\sqrt{d}+\sqrt{2\log(M/\delta)})^2}{\lambda_0^2}$.

\end{proof}

\begin{lemma}[DP guarantee. Theorem 2  and Corollary 1 in  \citep{levy2021learning}]
 WinsorizedMeanHighD($\hat{\theta}_{i,p}$) satisfies $(\varepsilon,\delta)$-DP and 
\[\mathbb{P}\left(\|\hat{\theta}^{(p+1)} - \theta^*\|_2 \leq \sqrt{\frac{\beta}{MT_p}}\right) \geq 1-\delta \]
\end{lemma}

\begin{proof}

Let $\tilde{\theta}^{p+1} = \frac{\sum_{i}\hat{\theta}_{i,p}}{M}$, and $\xi\sim Lap()$. Then,
by Theorem 2  and Corollary 1 in  \citep{levy2021learning}, we have for any subset $A\subset \mathbb{R}^d$,
\begin{align}
   \left| \mathbb{P}(\hat{\theta}^{p+1}\in A) - \mathbb{P}(\tilde{\theta}^{p+1}+\xi\in A) \right|\leq \delta
\end{align}

Since $\xi\sim Lap(\frac{\sqrt{\frac{2}{\lambda_o T_p}}\log d/\delta}{M\varepsilon})$, we have 
\begin{align}
    \mathbb{P}\left( \|\xi\|_2\geq \frac{2\sqrt{d}\log(dM/\delta)}{M\varepsilon\lambda_0\sqrt{T_p}}\right)\leq \delta
\end{align}
In addition, 
\begin{align}
    \mathbb{P}\left( \|\tilde{\theta} - \theta^*\|_2\geq \frac{2\sqrt{d}\log(M/\delta)}{\lambda_0\sqrt{ MT_p}}\right)\leq \delta
\end{align}

Therefore
\begin{align}
    \mathbb{P}\left(\|\tilde{\theta}^{p+1} + \xi - \theta^*\|_2\geq \max(\frac{1}{\varepsilon\sqrt{M}},1)\frac{2\sqrt{d}\log Md/\delta}{\lambda_0\sqrt{ MT_p}} \right)\leq 2\delta
\end{align}

We conclude that 
\begin{align}
    \mathbb{P}\left(\|\hat{\theta}^{p+1} - \theta^*\|_2\geq \max(\frac{1}{\varepsilon\sqrt{M}},1)\frac{2\sqrt{d}\log Md/\delta}{\lambda_0\sqrt{ MT_p}} \right)\leq 3\delta
\end{align}

Finally let $\beta \geq \max(\frac{1}{\varepsilon^2M},1)\frac{d\log^2Md/\delta}{\lambda_0^2}$

\end{proof}
\fi

{\bf Step 3: Upper Bound the Regret.}

Now, we are ready to provide the final result on the upper bound of the regret of \Cref{alg:greedy}.
\begin{theorem}[Regret]\label{thm:greedy upper bound}
Under the parameter setting in \Cref{prop:good event}, with probability at least $1-\beta$, the total regret of \Cref{alg:greedy} is upper bounded by
\[\text{Regret}(M,T)\leq \tilde{O}\left(\max\left(1, \frac{d\log T\log(1/\delta)\log^3(1/\beta)}{\varepsilon^2 M}\right)\frac{C_0d\log(1/\beta)\log T}{\lambda_0^2}\right). \]
\end{theorem}

\begin{proof}
Consider a phase $p>U$. Under the events $\Ec_V$ and $\Ec_{\theta}$ defined in \Cref{prop:good event}, we have
\begin{align*}
    \Eb&[r_{i,t}^* - r_{i,t}] = \Eb\left[\left(\phi(c_{i,t},a_{c_{i,t}}^*) - \phi(c_{i,t},a_{i,t})\right)^{\TT}\theta^*\right]\\
    & = \Eb\left[\mathds{1}\left\{a_{i,t}\neq a_{c_{i,t}}^*\right\}\left(\phi(c_{i,t},a_{c_{i,t}}^*) - \phi(c_{i,t},a_{i,t})\right)^{\TT}\theta^*\right]\\
    &\leq \Eb\left[\mathds{1}\left\{a_{i,t}\neq a_{c_{i,t}}^* \right\}\left(\phi(c_{i,t},a_{c_{i,t}}^*) - \phi(c_{i,t},a_{i,t})\right)^{\TT}(\theta^* - \hat{\theta}^{(p)})\right]\\
    & \overset{(a)}\leq \Eb\left[\mathds{1}\left\{a_{i,t}\neq a_{c_{i,t}}^*\right\}\frac{2c_2}{\sqrt{M|\Tc_{p-1}|}}\right]\\
    & = \frac{2c_2}{\sqrt{M|\Tc_{p-1}|}}\mathbb{P}(a_{i,t} \text{ is not optimal })\\
    & = \frac{2c_2}{\sqrt{M|\Tc_{p-1}|}}\Pb\left(   \big(\phi(c_{i,t},a_{i,t}) - \phi(c_{i,t},a_{c_{i,t}}^*)\big)^{\TT}\hat{\theta}^{(p)}\geq 0 \right)\\
    & = \frac{2c_2}{\sqrt{M|\Tc_{p-1}|}}\Pb\left(0>\big(\phi(c_{i,t},a_{i,t}) - \phi(c_{i,t},a_{c_{i,t}}^*)\big)^{\TT}\theta^*\geq \big(\phi(c_{i,t},a_{i,t}) - \phi(c_{i,t},a_{c_{i,t}}^*)\big)^{\TT}(\theta^*-\hat{\theta}^{(p)}) \right)\\
    &\overset{(b)}\leq \frac{2c_2}{\sqrt{M|\Tc_{p-1}|}}\Pb\left( 0 < \big( \phi(c_{i,t},a_{c_{i,t}}^*) - \phi(c_{i,t},a_{i,t}) \big)^{\TT}\theta^*\leq \frac{2c_2}{\sqrt{M|\Tc_{p-1}|}} \right) \\
    &\overset{(c)}\leq \frac{4C_0c_2^2}{M|\Tc_{p-1}|},
\end{align*}
where $(a)$ and $(b)$ follow from event $\Ec_{\theta}$, and $(c)$ is due to the margin condition in~\Cref{asm:margin}.

Therefore, with probability at least $1-\beta$,
\begin{align*}
    \text{Regret}(M,T)   &= \sum_{p\in[P]}\sum_{i\in[M],t\in\Tc_p}\Eb[r_{i,t}^* - r_{i,t}]\\
    &\leq \sum_{p>U} M|\Tc_p|\cdot \frac{4C_0c_2^2}{M|\Tc_{p-1}|} + M|\Tc_U|\\
    &= \sum_{p>U}8C_0c_2^2 + M|\Tc_U|\\
    & = 8C_0c_2^2P + M|\Tc_U|.
\end{align*}
We  complete the proof by noting that 
\begin{align*}
   & P = O(\log T),\\
&\varepsilon_0 = \varepsilon/\sqrt{P},\\
    & U = O(\log\log\log T),\\
    &c_2 = \frac{4\sqrt{2d\log(16d(M+1)P/\beta)}}{\lambda_0}\left(\frac{80\log(16dP/\beta)\sqrt{6d\log(16dMP/\beta)\log(1/\delta_0)}}{\sqrt{M}\varepsilon_0}  + 1 \right)\\
    &~~~ = \tilde{O}\left(\frac{\sqrt{d \log(1/\beta) } }{\lambda_0}\max\left(1, \frac{\log(1/\beta)\sqrt{d\log T\log(1/\beta)\log(1/\delta)}}{\varepsilon\sqrt{M}}\right)\right).
\end{align*}

\end{proof}

\section{Necessity of Diversity Assumption for Memoryless Algorithms under User-level CDP Constraint}\label{sec:diversity is necessary}

In this section, we prove that the diversity condition in \Cref{asm:diversity} is necessary for achieving sublinear regret when adopting almost-memoryless algorithms under the user-level CDP constraint.

\begin{proposition}
    If $\varepsilon<0.1$ and $\delta<0.001$, then, there exists a federated linear contextual bandits instance not satisfying \Cref{asm:diversity} such that any almost-memoryless algorithm must incur a regret lower bounded by $\Omega(T).$  
\end{proposition}

\begin{proof}
We consider a federated linear contextual bandits with two arms $\{1,2\}$ and $M$ clients, where the context is fixed for each client.
For each client $i\in[M]$, let features $\phi_i(c,1), \phi_i(c,2) \in \mathbb{R}^2$ be defined as follows. For the first client ($i = 1$), $\phi_1(c,1) = (0.5,0)^{\TT}$ and $\phi_1(c,2) = (-0.5,0)^{\TT}$. For any other client ($i> 1$), $\phi_i(c,1) = (0, 0.5)^{\TT}$ and $\phi_i(c,2) = (0, -0.5)^{\TT}$.
Let the model parameter $\theta^* \in \Theta = \{\pm 1\} \times \{\pm 1\}$, and the reward $ r_{i,t} \sim N \left(\phi_i(c,a_{i,t})^T \theta^*, 1\right)$.

We note that the feature distribution does not satisfy the \textbf{diversity assumption} (Assumption~\ref{asm:diversity}), since the smallest eigenvalue is $0$ for each $\phi_i(c,1)$ and $\phi_i(c,2)$, while it satisfies the margin condition (\Cref{asm:margin}) with $C_0 = 1$.

To reduce the dependency between each client, we again consider the extended history similarly in \Cref{sec:hard instance}. Let $H_{i,t} = \{c, a_{i,\tau}, r_{i,\tau}\}_{\tau<t}$ be the history of client $i$, and $\bar{H}_{i,t} = \{c,a=\{1,2\} , r_{i,\tau,a}\}_{\tau<t} \supset H_{i,t}$  be the extended history of client $i$ at time $t$. $\bar{H}_{i,t}$ is independent of $\bar{H}_{j,t}$ conditioned on model parameter $\theta^*$. 

Note that the total regret is lower bounded by the regret of the first user, and the regret of play sub-optimal arm is 1. Then, we have
\begin{align*}
    \text{Regret}(M,T)&\geq \sum_{t\in[T]}\Pb(a_{1,t} \neq e_1^{\TT}\theta^* )\\
    &\geq   \sum_{t\in[T]}\mathop{\inf}_{\hat{\theta}_{1, t}\in\mathcal{F}(\bar{\mathcal{I}},\{1,-1\}) \atop \{\bar{R}_i\}_{i\in[0,M]} }\Pb\left(e_1^{\TT}\theta^* \neq \hat{\theta}_{1,t}(\bar{q}_{\leq t}) \right).
\end{align*}

Now we aim to show that due to the user-level DP constraint, the estimation error cannot be too small for client 1.

Consider $\theta^*$ and $\theta'$, such that $e_1^{\TT}\theta^* = -e_1^{\TT}\theta' =1$, and $e_2^{\TT}\theta^*  = e_2^{\TT}\theta'$. We construct a coupling between $\bar{q}_{\leq t}|\theta^*$ and $\bar{q}_{\leq t}|\theta^{'}$. Specifically, if we flip the distribution of $r_{1,t,1} $ from $N \left(0.5, 1\right)$ to $N \left(-0.5, 1\right)$, and $r_{1,t,2}$'s distribution from $N \left(-0.5, 1\right)$ to $N \left(0.5, 1\right)$, we have changed the distribution from $\bar{H}_{1,t}|\theta^*$ to $\bar{H}_{1,t}|\theta^{'}$ while keeping the other $\bar{H}_{j,t} (j \neq 1)$ unchanged. Furthermore, the expected Hamming distance of this coupling is $1$, since only client 1's data has been changed.

By leveraging the private Le Cam's method (Theorem 1 in~\cite{acharya21}),
\begin{align*}
\mathop{\inf}_{\hat{\theta}_{1, t}\in\mathcal{F}(\bar{\mathcal{I}},\Theta) \atop \{\bar{R}_i\}_{i\in[0,M]} }\Pb\left(e_1^{\TT}\theta^* \neq \hat{\theta}_{1,t}(\bar{q}_{\leq t}) \right) \geq 0.9 e^{-\epsilon} - 10\delta.
\end{align*}
Therefore, we have 
\begin{align*}
\text{Regret}(M,T)& \geq T \left( 0.9 e^{-\epsilon} - 10\delta \right).
\end{align*}

If $\epsilon<1$ and $\delta<0.001$, we conclude that $\text{Regret}(M,T) = \Omega(T)$.

\end{proof}

\section{Proof of the Regret Lower Bounds Under User-level CDP Constraint}\label{sec:proof CDP lower bound}

In this section, we provide the full analysis for CDP lower bounds. The first subsection describes a hard instance of linear contextual bandits model that will be used in the proofs. The remaining subsections provide the proofs of \Cref{thm:main-CDP-lower-margin} and \Cref{thm:main-CDP-minimax}.

\subsection{General Setting of Hard Instance}\label{sec:hard instance}
First, we introduce the notation of truncated Gaussian distributions. If a Gaussian random variable $X\sim N(0,I_d)$ is truncated to $\{x:\|x\|_2\leq r\}$, then we denote the truncated Gaussian distribution of $X$ as $N(0,I_d|r)$. 

In the lower bound analysis, we follow the setting in \citet{he2022reduction}, as specified below. 

{\bf Arms and Dimension:} There are 2 arms: $\{1,2\}$, and the dimension $d$ is an even number. 

{\bf Feature Vectors and the Context Distribution:} The feature vector of the second arm is always $0$. For the feature vector of the first arm, let the distribution of the context $c_i$ for any client $i$ satisfy    $\phi(c_i,1) = (0,\ldots,z_{i,s}^{\TT},\ldots,0)^{\TT}$ with {$s$ uniformly distributed over $[d/2]$}, and $\{z_{i,s}\}_{i,s}\subset\Rb^2$ independently sampled from a truncated normal $N(0,I_2|1)$. 

{\bf Model Parameter and Its Distribution:} The model parameter $\theta^{*} = (\theta_1^{*\TT},\ldots,\theta_{d/2}^{*\TT})^{\TT}\in\Rb^d$ with each $\theta_s^*\in\Rb^2$ sampled independently and uniformly from a sphere $\mathbb{S}_r=\{x\in\Rb^2:\|x\|=r\}$, where $r\in[0,1/\sqrt{d}]$. The constraint on $r$ is due to the boundness assumption that $\|\theta^*\|\leq 1 $. Moreover, the parameter $C_0$ defined in~\Cref{asm:margin} satisfies $C_0=\Omega(1/r)$.

{\bf Notations of Available Information Used to Make Decisions: } Recall that $H_{i,t} = \{c_{i,\tau}, a_{i,\tau}, r_{i,\tau}\}_{\tau<t}$ is the history of client $i$. Note that $r_{i,t}$ is sampled from a Gaussian distribution with mean $\phi(c_{i,t},a_{i,t})^{\TT}\theta^*$ and variance $1$, if the true model is $\theta^*$. We further denote that $\phi(c_{i,t},a) = x_{i,t,a}$.

To reduce the dependency between histories of different clients, we define $\bar{H}_{i,t} = \{c_{i,\tau},a=\{1,2\} , r_{i,\tau,a}\}_{\tau<t} \supset H_{i,t}$ to be the extended history of client $i$ at time $t$, where $r_{i,t,a}$ is the (virtual) reward sampled from (un-played) pulled arm, such that $\bar{H}_{i,t}$ provides full information. It is worth pointing out that $\bar{H}_{i,t}$ is independent with $\bar{H}_{j,t}$ conditioned on model parameter $\theta^*$. With these notations, we introduce $\bar{q}_{i,\leq t}$ and $\bar{q}_{\leq t}$, which are outputs from the ``extended'' DP channels $\bar{R}_i$ and $\bar{R}_0$ with inputs $\bar{H}_{i,t}$ and $\{\bar{q}_{i,t}\}_{i\in[M]}$, respectively. Here, ``extended channels'' implies that $\bar{R}_i(H_{i,t}) = R_i(H_{i,t})$ and $\bar{R}_0(\{q_{i, \leq t}\}_i) = R_0(\{q_{i,\leq t}\}_i)$.

With the general setting described above, we present the generic regret lower bound modified from Proposition 3.5 in \citet{he2022reduction}. Note that while the original result holds for the single-client setting, it is straightforward to extend the result into the federated setting.

\begin{theorem}[ ]\label{thm: Regret to estimation error-appendix}
    For the hard instance described in \Cref{sec:hard instance}, the total regret of $M$ clients can be lower bounded by
    \begin{align*}
    \Omega \left(\sum_{i\in[M],t\in[T],s\in[d/2]}\mathop{\inf}_{\theta_{i,t,s}\in\mathcal{F}(\bar{\mathcal{I}}_i,\mathbb{S}_r) \atop \{\bar{R}_i\}_{i\in[0,M]}} \frac{1}{rd}\Eb_v\left[ \left\|\theta_s^*-  \theta_{i,t,s}\right\|^2 \right] \right),
    \end{align*}
    where $\bar{\mathcal{I}}_i$ is a set of all possible information $(\bar{H}_{i,t},\bar{q}_{\leq t})$ provided to client $i$.
\end{theorem}

Note that $\theta_s^*$ are independently and identically distributed, and $\{\theta_{i,t,s}\}_s$ are from the same set of measurable functions. Without loss of generality, it suffices to lower bound the estimation error for the first parameter $\theta_1^*$, which leads to the following corollary. 

\begin{corollary}\label{coro: generic lower bound}
    Under the same setting in \Cref{thm: Regret to estimation error-appendix}, the total regret of all clients can be lower bounded by
    \[\Omega \left(\sum_{i\in[M],t\in[T]}\mathop{\inf}_{\theta_{i,t}\in\mathcal{F}(\bar{\mathcal{I}}_i,\mathbb{S}_r) \atop \{\bar{R}_i\}_{i\in[0,M]}}  \frac{1}{2r}\Eb_v\left[ \left\|\theta_1^*-  \theta_{i,t}\right\|^2 \right] \right),
    \]
    where $r\in[0,1/\sqrt{d}].$
\end{corollary}

\Cref{coro: generic lower bound} suggests that it suffices to lower bound the estimation error for each single time step $t$ and client $i$, where the estimator $\theta_{i,t}$ is constructed from both the local data $\bar{H}_{i,t}$ and global information $\bar{q}_{\leq t}$. Since $q_{\leq t}$ is a private output from an $(\varepsilon,\delta)$-CDP mechanism, by the post-processing property, $\theta_{i,t}$ is also an $(\varepsilon,\delta)$-CDP estimator. However, we have to be cautious that it is differentilly private only with respect to client $j$'s local data, where $j\neq i$. Mathematically, for any subset $S\subset\mathbb{S}_r$ and $j$-neighboring dataset $H_t$, $H_t'$, where $j\neq i$, we have
\begin{align*}
    \Pb\left(\theta_{i,t}\in S|H_t\right)\leq e^{\varepsilon}\Pb\left(\theta_{i,t}\in S|H_t'\right) + \delta.
\end{align*}

\if{0}
By \Cref{lemma: Regret to estimation error}, for any parameter space $\Theta$ containing $\theta^*$, if $I_{i,t}\in\{q_{\leq t}, (H_{i,t},q_{\leq t})\}$ is the available information for client $i$'s decision at time $t$, then the total regret is lower bounded by
\begin{align}
    Regret(M,T)&\geq \frac{1}{r}\sum_{i\in[M], t\in[T]}\mathop{\inf}_{\hat{\theta}_{i,t}\in\mathcal{F}(\mathcal{I},\Theta) \atop \{R_i\}_{i\in[0,M]}}\Eb\left[ \left\| \theta^* - \frac{r \hat{\theta}_{i,t}(I_{i,t})}{\|\hat{\theta}_{i,t}(I_{i,t})\|} \right\|^2\right]\\
    & \geq \frac{1}{r}\sum_{i,\in[M],t\in[T]}\mathop{\inf}_{\hat{\theta}_{i, t}\in\mathcal{F}(\bar{\mathcal{I}},\Theta) \atop \{\bar{R}_i\}_{i\in[0,M]} }\Eb\left[ \left\|\theta^* - \frac{r \hat{\theta}_{i,t}(\bar{I}_{i,t})}{\|\hat{\theta}_{i,t}(\bar{I}_{i,t})\|} \right\|^2\right],\label{eqn: generic lower bound}
\end{align}
where $\bar{I}\in\{\bar{q}_{\leq t}, (\bar{H}_{i,t},\bar{q}_{\leq t})\}$ is the extended information, $\mathcal{I}, \bar{\mathcal{I}}$ are sets of all possible information $I$ and $\bar{I}$, respectively, and $\mathcal{F}(S_1,S_2)$ is a set of all measurable functions that map some set $S_1$to $S_2$.

Hence, it suffices to bound a single term, denoted by {\it Error}, in \Cref{eqn: generic lower bound}:
\begin{align}
    Error = \mathop{\inf}_{\hat{\theta}_{i, t}\in\mathcal{F}(\bar{\mathcal{I}},\Theta) \atop \{\bar{R}_i\}_{i\in[0,M]} }\Eb\left[ \left\|\theta^* - \frac{r \hat{\theta}_{i,t}(\bar{I}_{i,t})}{\|\hat{\theta}_{i,t}(\bar{I}_{i,t})\|} \right\|^2\right]
\end{align}

We point out that {\it Error} is not exactly a standard estimation error expressed as 
\begin{align}
    Error_0 = \mathop{\inf}_{\hat{\theta}_{i, t}\in\mathcal{F}(\bar{\mathcal{I}},\Theta) \atop \{\bar{R}_i\}_{i\in[0,M]} }\Eb\left[ \left\|\theta^* - \hat{\theta}_{i,t}(\bar{I}_{i,t}) \right\|^2\right],
\end{align}
and in general, these two types of errors are not comparable.

However, there are certain cases when $Error$ can be reduced to $Error_0$.
\begin{itemize}
    \item If $\Theta = \{x:\|x\|= r\}$, then $Error=Error_0$.
    \item If $r$ is sufficiently small, then $Error\geq \Omega(Error_0)$
\end{itemize}

We will leverage the above two facts in the following theorems.

$\mathcal{H}\times \mathcal{Q}$ is a set of all possible values of $(H_{i,t},q_{\leq t})$, and $\bar{\mathcal{H}}\times\bar{\mathcal{Q}}$ is a set of all possible values of $(\bar{H}_{i,t}, \bar{q}_{\leq t})$. 
\fi

\subsection{Proof of \Cref{thm:main-CDP-minimax}}\label{sec:CDP w. memory}
Equipped with \Cref{coro: generic lower bound}, we are able to lower bound the regret by the estimation error. 

Note that both true parameter and the estimator are two dimensional vectors with constant norm. We point out that our setting is different from other works on lower bound of estimation error where each coordinate of the true parameter is sampled from an interval independently \citep{levy2021learning,kamath2019privately}. Hence, we re-parameterize $\theta^*_1$ by its angle, i.e. $\theta_1^* = r(\cos\gamma^*,\sin\gamma^*)^{\TT}$, and $\gamma^*$ is sampled uniformly from the interval $[0,2\pi)$. We further denote $e_1 = (1,0)$ and $e_2 = (0,1)$ that form the canonical basis of $\Rb^2$.

{\it Proof Outline:} {\bf Step 1} is to decompose the estimation error to the expectations of $M$ random variables $\{Z_i\}_i$ which capture the covariance of the global estimator and local data of each client $i$. {\bf Step 2} upper bounds $\Eb[Z_i]$ for all $i\in[M]$ under the CDP constraint, indicating that the estimation error is bounded from below. {\bf Step 3} combines the previous steps to prove the final regret lower bound.

To simplify notations, in Step 1 and Step 2, we fix a time step $t$ and a client $i_0$, and aim to bound the estimation error 
$\Eb\left[ \left\|\theta_1^*-  \hat{\theta}\right\|^2 \right]$,
where $\hat{\theta}$ is the optimal solution of $\mathop{\inf}_{\hat{\theta}\in\mathcal{F}(\bar{\mathcal{I}}_{i_0},\mathbb{S}_r) \atop \{\bar{R}_i\}_{i\in[0,M]}} \Eb\left[ \left\|\theta_1^*-  \hat{\theta}\right\|^2 \right]$.

{\bf Step 1: Decompose the Estimation Error.}

We note that similar result is obtained when each coordinate of $\theta^*$ is independently sampled. (See Lemma 6.8 in \cite{kamath2019privately} and Lemma 3.6 in \cite{bun2017make}.)

\begin{lemma}[Fingerprinting Lemma\footnote{We adopt the fingerprinting lemma rather than DP Assouad's method \citep{acharya21}, because in general, the lower bound obtained by DP Assouad's method has an additional blow-up factor $\sqrt{d}$ compared to that obtained from the fingerprinting lemma.}
]\label{lemma: Fingerprinting}
    Define random variables $Z_i$ for each $i$ as follows.
    \begin{align*}
    &Z_{i} = (\hat{\theta} - \theta_1^*)^{\TT}(-e_1\sin\gamma^* + e_2\cos\gamma^*)(-e_1\sin\gamma^* + e_2\cos\gamma^*)^{\TT}\bar{V}_i(\bar{\theta}_i - \theta^*) ,
\end{align*}
where $\bar{V}_i = \sum_{\tau<t}x_{i,\tau,1}x_{i,\tau,1}^{\TT}$, and $\bar{\theta}_i = \bar{V}_i^{\dagger}\left(\sum_{\tau<t}x_{i,\tau,1}r_{i,\tau,1}\right)$, and recall that $r_{i,\tau,1}$ is sampled from $N(x_{i,\tau,1}^{\TT}\theta_1^*,1)$.

Then, we have
\begin{align*}
    \Eb\left[ \left\|\theta_1^*-  \hat{\theta}\right\|^2 \right] = 2r^2 - 2r^2\sum_{i}\Eb[Z_i] .
\end{align*}
\end{lemma}

\begin{proof}
Due to $\|\theta_1^*\| = \|\hat{\theta}\| = r$, it suffices to analyze the term $\Eb\left[\hat{\theta}^{\TT}\theta_1^*\right].$ Note that $\theta_1^* = r(\cos\gamma^*,\sin\gamma^*)^{\TT}$. Then, we have
\begin{align*}
\Eb\left[\hat{\theta}^{\TT}\theta_1^*\right]& = \frac{r}{2\pi}\int_0^{2\pi} e_1^{\TT}\Eb[\hat{\theta}|\gamma^*]\cos\gamma^* + e_2^{\TT}\Eb[\hat{\theta}|\gamma^*]\sin\gamma^* d\gamma^*\\
    & = \frac{r}{2\pi} \left(e_1^{\TT}\Eb[\hat{\theta}|\gamma^*]\sin\gamma^* - e_2^{\TT}\Eb[\hat{\theta}|\gamma^*]\cos\gamma^* \right)\bigg|_{\gamma^*=0}^{\gamma^*=2\pi}\\
    &\quad - \frac{r}{2\pi}\int_0^{2\pi} e_1^{\TT}\frac{\partial}{\partial\gamma^*}\Eb[\hat{\theta}|\gamma^*]\sin\gamma^* + e_2^{\TT}\frac{\partial}{\partial\gamma^*}\Eb[\hat{\theta}|\gamma^*]\cos\gamma^* d\gamma^*\\
    & = r\Eb_{\gamma^*}\left[ \left(-e_1\sin\gamma^* + e_2\cos\gamma^*\right) ^{\TT}\frac{\partial}{\partial\gamma^*}\Eb[\hat{\theta}|\gamma^*] \right].
\end{align*}

For the derivative, it is worth noting that $\Eb[\hat{\theta}|\gamma^*] = \Eb \left[\Eb\left[\hat{\theta}\big|\{\bar{H}_{i,t}\}_{i\in[M]}\right] \big|\gamma^*\right]$. We have
\begin{align*}
    \frac{\partial}{\partial\gamma^*}\Eb[\hat{\theta}|\gamma^*]   & = \int_{\{\bar{H}_{i,t}\}_i}\Eb\left[\hat{\theta}\big|\{\bar{H}_{i,t}\}_{i}\right]\frac{1}{(2\pi)^{M(t-1)/2}}\frac{\partial}{\partial\gamma^*}\exp\left(-\frac{1}{2}\sum_{i\in[M],\tau<t} \left(r_{i,\tau,1} - x_{i,\tau,1}^{\TT}\theta^* \right)^2\right)\\
    & = r\Eb \left[ \Eb\left[\hat{\theta}|\{\bar{H}_{i,t}\}_i\right] (-e_1\sin\gamma^* + e_2\cos\gamma^*)^{\TT}\sum_{i,\tau<t} x_{i,\tau,1}(r_{i,\tau,1}- x_{i,\tau,1}^{\TT}\theta^*) \bigg| \theta^*\right]\\
    & = r\Eb \left[ \hat{\theta} (-e_1\sin\gamma^* + e_2\cos\gamma^*)^{\TT}\sum_i\bar{V}_i(\bar{\theta}_i-\theta^*) \big| \theta^*\right].
\end{align*}

Combining with the fact that $\Eb[\bar{V}_i(\bar{\theta}_i - \theta^*)|\theta^*, \bar{V}_i] = 0$, we have
\begin{align*}
    \Eb\left[\hat{\theta}^{\TT}\theta_1^*\right] &= r^2\Eb\left[(-e_1\sin\gamma^* + e_2\cos\gamma^*)^{\TT}\left(\hat{\theta} - \theta^*\right)(-e_1\sin\gamma^* + e_2\cos\gamma^*)^{\TT}\sum_i\bar{V}_i(\bar{\theta}_i-\theta^*)\right]\\
    & = r^2\sum_i\Eb[Z_i].
\end{align*}

Therefore,
\begin{align*}
    \Eb\left[\left\|\theta_1^* - \hat{\theta}\right\|^2\right] & = 2r^2 - 2\Eb\left[\hat{\theta}^{\TT}\theta_1^*\right] = 2r^2 - 2r^2\sum_{i}\Eb[Z_i],
\end{align*}
which completes the proof.

\end{proof}

{\bf Step 2: Upper Bound Each $\Eb[Z_i]$ under the CDP Constraint.}

\begin{lemma}\label{lemma: upper bound Zi}
    Under the same setting as in \Cref{lemma: Fingerprinting}, if the federated algorithm satisfies user-level $(\varepsilon,\delta)$-CDP, we have
    \begin{align}
        &\Eb[Z_j] \leq (e^{\varepsilon}-1)\sqrt{\frac{2(t-1)}{d}\Eb\left[\|\hat{\theta} - \theta^*\|^2\right]} + 6r\delta\sqrt{2(t-1)}\log(1/\delta),~~\forall j\neq i_0 , \label{eqn: upper bound Zi privately} \\
        &\Eb[Z_{i_0}] \leq \sqrt{\frac{2(t-1)}{d}\Eb\left[\|\hat{\theta} - \theta^*\|^2\right]}. \label{eqn: upper bound Zi non-private}
    \end{align}
\end{lemma}

\begin{proof}

Recall that 
\begin{align*}
Z_i =  (\hat{\theta} - \theta_1^*)^{\TT}(-e_1\sin\gamma^* + e_2\cos\gamma^*)(-e_1\sin\gamma^* + e_2\cos\gamma^*)^{\TT}\sum_{\tau<t}x_{i,\tau,1}(r_{i,\tau,1} - x_{i,\tau,1}^{\TT}\theta^*),
\end{align*}
 where $r_{i,\tau,1}$ is sampled from $N(x_{i,\tau,1}^{\TT}\theta_1^*,1)$, and $x_{i,\tau,1}=(0,\ldots,z_{i,\tau,s},\ldots,0)$ with probability $2/d$, and $z_{i,\tau,s}$ is sampled independently from a truncated normal $N(0,I_2|1)$.

We have
\begin{align*}
    \Eb[Z_{i_0}]^2&\leq \Eb\left[\|\hat{\theta} - \theta_1^*\|^2\right]\Eb\left[\left(\big(-e_1\sin\gamma^* + e_2\cos\gamma^*\big)^{\TT}\sum_{\tau<t}x_{i,\tau,1}(r_{i,\tau,1} - x_{i,\tau,1}^{\TT}\theta^*)\right)^2\right]\\
    & = \Eb\left[\|\hat{\theta} - \theta_1^*\|^2\right]\Eb\left[\sum_{\tau<t}\left(\big(-e_1\sin\gamma^* + e_2\cos\gamma^*\big)^{\TT}x_{i,\tau,1}\right)^2\right]\\
    & =  \frac{2(t-1)}{d}\Eb\left[\|\hat{\theta} - \theta_1^*\|^2\right],
\end{align*}
which verifies the second \Cref{eqn: upper bound Zi non-private} of \Cref{lemma: upper bound Zi}.

To prove the first part of \Cref{lemma: upper bound Zi}, the approach is akin to \citet{kamath2019privately}. We introduce a statistically indistinguishable random variable $\tilde{Z}_j$ for each $Z_j$, $j\neq i$. Let $\bar{H}_{j,t}'$ be sampled independently and identically with $\bar{H}_{j,t}$, and $\hat{\theta}^{-j}$ be the estimator constructed from $\{\bar{H}_{i,t}\}_{i\neq j}\cup\bar{H}_{j,t}'$. By the definition of CDP, $\hat{\theta}^{-j}$ is statistically indistinguishable compared with $\hat{\theta}$. Then, define 
\begin{align}
    \tilde{Z}_j = (\hat{\theta}^{-j} - \theta_1^*)^{\TT}(-e_1\sin\gamma^* + e_2\cos\gamma^*)(-e_1\sin\gamma^* + e_2\cos\gamma^*)^{\TT}\bar{V}_j(\bar{\theta}_j - \theta^*).
\end{align}

We have several properties about $\tilde{Z}_j$. First, due to the independence between $H_{j,t}$ and $H_{j,t}'$, the expectation of $\tilde{Z}_j$ is 0, i.e. $\Eb[\tilde{Z}_{j}] = 0$. 

Second, the variance of $\tilde{Z}_j$ is upper bounded by the estimation error, because
\begin{align*}
    \Eb[\tilde{Z}_j^2]&\overset{(a)}\leq \Eb\left[\|\hat{\theta}^{-j} - \theta_1^*\|^2\left((-e_1\sin\gamma^* + e_2\cos\gamma^*)^{\TT}\bar{V}_j(\bar{\theta}_j - \theta^*)\right)^2\right]\\
    & \overset{(b)}= \Eb\left[\Eb\left[\|\hat{\theta}^{-j} - \theta_1^*\|^2\right]\Eb\left[\left((-e_1\sin\gamma^* + e_2\cos\gamma^*)^{\TT}\bar{V}_j(\bar{\theta}_j - \theta^*)\right)^2\right] \bigg| \theta_1^*\right]\\
    &\overset{(c)}\leq \frac{2(t-1)}{d}\Eb\left[\|\hat{\theta} - \theta_1^*\|^2\right],
\end{align*}
where $(a)$ follows from the Cauchy's inequality, and $(b),(c)$ are due to the fact that $H_{j,t}'$ and $H_{j,t}$ are IID sampled.

Then, by choosing a threshold $Z>0$ which will be specified later, we can bound $\Eb[Z_j]$ as follows.
\begin{align*}
    \Eb\left[Z_{j}\right] & = \Eb\left[Z_{j}\right]  - \Eb[\tilde{Z}_{j}]\\
    &\overset{(a)}=\Eb\left[\Eb \left[\int_0^{+\infty}\left[\Pb\left(Z_{j}>z|\{\bar{H}_{i,t}\}_i\right) - \Pb\left(\tilde{Z}_{j}>z|\{\bar{H}_{i,t}\}_i, \bar{H}_{j,t}'\right)\right]dz\bigg| \theta_1^*, \{\bar{H}_{i,t}\}_i, \bar{H}_{j,t}'\right] \right]\\
        &\quad - \Eb\left[\Eb\left[\int_{-\infty}^0\left[\Pb\left(Z_{j}<z|\{\bar{H}_{i,t}\}_i\right) - \Pb\left(\tilde{Z}_{j}<z|\{\bar{H}_{i,t}\}_i, \bar{H}_{j,t}'\right)\right]dz\bigg| \theta_1^*, \{\bar{H}_{i,t}\}_i,\bar{H}_{j,t}'\right] \right]\\
    & \leq  \Eb\left[\Eb\left[\int_0^{Z}\left[\Pb\left(Z_{j}>z|\{\bar{H}_{i,t}\}_i\right) - \Pb\left(\tilde{Z}_{j}>z|\{\bar{H}_{i,t}\}_i,\bar{H}_{j,t}'\right)\right]dz\bigg| \theta_1^*, \{\bar{H}_{i,t}\}_i, H_{j,t}'\right] \right]\\
        & \quad  + \Eb\left[\Eb\left[\int_Z^{+\infty}\Pb\left(Z_{j}>z|\{\bar{H}_{i,t}\}_i\right)dz\bigg| \theta_1^*, \{\bar{H}_{i,t}\}_i\right] \right]\\
        & \quad + \Eb\left[\Eb\left[\int_{-Z}^0\left[\Pb\left(\tilde{Z}_{j}<z|\{\bar{H}_{i,t}\}_i, \bar{H}_{j,t}'\right) - \Pb\left(Z_{j}<z|\{\bar{H}_{i,t}\}_i \right)\right]dz\bigg| \theta_1^*, \{\bar{H}_{i,t}\}_i, \bar{H}_{j,t}'\right] \right]\\
        & \quad +  \Eb\left[\Eb\left[\int_{-\infty}^{-Z} \Pb\left(\tilde{Z}_{j}<z|\{\bar{H}_{i,t}\}_i, \bar{H}_{j,t}' \right)dz\bigg| \theta_1^*, \{\bar{H}_{i,t}\}_i,\bar{H}_{j,t}'\right] \right]\\
    &\overset{(b)}\leq Z\delta + (e^{\varepsilon}-1)\Eb\left[\Eb\left[\int_0^{Z}\Pb\left(\tilde{Z}_{j}>z|\{\bar{H}_{i,t}\}_i, \bar{H}_{j,t}'\right)dz\bigg| \theta_1^*, \{\bar{H}_{i,t}\}_i,\bar{H}_{j,t}'\right]\right]\\
        &\quad + \Eb\left[\Eb\left[\int_Z^{+\infty}\Pb\left(Z_{j}>z|\{\bar{H}_{i,t}\}_i\right)dz\bigg| \theta_1^*, \{\bar{H}_{i,t}\}_i\right] \right]\\
        & \quad + Z\delta + (1-e^{-\varepsilon})\Eb\left[\Eb\left[\int_{-Z}^0\Pb\left(\tilde{Z}_{j}<z|\{\bar{H}_{i,t}\}_i, \bar{H}_{j,t}'\right)dz\bigg| \theta_1^*, \{\bar{H}_{i,t}\}_i,\bar{H}_{j,t}'\right] \right]\\
        & \quad + \Eb\left[\Eb\left[\int_{-\infty}^{-Z} \Pb\left(\tilde{Z}_{j}<z|\{\bar{H}_{i,t}\}_i, \bar{H}_{j,t}'\right)dz\bigg| \theta_1^*, \{\bar{H}_{i,t}\}_i,\bar{H}_{j,t}'\right] \right]\\
    &\overset{(c)}\leq 2Z\delta + (e^{\varepsilon}-1)\Eb\left[|\tilde{Z}_{j}|\right]  \\
    &\quad + \Eb\left[\Eb\left[\int_Z^{+\infty}\Pb\left(Z_{j}>z\right)dz\bigg| \theta_1^*\right] \right]  +  \Eb\left[\Eb\left[\int_{-\infty}^{-Z} \Pb\left(\tilde{Z}_{i}<z \right)dz\bigg|\theta_1^*\right] \right]\\
    &\overset{(d)}\leq (e^{\varepsilon}-1)\sqrt{\frac{2(t-1)}{d}\Eb\left[\|\hat{\theta} - \theta_1^*\|^2\right]} + 2Z\delta\\
    &\quad \Eb\left[\Eb\left[\int_Z^{+\infty}\Pb\left(Z_{j}>z\right)dz\bigg| \theta_1^*\right] \right]  +  \Eb\left[\Eb\left[\int_{-\infty}^{-Z} \Pb\left(\tilde{Z}_{i}<z \right)dz\bigg|\theta_1^*\right] \right],
\end{align*}
where $(a)$ follows from $\Eb[X] = \int_{0}^{+\infty}\Pb(X>x)dx - \int_{-\infty}^0\Pb(X<x)dx$, $(b)$ follows from $(\varepsilon,\delta)$-CDP, $(c)$ is due to $1-e^{-\varepsilon}\leq e^\varepsilon-1$, and $(d)$ is due to the Cauchy's inequality.

For the term $\int_Z^{+\infty}\Pb\left(Z_{j}>z | \theta_1^*\right)dz$, we can bound it as
\begin{align*}
    \int_Z^{+\infty}&\Pb\left(Z_{j}>z|\theta_1^*\right)dz \leq \int_Z^{+\infty}\Pb\left(2r\big|(-e_1\sin\gamma^* + e_2\cos\gamma^*)^{\TT}\bar{V}_j(\bar{\theta}_j-\theta^*)\big|>z \big|\theta^*\right)dz.
\end{align*}

Note that $\bar{V}_j\bar{\theta}_j\sim N(\bar{V}_i\theta_1^*, \bar{V}_i)$ conditioned on $(\theta_1^*,\bar{V}_j)$. If we denote $W_{j} = (-e_1\sin\gamma^* + e_2\cos\gamma^*)^{\TT}\bar{V}_j(\bar{\theta}_j - \theta^*)$, then $W_{j}\sim N(0,\|(-e_1\sin\gamma^* + e_2\cos\gamma^*)\|_{\bar{V}_j}^2)$ and is conditionally independent with other $\bar{\theta}_i$, where $i\neq j$. Hence,
\begin{align*}
    \Eb&\left[\Eb\left[\int_Z^{+\infty}\Pb\left(Z_{j}>z|\{\bar{V}_j\}\right)dz \big| \theta^*\right] \right]\\
    &\leq \Eb\left[\int_Z^{+\infty}\Pb\left(2rW_{j}>z|\theta^*, \bar{V}_i\right)dz \right]\\
    & = \int_Z^{+\infty}\int_{\frac{z}{2r\|(-e_1\sin\gamma^* + e_2\cos\gamma^*)\|_{\bar{V}_j}}}^{+\infty}\frac{1}{\sqrt{2\pi}}\exp\left(-x^2/2\right)dxdz\\
    &\leq \Eb\left[\int_{\frac{Z}{2r\|(-e_1\sin\gamma^* + e_2\cos\gamma^*)\|_{\bar{V}_j}}}\frac{2r\|(-e_1\sin\gamma^* + e_2\cos\gamma^*)\|_{\bar{V}_j} \cdot x}{\sqrt{2\pi}}\exp(-x^2/2)dx \right]\\
    & = \Eb\left[\frac{2r\|(-e_1\sin\gamma^* + e_2\cos\gamma^*)\|_{\bar{V}_j}\exp(-\frac{Z^2}{8r^2\|(-e_1\sin\gamma^* + e_2\cos\gamma^*)\|_{\bar{V}_j}^2})}{\sqrt{2\pi}} \right]\\
    & \leq 2r\sqrt{(t-1)/\pi}\exp\left( - \frac{Z^2}{8r^2(t-1)}\right),\label{eqn: bound E[Z]}
\end{align*}
where the last inequality again follows from $\|-e_1\sin\gamma^* + e_2\cos\gamma^*\|_{\bar{V}_j}\leq \sqrt{t-1}$.

Following the same reason, we also have
\begin{align*}
   \int_{-\infty}^{-Z}\Pb\left(\tilde{Z}_{i}<z\right)dz 
    \leq 2r\sqrt{(t-1)/\pi}\exp\left( - \frac{Z^2}{8r^2(t-1)}\right).
\end{align*}

Thus, we conclude that
\begin{align*}
    \Eb[Z_j]\leq (e^{\varepsilon}-1)\sqrt{\frac{2(t-1)}{d}\Eb\left[\|\hat{\theta} - \theta_1^*\|^2\right]} + 2Z\delta + 4r\sqrt{(t-1)/\pi}\exp\left( - \frac{Z^2}{8r^2(t-1)}\right).
\end{align*}

We finish the proof by choosing $Z = 2\sqrt{2(t-1)}r\log(1/\delta)$.

\end{proof}

{\bf Step 3: Lower Bound the Total Regret.}


\begin{theorem}[Restatement of \Cref{thm:main-CDP-minimax}]\label{thm:Lowerbound-w/Memory-CDP}
    Fix any $\varepsilon\in(0,\log 2)$, $\delta=\tilde{O}\left(\frac{1}{M\sqrt{T}}\right)$ , $T\geq d^2$. Then, there exists a federated linear contextual bandits instance satisfying \Cref{asm:diversity,asm:margin} such that any {\bf with-memory} federated algorithm satisfying user-level $(\varepsilon,\delta)$-CDP must incur a regret lower bounded by
    \begin{align*}\textstyle
    &\Omega \left(\min\left\{M, \max \left\{1, \frac{1}{M\varepsilon^2 } \right\} \right\}C_0d\log T  \right).
\end{align*}
If \Cref{asm:margin} is not satisfied, then the minimax regret lower bound becomes 
\begin{align*}\textstyle
    \Omega\left( \min\left\{M, \max \left\{\sqrt{M}, \frac{1}{\varepsilon} \right\}\right\}\sqrt{dT}\right).
\end{align*}
\end{theorem}

\begin{proof}

Combine Step 1 (\Cref{lemma: Fingerprinting}) and Step 2 (\Cref{lemma: upper bound Zi}), we have,
\begin{align*}
    2r^2 &= \Eb\left[\|\hat{\theta} -\theta^*\|^2\right] + 2r^2\sum_{i}\Eb\left[Z_{i}\right]\\
    &\leq  \Eb\left[\|\hat{\theta} -\theta^*\|^2\right] + 2\left((e^{\varepsilon}-1)(M-1) + 1\right)r^2\sqrt{\frac{2(t-1)}{d}\Eb\left[\|\hat{\theta}-\theta^*\|^2\right]}  + 12r^3\delta  M\sqrt{2(t-1)}\log(1/\delta).
\end{align*}

Consider the case when $\varepsilon\leq \log 2$, $\delta \leq \frac{\sqrt{\pi}}{12rM\sqrt{2T}\log(1/\delta)}$, we further have
\begin{align*}
    r^2
    \leq  \Eb\left[\|\hat{\theta} -\theta^*\|^2\right] + 2r^2 (\varepsilon (M-1) + 1)\sqrt{\frac{2(t-1)}{d}\Eb\left[\|\hat{\theta}-\theta^*\|^2\right]}.
\end{align*}

Therefore,
\begin{align*}
    \Eb\left[\|\hat{\theta} -\theta^*\|^2\right]\geq \frac{r^2}{8r^2(\varepsilon(M-1)+1)^2 (t-1)/d + 4}.
\end{align*}

Substituting the above result into the generic lower bound \Cref{coro: generic lower bound}, we conclude that
\begin{align*}
    \text{Regret}(M,T) &\geq \Omega\left(\sum_{i\in[M],t\in[T]} \frac{1}{r} \frac{r^2}{32r^2 \max\{\varepsilon^2 M^2 ,1\} (t-1)/d + 4} \right)\\
    & \geq \Omega\left(\frac{dM}{r\max\{\varepsilon^2M^2,1\}}\log \left(1 + r^2\max\{\varepsilon^2M^2,1\}T/d\right) \right).
\end{align*}

The rest of the proof consists of two parts, in which two suitable $r$'s are chosen such that we can obtain the desired regret bounds.

First, under the margin condition, note that our setting indicates $C_0 = \Omega(1/r).$ Thus, under the margin condition in~\Cref{asm:margin}, we have
\begin{align*}
    \text{Regret}(M,T)\geq \Omega \left(\min\left\{M, \frac{1}{\varepsilon^2 M} \right\}C_0d\log T  \right).
\end{align*}
Thus, the first lower bound is obtained by combining the regret lower bound $\Omega(C_0d\log T)$ of non-private case (\Cref{prop:non private lower bound}).

Then, we select $r = \frac{\sqrt{d}}{\max\{\varepsilon M, 1\}\sqrt{T}}$, where we require that $T>d^2$. We obtain a worst-case lower bound as follows
\begin{align}
    \text{Regret}(M,T)\geq \Omega\left( \min\left\{M, \frac{1}{\varepsilon}\right\}\sqrt{dT}\right).
\end{align}
We finally finish the proof by combining with the non-private regret lower bound $\Omega(\sqrt{dMT}).$

\end{proof}

\subsection{Proof of \Cref{thm:main-CDP-lower-margin}}

In this section, we leverage the previous analysis to prove the result for regret lower bound under the CDP constraint, almost-memoryless setting, and the margin condition in~\Cref{asm:margin}.

The hard instance is defined the same as that in \Cref{sec:hard instance}. We point out that the only difference of memoryless case is that the estimator $\theta_{i,t}$ is a private estimator with respect to all clients, including itself.

\begin{theorem}[Restatement of \Cref{thm:main-CDP-lower-margin}]\label{thm:Lowerbound-Memoryless-CDP}
    If $\varepsilon<\log 2$, $\delta=\tilde{O}(\frac{1}{M\sqrt{T}})$, then, {there exists a federated linear contextual bandits instance satisfying  \Cref{asm:diversity,asm:margin}, such that} any {\bf almost-memoryless} federated algorithm satisfying user-level $(\varepsilon,\delta)$-CDP must incur a regret lower bounded by
    \begin{equation*}\textstyle
        \Omega\left(\max\left\{1, \frac{1}{M\varepsilon^2}\right\}C_0d\log T  + e^{-M\varepsilon}C_0MT\right).
    \end{equation*}
\end{theorem}

\begin{proof}

According to \Cref{coro: generic lower bound}, and \Cref{def:memoryless} of almost-memoryless algorithms, it suffices to analyze the estimation error 
\begin{align*}
\mathop{\inf}_{\theta_{i_0,t}\in\mathcal{F}(\bar{\mathcal{I}},\mathbb{S}_r) \atop \{\bar{R}_i\}_{i\in[0,M]}}  \Eb_v\left[ \left\|\theta_1^*-  \theta_{i_0,t}\right\|^2 \right]
\end{align*}
for a ``memoryless'' time $t$ and client $i_0$, where $\bar{\mathcal{I}}$ is the set of all possible $q_{\leq t}$.  By the post-processing property, $\theta_{i,t}$ is a $(\varepsilon,\delta)$-CDP estimator. Mathematically, for any subset $S\subset\mathbb{S}_r$ and $j$-neighboring dataset $H_t$, $H_t'$, where $j\in[M]$, we have
\begin{align*}
    \Pb\left(\theta_{i_0,t}\in S|H_t\right)\leq e^{\varepsilon}\Pb\left(\theta_{i_0,t}\in S|H_t'\right) + \delta.
\end{align*}

The proof follows the same argument as in \Cref{sec:CDP w. memory}, where we have three steps. The first step remains the same, where we construct $M$ random variables $\{Z_i\}$ (defined in \Cref{lemma: Fingerprinting}), and show that
\begin{align*}
    \Eb[\|\hat{\theta}-\theta_1^*\|^2] = 2r^2 - 2r^2\sum_i\Eb[Z_i],
\end{align*}
where $\hat{\theta}=\arg\mathop{\inf}_{\theta_{i_0,t}\in\mathcal{F}(\bar{\mathcal{I}},\mathbb{S}_r) \atop \{\bar{R}_i\}_{i\in[0,M]}}  \Eb_v\left[ \left\|\theta_1^*-  \theta_{i_0,t}\right\|^2 \right]$.

The second step is almost the same as in \Cref{lemma: upper bound Zi}, except that we can upper bound $\Eb[Z_{i,0}]$ using the same inequality in \Cref{eqn: upper bound Zi privately}.

Therefore, we obtain, 
\begin{align*}
    2r^2 &= \Eb\left[\|\hat{\theta} -\theta^*\|^2\right] + 2r^2\sum_{i}\Eb\left[Z_{i}\right]\\
    &\leq  \Eb\left[\|\hat{\theta} -\theta^*\|^2\right] + 2(e^{\varepsilon}-1)M r^2\sqrt{\frac{2(t-1)}{d}\Eb\left[\|\hat{\theta}-\theta^*\|^2\right]}  + 12r^3\delta  M\sqrt{2(t-1)}\log(1/\delta).
\end{align*}

By choosing $\varepsilon<\log 2$, and $\delta< \frac{1}{12rM\sqrt{2T}\log(1/\delta)}$, we have
\begin{align*}
    \Eb\left[\|\hat{\theta} -\theta^*\|^2\right] \geq \frac{r^2}{8r^2\varepsilon^2M^2(t-1)/d + 4}.
\end{align*}

Substituting the above inequality into the generic lower bound in \Cref{coro: generic lower bound}, and noting that $r= O(1/C_0)$, we conclude that
\begin{align*}
    \text{Regret}(M,T)\geq \Omega\left( \frac{C_0d\log T}{\varepsilon^2 M}\right).
\end{align*}

To derive the second term $e^{-\varepsilon M}MT$ in the regret lower bound, we directly analyze the estimation error as follows
\begin{align*}
    \Eb[\|\hat{\theta} - \theta_1^*\|^2]& = \Eb\left[\int_{0}^{4r^2}\Pb\left(\|\hat{\theta} - \theta^*\|^2 > z | \{\bar{H}_{i,t}\}_i\right) \right]\\
    & \overset{(a)}\geq e^{-M\varepsilon}\Eb\left[\int_{0}^{4r^2}\Pb\left(\|\hat{\theta} - \theta^*\|^2 > z \big| \{\bar{H}_{i,t}\}_i = \{x_{i,\tau,a}=0,r_{i,\tau,a}=0\}_{i,\tau}\right) \right] - 4r^2M\delta\\
    & =  e^{-M\varepsilon}\Eb\left[\|\hat{\theta}^0 - \theta^*\|^2 \right] - 4r^2M\delta,\\
    & \overset{(b)}\geq e^{-M\varepsilon}\left(r^2 - 2\Eb[\hat{\theta}^{0\TT}\theta_1^*]\right) - 4r^2\delta\\
    & = e^{-M\varepsilon}r^2 - 4r^2M\delta,
\end{align*}
where $(a)$ is due to the CDP constraint, $\hat{\theta}^0$ is the output from a fixed ``zero'' dataset, which is independent with $\theta^*$, and $(b)$ follows from the independency, and $\Eb[\theta^*] = 0$.

Note that $r=O(1/C_0)$. By choosing $\delta<O(\frac{1}{dM^2 T})$, we have
\begin{align*}
    \text{Regret}(M,T) \geq \Omega\left(e^{-M\varepsilon}C_0MT - 1\right).
\end{align*}
We finish the proof by noting that the non-private regret lower bound is $\Omega(C_0d\log T)$ under the margin condition, according to \Cref{prop:non private lower bound}.

\end{proof}

\section{Proof of the Regret Lower Bounds Under User-level LDP Constraint}\label{sec:proof of LDP lower bound}

In this section, we provide the proof for regret lower bounds under the user-level LDP constraint. It consists of two subsections. The first subsection lists several general lemmas, which are used to bound the total variation distance between multivariate distributions. The second subsection provides the full proof of the lower bounds.

\subsection{Useful Lemmas for the Proof of \Cref{thm:main-LDP}}
We first introduce a lemma that bounds the divergence of the output distributions from a DP channel with different input distributions.
\begin{lemma}[Adapted from Lemma 2 in \citet{asoodeh2021local}]\label{lemma: From q to H}
If $q_{i,\leq t}$ is the output of an $(\varepsilon,\delta)$-LDP channel $\Rt_i$ with input $H_{i,t}$, and $H_{i,t}$ follows a prior distribution parameterized by $\theta$, then, for any two different $\theta,\theta'$, let $\Pb(q_{i,\leq t}|\theta)$ be the marginal distribution of $q_{i,t}$, and we have
    \[\begin{aligned}
    &\KL(\Pb(q_{i,\leq t}|\theta),\Pb(q_{i, \leq t}|\theta')) \leq \left(1-e^{-\varepsilon}(1-\delta)\right)\KL\left(\Pb( H_{i,t} |\theta), \Pb( H_{i,t}|\theta')\right),\\
    &d_{TV}(\Pb(q_{i,\leq t}|\theta),\Pb(q_{i, \leq t}|\theta')) \leq \left(1-e^{-\varepsilon}(1-\delta)\right)d_{TV}\left(\Pb( H_{i,t} |\theta), \Pb( H_{i,t}|\theta')\right).
    \end{aligned}\]
\end{lemma}

In the following, we give a tighter bound on the total variation distance of two multivariate distributions. It is crucial since the policy depends on both local data $H_{i,t}$ and global information $q_{\leq t}$. While $q_{\leq t}$ is from a DP channel, $H_{i,t}$ is a non-private information and should be analyzed separately.

First, we introduce the notion of coupling and relate the total variation distance with an error probability.

\begin{definition}[Coupling \citep{den2012probability}]
    A coupling of two random variables $X, X'$ is any pair of random variables $(\hat{X},\hat{X}')$ such that their marginals have the same distribution as $X$ and $X'$, i.e. $\hat{X}\overset{D}=X$, and $\hat{X}'\overset{D}=X'$. The law $\hat{\Pb}$ of $(\hat{X},\hat{X}')$ is a coupling of the laws $\Pb$ and $\Pb'$ of $X$ and $X'$.
\end{definition}

\begin{lemma}[Theorem 2.4 \& Theorem 2.12 in \citet{den2012probability}]\label{lemma: coupling}
    For any two probability measures $\Pb$ and $\Pb'$ on the same measurable space, any coupling $\hat{\Pb}$ satisfies \[d_{TV}\left(\Pb, \Pb'\right) \leq \hat{\Pb}(\hat{X}\neq \hat{X}').\]
    Moreover, there exists a coupling $\hat{\Pb}_0$ such that
    \[d_{TV}\left(\Pb, \Pb'\right) = \hat{\Pb}_0(\hat{X}\neq \hat{X}').\]
    
\end{lemma}

Equipped with the coupling method, we are able to upper bound the total variation distance of two multivariate distributions, as shown in the following lemma.
\begin{lemma}[Total variation distance of two multivariate distributions]\label{lemma: 2RV-TV-bound}

Let $\Pb$ and $\Qb$ be two multivariate distributions defined on $\mathcal{X}\times\mathcal{Y}$, and suppose $\Pb(X,Y) =\Pb_1(X|Y)\Pb_2(Y) $, $\Qb(X,Y) = \Qb_1(X|Y)\Qb_2(Y)$. Then, we have
    \begin{align*}
        d_{TV}\left(\Pb -\Qb\right)\leq 1- \left(1-\max_yd_{TV}\left(\Pb_1(\cdot|y),\Qb_1(\cdot|y)\right)\right)\left(1-d_{TV}\left(\Pb_2(Y),\Qb_2(Y)\right)\right).
    \end{align*}
\end{lemma}

\begin{proof}

By \Cref{lemma: coupling}, we can find couplings $\hat{\Pb}_1(\hat{X},\hat{X}'|\hat{Y},\hat{Y}')$ and $\hat{\Pb}_2(\hat{Y},\hat{Y}')$ such that

\begin{itemize}
    \item $\hat{\Pb}_1(\hat{X},\hat{X}'|\hat{Y},\hat{Y}')$ is a coupling of $\Pb_1(X|Y)$ and $\Qb_1(X'|Y')$. Moreover, 
    \[\hat{\Pb}_1(\hat{X}\neq \hat{X}'|\hat{Y},\hat{Y}')=d_{TV}\left(\Pb_1(\cdot|\hat{Y}), \Qb(\cdot|\hat{Y}')\right).\]
    \item $\hat{\Pb}_2(\hat{Y},\hat{Y}')$ is a coupling of $\Pb_2(Y)$ and $\Qb_2(Y)$. Moreover,
    \[\hat{\Pb}_2(\hat{Y}\neq \hat{Y}') = d_{TV}\left(\Pb_2,\Qb_2\right).\]
\end{itemize}

Then, if we define $\hat{\Pb}(\hat{X},\hat{X}',\hat{Y},\hat{Y}') = \hat{\Pb}_1(\hat{X},\hat{X}'|\hat{Y},\hat{Y}')\hat{\Pb}_2(\hat{Y},\hat{Y}') $, it can be verified that $\hat{\Pb}(\hat{X},\hat{X}',\hat{Y},\hat{Y}')$ is a coupling of $\Pb(X,Y)$ and $\Qb(X,Y)$, since
    \begin{align*}
        \int_{\hat{X}',\hat{Y}'} \hat{\Pb}(\hat{X},\hat{X}',\hat{Y},\hat{Y}') &= \int_{\hat{Y}'}\int_{\hat{X}'} \hat{\Pb}_1(\hat{X},\hat{X}'|\hat{Y},\hat{Y}')\hat{\Pb}_2(Y,\hat{Y}') \\
        & = \int_{\hat{Y}'} \Pb_1(\hat{X}|\hat{Y})\hat{\Pb}_2(\hat{Y},\hat{Y}')\\
        & = \Pb_1(\hat{X}|\hat{Y})\Pb_2(\hat{Y})\\
        & = \Pb(\hat{X},\hat{Y}).
    \end{align*}

Then, by \Cref{lemma: coupling}, we have
    \begin{align*}
        d_{TV}\left(\Pb,\Qb\right) &\leq \hat{\Pb}((\hat{X},\hat{Y})\neq (\hat{X}',\hat{Y}')) \\
        & = 1 - \hat{\Pb}((\hat{X},\hat{Y}) =  (\hat{X}',\hat{Y}'))\\
        & = 1 - \hat{\Pb}(\hat{X} = \hat{X}', \hat{Y}=\hat{Y}')\\
        & = 1 - \hat{\Pb}_1(\hat{X}=\hat{X}'|\hat{Y}=\hat{Y}') \hat{\Pb}_2(\hat{Y}=\hat{Y}')\\
        & \leq 1 - \left( 1 - \max_{y} d_{TV}\big(\Pb_1(\cdot|y),\Qb_1(\cdot|y)\big)\right)\left(1-d_{TV}\left(\Pb_2,\Qb_2\right)\right),
    \end{align*}
which completes the proof.

\end{proof}

\if{0}

\begin{proposition}[LDP, Memoryless, Margin condition]\label{prop:LDP memoryless margin}
If $\delta<1/M$, then the lower bound is $\Omega(e^{-M\varepsilon}MT)$ given the margin condition.
\end{proposition}

\begin{proof}

We follow almost the same setting and notations in \Cref{prop:CDP memoryless e<1/M margin}.

Similar to \Cref{thm:Lowerbound-w/Memory-CDP}, we aim to bound 
\begin{equation}\label{eqn:instance-err-LDP-memoryless }
\mathop{\inf}_{\hat{\theta}_{i, t}\in\mathcal{F}(\bar{\mathcal{Q}},\Rb^2) \atop \{\bar{R}_i\}_{i\in[0,M]} }\Eb\left[\|\theta^* - \hat{\theta}_{i, t}(\bar{q}_{\leq t})\|^2\right] = \mathop{\inf}_{\hat{\theta}_{t}\in\mathcal{F}(\bar{\mathcal{Q}},\Rb^2) \atop \{\bar{R}_i\}_{i\in[0,M]} }\Eb\left[\|\theta^* - \hat{\theta}_{t}(\bar{q}_{\leq t})\|^2\right]
\end{equation}
for a single client $i=i_0\in[M]$ and a fixed time $t$.

Following techniques in proposition 4.1 in \cite{he2022reduction}, we have for any $\theta,\theta'\in\Theta$,
\begin{align}
    \Pb\left(\|\theta^* - \hat{\theta}_{t}\|_2^2\geq \frac{1}{4}\|\theta-\theta'\|_2^2 \bigg|\theta^* = \theta \right) \geq \frac{1}{2}\left(1 - d_{TV}\left(\Pb(\bar{q}_{\leq t}\big| \theta) , \Pb(\bar{q}_{\leq t}\big|\theta')\right) \right).
\end{align}

To upper bound the total variation $d_{TV}\left(\Pb(q_{\leq t}\big| \theta) , \Pb(q_{\leq t}\big|\theta')\right)$, different from \Cref{prop:CDP memoryless e<1/M margin}, due to LDP, $\bar{q}_{i,\leq t}$ are sampled independently conditioned on the extended history $\bar{H}_{i,t}$ which are also iid samples.

Thus, by \Cref{lemma: From q to H}
\begin{align}
    d_{TV}&\left(\Pb(\bar{q}_{\leq t}\big|\theta),\Pb(\bar{q}_{\leq t}\big|\theta')\right)\\
    &\leq (1-e^{-M\varepsilon}(1-\delta)^M)) d_{TV}\left(\Pb(\{\bar{H}_{i,t}\}_i|\theta),\Pb(\{\bar{H}_{i,t}\}_i|\theta')\right) \\
    & \overset{(a)}\leq (1-e^{-M\varepsilon}(1-\delta)^M))\sqrt{1-\exp\left(-2\|\theta-\theta'\|_2^2M(t-1)\right)}
\end{align}
where $(a)$ follows from $d_{TV}(P,Q)\leq \sqrt{1-\exp(-\KL(P\|Q))}$.

Let $c = (1-e^{-M\varepsilon}(1-\delta)^M))$, Then, we have
\begin{align}
    \Eb[&\|\theta^* - \hat{\theta}_{t}\|^2]\\
    &\geq \int_0^{r^2}\left(1- c\sqrt{1-e^{-2qM(t-1)}} \right)dq\\
    &=\left\{
        \begin{aligned}
        &(1-c)r^2 + c\frac{\sqrt{1-e^{-2r^2M(t-1)}} - \log(1+\sqrt{1-e^{-2r^2M(t-1)}})}{M(t-1)},~~ t\geq 2\\
        & \frac{r^2}{2},~~ t=1
        \end{aligned}
    \right.\\
    &\geq \left\{
        \begin{aligned}
        &(1-c)r^2 ~~ t\geq 2\\
        & \frac{r^2}{2},~~ t=1
        \end{aligned}
    \right.
\end{align}

Substituting the above inequality into the general lower bound \Cref{eqn: generic lower bound}, and let $r=\Omega(1)$, we conclude that
\begin{align}
    Regret(M,T)\geq \Omega(e^{-\varepsilon M}MT).
\end{align}

\end{proof}
\fi

\subsection{Proof of \Cref{thm:main-LDP}}\label{sec:proof of main-LDP}

We follow the same setting defined in \Cref{sec:hard instance}. Hence, it suffices to lower bound the estimation error 
\[\mathop{\inf}_{\theta_{i,t}\in\mathcal{F}(\bar{\mathcal{I}}_i,\mathbb{S}_r) \atop \{\bar{R}_i\}_{i\in[0,M]}}  \Eb_v\left[ \left\|\theta_1^*-  \theta_{i,t}\right\|^2 \right],
    \]
for each single time step $t$ and each client $i$. We emphasize that $\theta_{i,t}$ is an LDP estimator with respect to all clients $j\neq i$. Therefore, without loss of generality, we assume $\bar{q}_{\leq t} = \{\bar{q}_{i,\leq t}\}$, i.e. $\Rt_0$ is an identical map that does not perform any operation on the aggregated information.

\if{0}
\begin{theorem}[LDP, With-memory]\label{thm:lowerbound LDP w/memory}
\begin{align*}
     \text{Regret}&(M,T) \\
     &\geq \left\{
        \begin{aligned}
            &\min\left\{ 1, \frac{1}{\sqrt{M\varepsilon_0}}\right\}M\sqrt{T}, ~~ \text{ Worst case }\\
            &\min\left\{1, \frac{1}{M\varepsilon_0}\right\} M\log T,~~\text{ if Margin condition is satisfied }
        \end{aligned}
    \right.
\end{align*}
\end{theorem}
\fi

\begin{proof}[Proof of \Cref{thm:main-LDP}]

Recall that the DP mechanism $R_i$ is non-interactive, i.e. $q_{i,\leq t}$ is independent with other client's data conditioned on the client $i$'s own data. 

Moreover, the full information data set $\bar{H}_{i,t}$ ($\bar{H}_{i,t}$ contains rewards sampled from all un-played arms) is independent with the global information $q_{\leq t}$ conditioned on $\theta_1^*$. Thus, for any $\theta$ and $\theta'$,
\begin{align*}
\KL &\left[\Pb(\bar{H}_{i,t}|\theta,q_{\leq t})\|\Pb(\bar{H}_{i,t}|\theta',q_{\leq t})\right]\\
& = \Eb\left[\log\frac{\Pb\left(\bar{H}_{i,t}|\theta\right)}{\Pb\left(\bar{H}_{i,t}|\theta'\right)}\right]\\
& = \sum_{\tau\in[t-1],a\in\{1,2\}}^{t-1}\Eb\left[\log\frac{\Pb\left(c_{i,\tau},a,r_{i,\tau,a}|\theta\right)}{\Pb\left(c_{i,\tau},a,r_{i,\tau,a}|\theta'\right)}\right]\\
& \overset{(a)}= \sum_{\tau\in[t-1]}\Eb\left[\log\frac{\Pb\left(r_{i,\tau,1}|\theta,x_{i,\tau,1}\right)}{\Pb\left(r_{i,\tau,a}|\theta',x_{i,\tau,1}\right)}\right]\\
&\overset{(b)}\leq \|\theta-\theta'\|^2(t-1)/d,
\end{align*}
where $(b)$ follows from that only the reward of the first arm depends on the model parameter, and $(a)$ is due to the fact that the KL-divergence of two Gaussian random variables is upper bounded by the squared difference of their expectations.

Similarly, due to $\theta^* - \{H_{i,t}\}_{i\in[M]}-\{q_{\leq t}\}$, we have
\[ \KL\left[\Pb(H_{1,t},...,H_{M,t}|\theta)\|\Pb(H_{1,t},...,H_{M,t}|\theta')\right]\leq \|\theta -\theta'\|^2M(t-1)/d.
\]

Then, we apply \Cref{lemma: From q to H} and the chain rule of KL-divergence on the total variation distance between $\Pb(\bar{q}_{\leq t}|\theta) $ and $\Pb(\bar{q}_{\leq t}|\theta')$.

\begin{align*}
    d_{TV}&\left(\Pb(\bar{q}_{\leq t}|\theta) ,\Pb(\bar{q}_{\leq t}|\theta')\right)\\
    &\overset{(a)}\leq \sqrt{ 1- \exp\left( - \KL\left[\Pb(\bar{q}_{\leq t}|\theta)\|\Pb(\bar{q}_{\leq t}|\theta')\right]\right)}\\
    & =  \sqrt{ 1- \exp\left( - \sum_{i\in[M]}\KL\left[\Pb(\bar{q}_{i, \leq t}|\theta)\|\Pb(\bar{q}_{i, \leq t}|\theta')\right]\right)}\\
    &\leq \sqrt{ 1- \exp\left( - (1-e^{-\varepsilon}(1-\delta))\sum_{i\in[M]}\KL\left[\Pb(\bar{H}_{i, \leq t}|\theta)\|\Pb(\bar{H}_{i, \leq t}|\theta')\right]\right)}\\
    &\leq \sqrt{ 1- \exp\bigg( - (1-e^{-\varepsilon}(1-\delta))M\|\theta-\theta'\|^2(t-1)/d\bigg)},
\end{align*}
where $(a)$ is due to the Bretagnolle–Huber  inequality.

Let $\varepsilon' = (1-e^{-\varepsilon}(1-\delta))$. Based on \Cref{lemma: 2RV-TV-bound}, we characterize the total variation distance of the joint distribution of local data $H_{i,t}$ and global information $\bar{q}_{\leq t}$ as follows.
\begin{align*}
d_{TV}&\left(\Pb(\bar{H}_{i,t},\bar{q}_{\leq t}|\theta) ,\Pb(\bar{H}_{i,t},\bar{q}_{\leq t}|\theta')\right) \\
&\leq 1-\left(1-d_{TV}\left(\Pb(\bar{H}_{i,t}|\theta),Pr(\bar{H}_{i,t}|\theta')\right)\right)\left( 1-  d_{TV}\left(\Pb(\bar{q}_{\leq t}\big|\theta),\Pb(\bar{q}_{\leq t}\big|\theta')\right)\right)\\
&\leq 1-\left(1-\sqrt{1-e^{-2\|\theta-\theta'\|^2(t-1)/d}}\right)\left(1- \sqrt{1-e^{-2\varepsilon_0\|\theta-\theta'\|^2M(t-1)/d}}\right).
\end{align*}

Now, applying the technique in Proposition 4.1 in \citet{he2022reduction}, we have for any $\theta,\theta'\in\Theta$,
\begin{align*}
    \Pb&\left(\|\theta^* - \theta_{i,t}\|_2^2\geq \frac{1}{4}\|\theta-\theta^{'}\|_2^2|\theta^* = \theta\right)\\
    &\geq \frac{1}{2}\left((1 - d_{TV}\left( \Pb(H_{i,t},q_t \big| \theta) , \Pb(H_{i,t},q_t  \big|\theta')\right)\right)\\
    &\geq \frac{1}{2}\left(1 -  \sqrt{1-e^{-\varepsilon'\|\theta-\theta'\|^2M(t-1)/d}}\right)\left(1- \sqrt{1-e^{-\|\theta-\theta'\|^2(t-1)/d}}\right).
\end{align*}

Therefore, if  $t>1$, we have
\begin{align*}
    \Eb[\|\theta - \theta_{i,t}\|^2]&\geq \frac{1}{2}\int_0^{r^2}\left(1- \sqrt{1-e^{-q\varepsilon'M(t-1)/d}} \right)\left(1-\sqrt{1-e^{-q(t-1)/d}}\right)dq\\
    &\geq \frac{1}{2}\int_0^{\min\{r^2, \frac{d}{\varepsilon' M(t-1)}, \frac{d}{t-1}\}} \left(1-\sqrt{q}\sqrt{\varepsilon' M(t-1)/d}\right)\left(1 - \sqrt{q}\sqrt{(t-1)/d}\right)\\
    & =\left\{
        \begin{array}{ll}
            r^2\left(1 + r^2/2 - \frac{2}{3\sqrt{d}}r(1+M\varepsilon')\sqrt{t-1}\right), & \text{ if } r^2 < \min\left\{1,\frac{1}{M\varepsilon'}\right\}\frac{d}{t-1},\\
            (2-\sqrt{M\varepsilon'})\frac{d}{t-1},&\text{ if } \varepsilon' M<1, r^2 > \frac{d}{t-1},\\
            (2 - 1/\sqrt{M\varepsilon'})\frac{d}{M\varepsilon'(t-1)}, & \text{ if } \varepsilon' M > 1, r^2>\frac{d}{M\varepsilon' (t-1)}.
        \end{array}
    \right.
\end{align*}
Thus, the regret is bounded below by
\begin{align*}
     \text{Regret}(M,T) &\geq \Omega\left(\frac{1}{r}\sum_{i,t}\Eb[\|\theta - \theta_{i,t}\|^2] \right)\\
     &\geq \left\{
        \begin{array}{ll}
            \Omega\left(rM\left(T + Tr^2/2 - \frac{4}{9\sqrt{d}}r(1+M\varepsilon')T^{3/2}\right) \right), ~~ &\text{ if } r^2 < \min\left\{1,\frac{1}{M\varepsilon'}\right\}\frac{d}{t-1}, \\
            \Omega\left(C_0dM \log T\right),~~ &\text{ if } \varepsilon' M<1, r=O(1/C_0),\\
            \Omega\left( \frac{C_0d\log T}{\varepsilon'} \right), ~~ &\text{ if } \varepsilon' M > 1, r=O(1/C_0).
        \end{array}
    \right.
\end{align*}

By selecting $r = O(\min\{1,1/\sqrt{M\varepsilon'}\})\sqrt{d/T}$, and noting that $\varepsilon' = O(\varepsilon)$ when $\varepsilon<\log 2, \delta<0.1$ we obtain two lower bounds:
\begin{align*}
     \text{Regret}(M,T)   &\geq \left\{
        \begin{array}{ll}
             \Omega\left(\min\left\{ M, \frac{\sqrt{M}}{\sqrt{\varepsilon}}\right\}\sqrt{dT} \right), &\text{without \Cref{asm:margin}},\\
             \Omega\left( \min\left\{M, \frac{1}{\varepsilon}\right\} C_0d\log T \right), &\text{with \Cref{asm:margin}.}
        \end{array}
    \right.
\end{align*}
\end{proof}

\subsection{Proof of the Regret Lower Bounds Under User-level Pure-LDP Constraint}\label{sec: proof of pure LDP lower bound}

\begin{corollary}[Restatement of \Cref{coro:pure LDP lower bound}]
    For any $\varepsilon\in(0,\log 2)$, there exists a federated linear contextual bandits instance satisfying \Cref{asm:diversity,asm:margin} such that any {\bf with-memory} federated algorithm satisfying $\varepsilon $-LDP must incur a regret lower bounded by
    \begin{equation*}\textstyle
        \Omega\left( \min\left\{M, 1/\varepsilon^2  \right\}C_0d\log T\right).
    \end{equation*}
    If \Cref{asm:margin} is not satisfied, then the minimax regret lower bound becomes 
    \begin{equation*}\textstyle
        \Omega\left(\min \left\{M, \sqrt{M}/\varepsilon \right\} \sqrt{dT} \right).
    \end{equation*}
    
\end{corollary}

\begin{proof}
    We follow nearly the same argument in \Cref{sec:proof of main-LDP}, except the upper bound of $d_{TV}\left( \Pb(\bar{q}_{\leq t} | \theta), \Pb(\bar{q}_{\leq} |\theta') \right)$. Since we are under the $\varepsilon$-LDP constraint, we have
    \begin{align}
         d_{TV}&\left(\Pb(\bar{q}_{\leq t}|\theta) ,\Pb(\bar{q}_{\leq t}|\theta')\right)\\
        &\leq \sqrt{ 1- \exp\left( - \KL\left[\Pb(\bar{q}_{\leq t}|\theta)\|\Pb(\bar{q}_{\leq t}|\theta')\right]\right)}\\
        & =  \sqrt{ 1- \exp\left( - \sum_{i\in[M]}\KL\left[\Pb(\bar{q}_{i, \leq t}|\theta)\|\Pb(\bar{q}_{i, \leq t}|\theta')\right]\right)}\\
        &\overset{(a)}\leq \sqrt{ 1- \exp\left( - 4\varepsilon^2\sum_{i\in[M]}d_{TV}^2\left[\Pb(\bar{H}_{i, \leq t}|\theta)\|\Pb(\bar{H}_{i, \leq t}|\theta')\right]\right)}\\
        &\overset{(b)}\leq \sqrt{ 1- \exp\left( - 2\varepsilon^2\sum_{i\in[M]}\KL\left[\Pb(\bar{H}_{i, \leq t}|\theta)\|\Pb(\bar{H}_{i, \leq t}|\theta')\right]\right)}\\
        &\leq \sqrt{ 1- \exp\bigg( - 2\varepsilon^2M\|\theta-\theta'\|^2(t-1)/d\bigg)},
    \end{align}
where $(a)$ is due to Theorem 1 in \citet{duchi2013local}, and $(b)$ is due to the Pinsker's inequality.

Following the same argument, we can conlude that under the user-level $\varepsilon$-LDP, we have
\begin{align*}
     \text{Regret}(M,T) &\geq \left\{
        \begin{array}{ll}
             \Omega\left(\min\left\{ M, \frac{\sqrt{M}}{ \varepsilon}\right\}\sqrt{dT} \right), &\text{without \Cref{asm:margin}, }\\
             \Omega\left( \min\left\{M, \frac{1}{\varepsilon^2}\right\} C_0d\log T \right), &\text{with \Cref{asm:margin}. }
        \end{array}
    \right.
\end{align*}

\end{proof}

\section{Auxiliary Lemmas}
This section present lemmas that are commonly used in both bandits literature and differential privacy works, including concentration inequality, composition rule and elliptical potential lemma.

The first is the advanced composition rule, which allows us to reduce the dependency on $k$ for a $k$-fold composition mechanism.
\begin{lemma}[Advanced composition rule, Theorem 3.20 in \cite{dwork2014algorithmic}]\label{lemma:advanced composition}
    For all $\varepsilon,\delta,\delta'>0$, the class of $(\varepsilon,\delta)$-differentially private mechanisms satisfies $(\varepsilon',k\delta + \delta')$-differential privacy under $k$-fold composition for 
    \[\varepsilon' = \varepsilon\sqrt{2k\log(1/\delta')} + k\varepsilon(e^{\varepsilon} - 1).\]
\end{lemma}

By noting that $e^{\varepsilon} - 1 <\varepsilon$ when $\varepsilon<\log 2$, we have the following corollary.
\begin{corollary}\label{coro:advanced composition}
    Under the same setting in the advanced composition rule, when  $\varepsilon<1/\sqrt{k}<\log 2$, {we must have } $\varepsilon'\leq \sqrt{6k\log(1/\delta')}$.
\end{corollary}

Then, we provide several probability bound random vector and random matrices.
\begin{lemma}\label{lemma:Laplace tail bound}
    Let $X_1,\ldots,X_d$ be $d$ IID random variables following distribution $\text{Laplace}(0,b)$. Then, for any $\beta>0$, we have
    \begin{align*}
        \Pb\left(\sqrt{\sum_{s=1}^dX_s^2} \geq b\sqrt{d}\log(d/\beta)\right) \leq \beta.
    \end{align*}
\end{lemma}

\begin{proof}
We note that
    \begin{align*}
        \Pb\left(\sqrt{\sum_{s=1}^dX_s^2} \geq t\right)& = \Pb\left(\sum_{s=1}^dX_s^2 \geq t^2\right)\\
        &\leq \Pb(\max_sX_s^2 \geq t^2/d)\\
        &\overset{(a)}\leq \sum_{s=1}^d\Pb(|X_s|\geq t/\sqrt{d})\\
        &\overset{(b)} \leq de^{-\frac{t}{b\sqrt{d}}},
    \end{align*}
    where $(a)$ is due to union bound, and $(b)$ follows from the fact that the density of Lap($0,b$) is $e^{-|x|/b}/(2b)$.
    Setting $t=b\sqrt{d}\log(d/\beta)$, we complete the proof.
\end{proof}

\begin{lemma}[Matrix concentration lemma, Theorem 1.2 in  \cite{tropp2011freedman}]\label{lemma:matrix concentration}
Consider a martingale difference sequence of symmetric random matrices $\{X_t\}_{t}$ with filtration $\{\Hc_t\}_t$. Suppose $\Eb[X_t|\mathcal{H}_t]= 0$ and $\lambda_{\max}\left(X_t\right)\leq R $ almost surely. Then, 

\[\mathbb{P}\left(\lambda_{\max}\left(\sum_{t=1}^TX_t \right) \geq n 
 , \text{ and }  \left\|\sum_{t\in[T]}\mathbb{E}\left[X_t^2\right]\right\| \leq \sigma^2\right) \leq d\exp\left(\frac{-n^2/2}{\sigma^2+Rn/3}\right),\]

where $\|X\|$ is the spectral norm of a matrix $X$.

The following lemma is widely used in the linear bandits literature.
\begin{lemma}[Elliptical potential lemma, Proposition 1 in \citet{carpentier2020elliptical}]\label{lemma:elliptical}
    Let $\{x_t\}_{t\geq 1}\subset\Rb^d$ be an arbitrary sequence of $d$-dimensional vectors such that $\|x_t\|\leq 1$ for all $t\geq 1$. If $V_t = \lambda I_d + \sum_{\tau=1}^{t-1}x_{\tau}x_{\tau}^{\TT}$, then
    \begin{align}
        \sum_{t=1}^T\|x_t\|_{V_t^{-1}}\leq \sqrt{dT\log\frac{T+d\lambda}{d\lambda}}.
    \end{align}
\end{lemma}

\end{lemma}
Finally, we provide the regret lower bound under the non-private setting and the margin condition for completeness.

Due the margin condition in \Cref{asm:margin}, we choose $r=O(1/C_0)$ in Proposition 4.1 in \citet{he2022reduction}. Then, we have the following non-private regret lower bound under the margin condition.
\begin{proposition}\label{prop:non private lower bound}
There exists a federated linear contextual bandits instance satisfying the diversity (\Cref{asm:diversity}) and the margin (\Cref{asm:margin}) conditions such that any non-private federated learning algorithm must incur a regret lower bounded by
    $\Omega(C_0d\log T).$ 
\end{proposition}
\end{document}